\documentclass[11pt]{article}

\usepackage{amsthm,amsmath,amssymb,cite,graphicx,fullpage,mathtools,bm,makecell, multirow}

\usepackage[colorlinks=true,
linkcolor=blue,
urlcolor=blue,
citecolor=blue]{hyperref}

\allowdisplaybreaks

\newtheorem{thm}{Theorem}
\newtheorem{cor}[thm]{Corollary}
\newtheorem{prop}[thm]{Proposition}
\newtheorem{lemma}[thm]{Lemma}
\newtheorem{ass}{Assumption}

\theoremstyle{remark}\newtheorem{remark}{Remark}

\newcommand{\0}{\mathbf{0}}

\newcommand{\ra}{\rightarrow}

\newcommand{\B}{\mathcal{B}}
\newcommand{\cov}{\text{\rm Cov}}
\newcommand{\var}{\text{\rm Var}}
\newcommand{\diag}{\text{\rm diag}}
\newcommand{\PP}{\mathbb{P}}
\newcommand{\EE}{\mathbb{E}}
\newcommand{\Ec}{\mathcal{E}}

\newcommand{\X}{\mathbf{X}}

\newcommand{\W}{\mathbf{W}}
\newcommand{\tX}{\tilde{ \mathbf{X}}}
\newcommand{\tE}{\tilde{ \mathbf{E}}}
\newcommand{\tZ}{\tilde{ \mathbf{Z}}}
\newcommand{\R}{\mathbb{R}}
\newcommand{\y}{\mathbf{y}}
\newcommand{\Z}{\mathbf{Z}}
\newcommand{\Eps}{\bm{\eps}}
\newcommand{\Eta}{\bm{\eta}}
\newcommand{\tEps}{\tilde{\bm{\eps}}}
\newcommand{\E}{\mathbf{E}}
\newcommand{\wh}{\widehat}
\newcommand{\eps}{\varepsilon}

\renewcommand{\i}{\infty}
\newcommand{\wt}{\widetilde}
\renewcommand{\a}{\wh\alpha}
\newcommand{\ab}{\alpha^*}

\newcommand{\sep}{\sigma_\eps}
\newcommand{\tr}{\textrm{\rm tr}}

\newcommand{\ke}{\kappa({\se})}

\newcommand{\sk}{\sigma_K}
\newcommand{\sn}{\sigma_n}
\newcommand{\lk}{\lambda_K}
\newcommand{\lp}{\lambda_p}

\renewcommand{\ln}{\lambda_n}
\newcommand{\se}{\Sigma_E}
\newcommand{\sza}{A\Sigma_Z A^\top}
\newcommand{\sxy}{\Sigma_{Xy}}
\newcommand{\sz}{\Sigma_Z}
\newcommand{\sx}{\Sigma_X}

\newcommand{\rank}{\text{\rm rank}}

\newcommand{\sy}{\sigma_y}

\newcommand{\N}{\mathcal{N}}

\newcommand{\kaz}{\kappa(A\Sigma_Z A^\top)}
\newcommand{\V}{V}
\newcommand{\A}{\mathcal{A}}
\newcommand{\F}{\mathcal{F}}
\newcommand{\D}{\mathcal{D}}
\newcommand{\bA}{\bar A}
\newcommand{\bG}{\bar G}
\newcommand{\bb}{\bar\beta}
\newcommand{\sg}{\Sigma}
\newcommand{\seta}{\sigma_\eta}
\newcommand{\bpcr}{\hat\beta}
\newcommand{\rpcr}{R_{\text{\rm PCR}}}

\newcommand{\sbt}{\,\begin{picture}(-1,1)(-0.5,-2)\circle*{2.3}\end{picture}\ }

\title{
Interpolating Predictors in High-Dimensional Factor Regression  }

\author{Florentina Bunea\thanks{Department of Statistics and Data Science, Cornell University, Ithaca, NY 14850, USA. E-mail: \texttt{\href{mailto:fb238@cornell.edu}{fb238@cornell.edu}}. Partially supported by NSF DMS-1712709.}\and 
Seth Strimas-Mackey\thanks{Corresponding author. Department of Statistics and Data Science, Cornell University, Ithaca, NY 14850, USA. E-mail: \texttt{\href{mailto:scs324@cornell.edu}{scs324@cornell.edu}}. Partially supported by NSERC PGS-D.}\and
Marten Wegkamp\thanks{Department of Mathematics and Department of Statistics and Data Science, Cornell University, Ithaca, NY 14850, USA. E-mail: \texttt{\href{mailto:mhw73@cornell.edu}{mhw73@cornell.edu}}. Partially supported by NSF DMS-1712709.}
}
\date{}
\begin{document}

	\maketitle
\begin{abstract}
    
      This work studies  finite-sample properties of the risk of the minimum-norm interpolating predictor in high-dimensional regression models.   If the effective rank of the covariance matrix $\Sigma$ of the $p$ regression features is much larger than the sample size $n$,  we show that the min-norm interpolating  predictor is not desirable, as its risk approaches the risk of trivially predicting the response by 0. However, 
    our detailed finite-sample analysis reveals, surprisingly,  
    that  this behavior is not present when  the regression response and the features are {\it jointly} low-dimensional, following a widely used  factor regression model. 
    Within this popular model class, and when the effective rank of $\Sigma$ is smaller than $n$, while still allowing for $p \gg n$, both the bias and the variance terms of the excess risk can be controlled, and the risk of the minimum-norm interpolating predictor approaches optimal benchmarks. Moreover, through a  detailed analysis of the bias term, we exhibit model classes under   which our upper bound on the excess risk approaches zero, while the corresponding upper bound  in the recent work  \cite{bartlett2019} diverges. Furthermore,  we show that the minimum-norm interpolating predictor analyzed under the factor regression model, despite being model-agnostic and devoid of tuning parameters, can have similar risk to predictors based on principal components regression and ridge regression, and  can improve over LASSO based predictors, in the high-dimensional regime.\\
    

\end{abstract}

 \section{Introduction} 

 Motivated by the widely observed phenomenon that interpolating deep neural networks generalize well despite having zero training error, there has been a recent wave of literature showing that this is a general behaviour that can occur for a variety of models and prediction methods \cite{hastie2019, feldman2019does, muthukumar2019harmless, mei2019generalization, Belkin15849, belkin2018understand, 
 belkin2018overfitting, belkin2019models, belkin2018does, jun2019kernel, mitra2019understanding , ma2017power, liang2019, xing2018statistical, bartlett2019}.
%  {\tiny One of the simplest settings
%  %in which interpolation can be studied
%  is the prediction of a real-valued response $y\in\R$ from vector-valued features $X\in \R^p$ via a linear predictor $\wh y_x \coloneqq X^\top \a$ with $\a$ defined as the vector with the smallest  Euclidean norm among all weight vectors that perfectly fit the training data $(\X,\y)$.}

 One of the simplest settings
 %in which interpolation can be studied
 is the prediction of a real-valued response $y\in\R$ from vector-valued features $X\in \R^p$ via generalized least squares (GLS). The GLS estimator $\a=\X^+ \y$ is based on the Moore-Penrose pseudo-inverse of the $n\times p$ data matrix $\X$ and response vector $\y\in \R^n$, obtained from $n$ i.i.d.~copies $(X_i,y_i)$, $i\in [n]$, 
of $(X, y)$, with $p > n$. It coincides with the minimum-norm estimator, which in the case that $\X$ has full rank, interpolates the data. The interpolation property of $\a$ means that $\X\a=\y$.  We refer to the corresponding predictor as the minimum-norm interpolating predictor.

%  {\tiny
% The data consists of the $n\times p$ data matrix $\X$ and response vector $\y\in \R^n$, obtained from $n$ i.i.d.~copies $(X_i,y_i)$, $i\in [n]$, 
% of $(X, y)$, with $p > n$. The interpolation property of $\a$ means that $\X\a=\y$.  We refer to the corresponding predictor as the minimum-norm interpolating predictor.}
% %

This paper is devoted to  the  finite-sample statistical analysis of prediction via the generalized least squares estimator $\a$. %%minimum-norm interpolator $\a$.
We first note that ideally, 
the prediction risk $ R(\a):= \EE_{X,y}\left[(   X^\top \a-y)^2\right]$  of $\a$ %would %asymptotically 
approaches the optimal risk $   \inf_{\alpha\in\R^p} \EE_{X,y} \left[ (X^\top \alpha- y)^2\right]$. %  that is achieved by   $\ab=\sx^{+} \sxy$.  
Unfortunately, 
this  often turns out not to be the case.
%
%Theorem \ref{thm:a norm} of Section \ref{sec:interpolation_null_risk} below shows that    $\|\a\|$ decreases to $0$ as $n$ increases and $p \gg n$ under appropriate conditions on the covariance matrix $\sx$ of $X$ and variance $\sy^2$ of $y$.
% Furthermore, if $\|\a\|$ is very close to $0$, one might expect $R(\a)$ to be approximately equal to $R(\0)$, with $\0\in \R^p$. This would be undesirable as $R({\bf 0})= \EE[ y^2]:=\sigma_y^2$ is  the non-optimal null risk of trivially predicting via the zero weight vector, ignoring the data.}
%one might expect   the risk of $\a$  to have the undesirable property of approaching the non-vanishing null risk $R({\bf 0})= \EE[ y^2]:=\sigma_y^2$, which is the risk of predicting via the zero weight vector $\alpha={\bf 0}\in\R^p$. 
Theorem \ref{thm:alpha null risk}, stated in Section \ref{sec:interpolation_null_risk},  %confirms that indeed
proves that the ratio $R(\a)/ R(\0 )$ approaches $1$   in the regime 
${\rm r_e} (\sx) \gg n$. 
 Clearly, this is undesirable as $R({\bf 0})$ is 
the non-optimal
null risk of trivially predicting via the zero weight vector, %utterly
ignoring the data.
The {\it effective rank}   %with 
${\rm r_e} (\sx)$ 
%denoting  the effective rank 
of the $p\times p$ covariance matrix $\sx$ of $X$ 
%,
is
defined as the ratio between the trace of $\sx$ and its operator norm, and is at most equal to its rank,  ${\rm r_e} (\sx) \leq p$.
%Theorem \ref{thm:alpha null risk} shows, in 
In particular, if $\sx$ is well-conditioned, with
%effective rank close to its rank (
${\rm r_e} (\sx) \asymp p$, then the prediction risk $R(\a)$ of the minimum norm interpolator approaches the trivial risk $R(\0 )$, whenever $p \gg n$. This was previously observed, from a different perspective, in \cite{hastie2019}.
% A similar phenomenon has been observed, from a different perspective, in the recent work   \cite{bartlett2019}, via a finite sample treatment of the risk, 
%discussed in detail in Section \ref{sec:compare} below, 
% and in \cite{hastie2019}, via an asymptotic analysis of the risk.

%{\color{red} TO DO: call $\a$ the GLS in general, and say that in some cases (when $\X$ is full rank), it coincides with min-norm interpolator. Go through document and make sure to only call it min-norm interpolator when it actually interpolates.}

This opens the question as to whether, 
in the high-dimensional $p>n$ setting, there exist underlying distributions of the data that allow $R(\a)$ to be close  to an optimal risk benchmark. The recent work \cite{bartlett2019} provides a positive answer to this question, primarily focusing on sufficient conditions on the spectrum of $\sx$ that can lead to consistent prediction. 

In this paper we show that the {\it joint} structure of $(X,y)$, not just the marginal structure of $X$ as considered in \cite{bartlett2019}, is important to understanding the conditions under which consistent prediction is possible with $\a$. In particular, we provide a detailed and novel finite-sample analysis of the prediction risk $R(\a)$ 
% under a class of factor regression models that are   extremely relevant in practical,  $p>n$ scenarios.
% This work  provides a positive answer to this question, 
when  the pair $(X,y)$ follows a linear factor regression model,  $y= Z^\top \beta + \eps$, $X = AZ + E$, in the regime 
\[p\gg n \quad \text{ but } \quad {\rm r_e} (\sx) < c\cdot n,\]
% {\color{blue} $${ p \gg n,  \ \ \mbox{but} \ \ \rm r_e} (\sx) < c\cdot n,$$  } 
for an absolute constant $c>0$. Here $(X,y)\in \R^p\times \R$ are   observable random features and response, $Z\in \R^K$ is a vector of unobservable sub-Gaussian random latent factors with $K <  p$, $A\in \R^{p\times K}$ is a loading matrix relating $Z$ to $X$, and $E$ and $\eps$ are mean-zero sub-Gaussian noise terms independent of $Z$ and each other. Under this model,   the observation made in inequality (\ref{eqn:re x fm}) of Section \ref{sec:effrank} below shows 
%to observe 
that ${\rm r_e}(\sx) $
%\le K (1+\xi^{-1})+ {\rm r_e}(\se) \xi^{-1}$, which 
is less than $c\cdot n$ as long as $K<c_1\cdot n$ and the signal-to-noise ratio $\xi:= \lambda_K(\sza)/\|\se\| \gtrsim p/n \ge c_2\cdot{\rm r_e}(\se)/n$ for suitable absolute constants $c_1,c_2>0$.  % where
Here $\sz$ and $\se$ denote the covariance matrices of $Z$ and $E$ respectively, and $\xi$ is the ratio between the $K$th eigenvalue of $\sza$ and the operator norm of $\se$.
{ Section \ref{sec:frm} is dedicated to deriving population-level properties of the factor regression model that are relevant to the performance of the GLS $\a$.}
% Our results show that $R(\a)$ does not approach the trivial risk $R(\0)$ and, moreover, that the excess risk relative to optimal risk benchmarks, discussed in  Section \ref{sec:bench}, can indeed approach zero.  {\color{red}perhaps one more sentence here}
 
%Our main contribution, presented
Our primary contribution is the 
study of  $R(\a)$ under the factor regression model, and in this regime. In Section \ref{section:min_ell_2} we present a detailed  finite-sample study of the risk $R(\a)$ of the %model-free
model-agnostic interpolating predictor $\wh y_x=X^\top \a$ in factor regression models with $p > n$ and  $K < n$, but with $K$ allowed to grow with $n$. 
Our main result is Theorem \ref{thm:upper bound} in Section \ref{main}. It  provides a finite-sample   bound on the {\it excess  risk } $R(\a)-\sep^2$ of $\a$ in the high-dimensional setting $p>n$, relative to  the natural risk benchmark  $\EE[ \eps^2] := \sep^2$ in  the factor regression model; the excess risk relative to the benchmark $\inf_{\alpha\in\R^p} \EE_{X,y} \left[ (X^\top \alpha- y)^2\right]$ is also derived in this theorem. 
As a consequence, we obtain sufficient conditions under which  the prediction risk $R(\a)$ approaches the optimal risk, by adapting to the embedded  dimension $K$.
The excess risk not only decreases beyond the interpolation boundary to a non-zero value as observed in
\cite{hastie2019},
but does indeed decrease to zero, as desired. We remark that at least for Gaussian $(X,y)$, \cite{bartlett2019} provides an alternative bound to Theorem \ref{thm:upper bound}. However, Theorem \ref{thm:upper bound} provides an improved rate for typical factor regression models, and in particular provides examples when the upper bound on the excess risk in \cite{bartlett2019} diverges, yet  our results show that prediction is consistent; see Section \ref{sec:compare} for a detailed comparison.

Table \ref{table:results summary} below offers a snap-shot of our main results. The first row is a reminder that all results are established for $p > n$, while the second row separates the regimes of $r_e(\sx)$ larger or smaller than $n$.
%, where $C>1$ and $c > 0$ are absolute constants with $C>c$. 
The third row specifies the assumptions on $(X,y)$, namely sub-Gaussianity or, in addition, the factor regression model.
The last  row gives finite-sample bounds. 
The risk bounds in the bottom right panel are stated under the  assumptions that 
%the signal-to-noise ratio $\xi:=\lambda_K(\sza)/ \| \se\|$
  the operator norms $\|\sz\|$  and $\|\se\|$ are constant %,
 % the condition number of $\sza$ is bounded
 and
   ${\rm r_e}(\se)\asymp p$. These simplifying assumptions are made here for transparency of presentation and
are not made in the  body of the paper.
\begin{table}[h!] \label{table:results summary}
\centering
    \begin{tabular}{|c|c|}
    \hline
    \multicolumn{2}{|c|}{\rule{0pt}{0.35cm}$p > n$\rule[-0.2cm]{0pt}{0.2cm}}\\
    \hline
    \rule{0pt}{0.45cm}$ {{\rm r_e} (\sx)}{}> C \cdot n$ \rule[-0.3cm]{0pt}{0.3cm}& \multicolumn{1}{c|}{\rule{0pt}{0.45cm} $ {{\rm r_e} (\sx)}{}< c\cdot n$,  \ \   $K < n$\rule[-0.3cm]{0pt}{0.3cm}}\\
    \hline
      $(X,y)$ sub-Gaussian & \makecell{\rule{0pt}{0.5cm}$(X,y)$ sub-Gaussian\\$y = \beta^\top Z + \eps$\\
        $X = AZ+E$\rule{0pt}{0.4cm}\rule{0pt}{0.4cm}\rule[-0.3cm]{0pt}{0.3cm}}\\
    \hline 
    $\left |\frac{R(\a)}{R(\0 )} - 1\right|\lesssim \sqrt{ {n}/ {{\rm r_e} (\sx)}}$ & 
    \makecell{\rule{0pt}{0.6cm}  $R(\a)-\sep^2\lesssim B_Z + V$ \\ $ B_Z= \|\beta\|^2\cdot p/(n\cdot \xi)$\rule{0pt}{0.45cm}\\ 
    $V=\left\{ ({n}/{p}) +({K}/{n})\right\}\log n$\rule{0pt}{0.45cm}\rule[-0.4cm]{0pt}{0.4cm} }\\
    \hline
    \end{tabular}
    \caption{Behavior of risk $R(\a)$. Here $C>1,c>0$ are absolute constants with $C>c$. (i) $R(\a)$ approaches null risk $R(\0)$ for well-conditioned matrices $\sx$ when  $ p\gg n$ (left panel); (ii) Variance term vanishes when $ p\gg n\log n $ and $K\log n \ll n$; Bias term
    vanishes for $\xi := \lambda_K(\sza)/ \| \se\|\gg \|\beta\|^2 p /n$
    (right panel).}
\end{table}
The bottom right panel shows that the variance term $V$ decreases if $ p\gg n\log n $ and $K\log n \ll n$ and that the bias term $B_Z$   decreases provided  that the signal-to-noise ratio $\xi := \lambda_K(\sza)/ \| \se\|$  is large enough. Specifically, we need that $\xi\gg \|\beta\|^2 p /n$, which for $\| \beta\|^2 \lesssim{K}$ amounts to $\xi \gg p\cdot K /n$.
For instance, as explained in Section \ref{sec:effrank}, a common, natural situation is
$\xi\asymp p$ and the bias is small for $K\ll n$.
In clustering problems where the $p$ coordinates of $X$ can be clustered in $K$ groups of approximately eqal size $m\approx p/K$ as discussed in Section \ref{sec:effrank}, 
we find
$\xi \asymp p/K$. In that case, $B_Z$ %is of order $O(K^2/n)$, which
vanishes if $n\gg K^2$.

We emphasize that a condition on the effective rank of $\sx$ alone is not enough to guarantee that $R(\a)$ is close to the optimal risk $\sep^2$. As argued in Section \ref{sec:lin model}, if we assume the model $X = AZ + E$, but instead of assuming that $y$ is also a function of $Z$,  as in this work, we   have a standard linear model
$y= X^\top \theta + \eta$, with $\theta \in \R^p$, then the bias term {\it cannot} be ignored, unless $\|\theta\|\to0$, which is typically not the case in high dimensions. In Section \ref{sec:low dim} we show that the best linear predictor $\ab = \sx^+\sxy$,  that minimizes the risk $\EE_{X,y} \left[ (X^\top\alpha - y)^2 \right]$, does in fact satisfy $\|\ab\|\to 0$ under the factor regression model  $y=Z^\top\beta + \eps$ and thus that this is a natural setting for studying when the GLS generalizes well. From this perspective, this work illustrates the critical role played in the risk analysis by a modeling assumption in which $(X,y)$ are jointly low-dimensional.

Finally, we remark that prediction under factor regression models has been well studied, starting with classical factor analysis that can be traced back to the 1940s
\cite{J67,J69,J70,J77,L40,L41,L43}, including the pertinent work \cite{anderson1956}.
A number of works
  ranging from purely Bayesian \cite{ag2000, dunson2011, hahn, carvalho2008high}  to variational Bayes \cite{blei:var}   to frequentist
\cite{bing2019essential,fan2013, fan2011, fan2013large,  fan2017, Izenman-book, Jolliffe, SW2002_JASA, SW2002_JB, SW2012}
 show that this class of models  can be a useful  framework for constructing and analyzing predictors of $y$ from high-dimensional and  correlated data. %, when $p > n$ and the noise component is small relative to the signal. 
  The literature on finite-sample prediction bounds under factor regression models is relatively limited, with instances provided by  
  \cite{bing2019essential,fan2013, fan2011, fan2013large,  fan2017}, and most existing results established for $K$ fixed. 
  %Most importantly 
  Relevant for the work presented here, the (non-Bayesian) prediction schemes that have been studied in generic factor regression models  are often variations of principal component regression in $K < n$ fixed  dimensions, and therefore   typically do not interpolate the data.  From this perspective, the results of this paper complement this existing literature, by studying the behavior of interpolating predictors in factor regression. Furthermore, in Section \ref{sec:pcr} we derive an upper bound on the excess risk of prediction based on principal components, under the factor regression model, and find that it is comparable to the excess risk bound of the interpolating predictor, in the regime $p\gg n$, provided that the covariance matrix $\se$ of the noise is well conditioned. This provides further motivation for the use of $\a$ in the setting discussed here. \\

  \noindent
  The  rest of the paper is organized as follows.
  
  Section \ref{sec:interpolation_null_risk} derives sufficient conditions on $\sx$ and $\sy^2 \coloneqq \EE[y^2]$ under which 
%   $\|\a\|$ approaches zero,   and related conditions  that imply
  $R(\a)$ approaches the trivial risk $R(\0 )$. This section motivates the remainder of the paper, in which  we study  the risk behaviour when these conditions are violated.
  
 Section \ref{sec:frm} introduces the factor regression model (\ref{model}) and derives population-level properties that are relevant to the performance of the GLS $\a$.
%  The latter reveals what  key quantities to control in order to obtain  non-trivial prediction risk bounds associated with the GLS estimate $\a$. Target risk benchmarks then are introduced in Section \ref{sec:bench}. 
 Bounds on the effective rank and spectrum of $\sx$ under  (\ref{model})  are given in Section \ref{sec:effrank}, and reveal what  key quantities to control in order to obtain  non-trivial prediction risk bounds associated with the GLS estimate $\a$. Target risk benchmarks then are introduced in Section \ref{sec:bench}.

Section \ref{sec:low dim} investigates at the population level the properties of the best linear predictor $\ab=\sx^+ \sxy$, under the factor regression model.
  We demonstrate the interesting phenomenon that under model (\ref{model}),  $\|\ab\|\to 0$ and yet $R(\ab)/R(\0)\not\to1$. 
%   {\color{red} and in fact does something good }
 %We contrast this paradox with the standard regression model $y=\theta^\top X +\eps$ setting.
 We argue that this is in contrast to the behaviour of the best linear predictor $\theta$ in a standard linear regression model in which $\EE[y|X] = X^\top\theta$ and typically $\|\theta\|$ is fixed or growing with $p$. We give a comparison between factor regression and standard linear regression  in Section \ref{sec:lin model}, commenting on assumptions on the operator norm of $\sx$, and on implications for prediction with the GLS.

%  We then argue that the best linear predictor typically does not exhibit this paradoxical behavior, should the data have been generated instead  by a  standard linear regression model $y=\theta^\top X +\eps$.{\color{red} then what does it do} 

 The remainder of the paper, Section \ref{section:min_ell_2}, contains our analysis of the GLS $\a$ and its prediction risk, under the factor regression model. Section \ref{sec:noiseless} gives a preview of our main findings. In the noiseless case $\se=0$, we have that $\|\a\|\to 0$ (just like $\|\ab\|\to 0$), but $R(\a)-R(\ab) $ achieves the parametric rate $K/n$, up to a $\log n$ factor. In fact, we establish $X^\top \a= Z^\top \wh \beta$ for the least squares estimate $\wh \beta$ based on observed $(\Z,\y)$. 
 
 Section \ref{main} contains our main results in the more realistic setting $\se\ne0$. It establishes when $\a$ interpolates, and shows that typically $\|\a\|\to0$, as in the noiseless case.  Furthermore, in agreement with the findings in Section \ref{sec:noiseless}, $R(\a)/R(\0)$ does not approach $1$. Instead,
  the finite-sample risk bound  in Theorem \ref{thm:upper bound} shows that under appropriate conditions on ${\rm r_e}(\se)$ and 
   the signal-to-noise ratio $\xi$, the  excess risk $R(\a)-R(\ab)$ converges to zero.
  %regarding the approximate adaptation of the minimum-norm interpolating predictor in this model
  
  %that prevent $\|\a\|$ from approaching  zero. {\marten This is plain wrong!} {\seth yep :-))) }
  
 % In Section \ref{sec:cluster}
  % we present the analysis of  interpolating predictors in  a particular factor regression model, motivated by regression with clustered design. \\

   Section \ref{sec:compare} presents a 
   comparison with
   recent related work. 
   %Another illustration of the impact of the bias term on the overall prediction accuracy of $\wh y_x$ is given in Section \ref{sec:compare}, where we contrast our findings with
%   \cite{bartlett2019},    % Bartlett {\color{red} add quote}. 
   In particular, we give a detailed comparison with \cite{bartlett2019}, which provides risk bounds for  $\wh y_x= X^\top \a$, for sub-Gaussian data $(X, y)$, and offers sufficient conditions on $\Sigma_X$ for optimal risk behavior, with emphasis on the optimality of the  variance component of the risk. We present simplified versions of the generic bias and variance bounds obtained in \cite{bartlett2019} under the factor regression model, which are derived in Appendix \ref{sec:finite sample}. Table \ref{table:bart} of Section \ref{sec:compare} summarizes our findings that the bound on the excess risk in \cite{bartlett2019} 
 is often larger in order of magnitude than the bound given in   Theorem  \ref{thm:upper bound} of Section \ref{main}. In particular, we exhibit instances of the factor regression model class under which the excess risk upper bound in \cite{bartlett2019} diverges, yet our upper bound approaches zero. We also compare our work to \cite{mei2019generalization}, which gives an asymptotic analysis of the ridge regression estimator with arbitrarily small (but non-zero) regularization for a type of factor regression model.
   
%   {\color{red} Somewhere here we should mention the $\| \sx \|$ finite vs allowed to grow contrast + consequences.} \\

%Consequently, we provide a new set  of sufficient conditions on data generating mechanisms, milder under the regime considered,  that  guarantee optimal risk behavior of the minimum-norm interpolating predictor. 

 Section \ref{sec:pcr} is devoted to   a comparison with prediction via  principal component regression  and $\ell_1$ and $\ell_2$ penalized least squares, under the factor regression model. 
 
 All proofs   and ancillary results  are deferred to the Appendix. 
    %   {\tiny In particular, as a supplement to Section \ref{sec:interpolation_null_risk}, a quasi-heuristic geometric explanation of why $\|\a\|$ approaches  0,  under suitable conditions on $\sx$ and $\sy^2$ is given in Appendix \ref{sec: geometry}.} 
      In particular, Theorem \ref{LS} in the Appendix complements Theorem \ref{thm:upper bound} by showing the risk behavior of $\a$ for $n > c\cdot p$ for an absolute constant $c>0$, and is included for completeness.

\subsection{Notation}
Throughout the paper, for a vector $v\in \R^d$, $\|v\|$ denotes the Euclidean norm 
%and $v_{(1)},\ldots, v_{(d)}$ the components 
of $v$.\\
For any matrix $A\in \R^{n \times m}$, $\|A\|$ denotes the operator norm and $A^+$ the Moore-Penrose pseudo-inverse.   See Appendix \ref{sec:pseudo-inverse} for a definition of the pseudo-inverse and a summary of its properties used in this paper. \\
For a positive semi-definite matrix $Q\in \R^{p\times p}$, and vector $v\in\R^p$, we define $\|v\|_Q^2 \coloneqq v^\top Qv$, let $\lambda_1(Q) \ge \lambda_2(Q) \ge \cdots \ge \lambda_p(Q)$ be its ordered eigenvalues, $\kappa(Q) :=\lambda_1(Q)/\lambda_p(Q)$ its condition number, and  
${\rm r_e} (Q):= \textrm{tr}(Q) / \| Q\|$ its effective rank.
\\
%If $B$ is positive semi-definite let $\lambda_1(B)\ge \cdots\ge\lambda_m(B)$ denote the $m$ eigenvalues of $B$. For arbitrary $B$, let $\sigma_i(B) \coloneqq \lambda_i(B'B) = \lambda_i(BB')$ for $i=1,\ldots,\min(m,n)$ denote the singular values of $B$.
%
%
%Its effective rank is defined as ${\rm r_e} (Q):= \textrm{tr}(Q) / \| Q\|$.
The identity matrix in dimension $m$ is denoted $I_m$.
\\
The set $\{1,2,\ldots, m\}$ is denoted $[m]$.\\
Letters $c$, $c'$, $c_1$, $C$, etc., are used to denote absolute constants, and may change from line to line.

\section{Interpolation and the null risk}\label{sec:interpolation_null_risk}
Given i.i.d.~observations $(X_1,y_1),\ldots,(X_n,y_n)$, distributed  as $(X,y)\in \R^p\times\R$, let $\X \in \R^{n\times p}$ be the corresponding data matrix with rows $X_1,\ldots X_n$, and let $\y \coloneqq (y_1,\ldots,y_n)^\top  \in \R^n$. For the rest of the paper, unless specified otherwise, we make the blanket assumption that $p > n$.\\
%  and  {\seth [REMOVE?]$\sigma_y^2 := \EE[y^2] >0$.}{\marten Yes, I vote to move it to Assumption 1. }

We are interested in studying the prediction risk associated with the minimum $\ell_2$-norm estimator $\a$ defined as
\begin{equation}\label{eqn:min-def}
    \a \coloneqq \arg\min\left\{\|\alpha\|:\ \|\X\alpha - \y\| = \min_u \|\X u - \y\|\right\}.
\end{equation}
We define the prediction risk for any $\alpha \in \R^p$ as
\begin{equation}
    R(\alpha) \coloneqq \EE_{X,y}[(X^\top \alpha - y)^2].
\end{equation}
The expectation is over the new data point $(X,y)$, independent of the observed data $(\X,\y)$. In particular, since $\a$ is independent of $(X,y)$, we have $R(\a)= \EE_{X,y}\left[ (X^\top \a - y)^2\,  | \, \X,\y \right] = \EE_{X,y}\left[ (X^\top \a - y)^2\, \right]$.
If the data matrix $\X$ has full rank of $n <p$, then $\min_{u\in \R^p} \|\X u - \y\|=0$  and
\begin{equation}\label{eqn:min-def-fullrank}
    \a \coloneqq \arg\min_{\alpha:\ 
    \X\alpha = \y} \|\alpha\|.
\end{equation}
Regardless of the rank of $\X$, Equation (\ref{eqn:min-def}) always has the closed form solution $\a = \X^+
\y$, where $\X^+$ is the Moore-Penrose pseudo-inverse of $\X$; we prove this fact in section \ref{sec:closed form a} for completeness.
We begin our consideration of the minimum-norm estimator $\a = \X^+\y$ by showing that its risk $R(\a)$ approaches the null risk $R(\0)$ whenever the effective rank ${\rm r_e}(\sx)$ grows at a rate faster than $n$.
% We begin our consideration of the minimum-norm estimator $\a = \X^+\y$ by showing that when the effective rank ${\rm r_e}(\sx)$ is large enough and $\tr(\sx)$ grows at a rate faster than $n\cdot \sy^2$, the norm $\|\a\|$ approaches zero as $n$ grows.
% %\to 0$ as $n\to \i$. 
Proofs for this section are contained in Appendix \ref{proofs:null}. We make the following distributional assumption.

\begin{ass}\label{ass:x}
$X = \sx^{1/2}\tilde X$ and $y = \sy\tilde y$, where $\tilde X\in\R^p$ has independent entries, and both $\tilde X$ and $\tilde y$  have zero mean, unit variance, and sub-Gaussian constants bounded by an absolute constant.
\end{ass}

\begin{thm}\label{thm:alpha null risk}
Suppose Assumption \ref{ass:x} holds and ${\rm r_e} (\sx) > C\cdot n$ for some absolute constant $C> 1$ large enough.
% Then there exists an absolute constant $c_2>1$ such that if ${\rm r_e} (\sx) > c_2n$ then 
Then, with probability at least $1-c{\rm e} ^{-c'n}$ for absolute constants $c,c'>0$,
\begin{equation}
    \left|\frac{R(\a)}{R(\0 )} - 1\right| \lesssim
    \sqrt{\frac{n}{{\rm r_e} (\sx)}}.
\end{equation}
\end{thm}
% \begin{proof}
% See section \ref{proof:alpha null risk}.
% \end{proof}
As a consequence,
%  Theorem \ref{thm:alpha null risk} shows that whenever ${\rm r_e} (\sx)/n \rightarrow \infty$ the risk $R(\a)$ converges to the risk $R(\0 ) = \sigma_y^2$ of the null predictor $0\in \R^p$, thus showing that  
 $\a$ is not a useful estimator in the regime $r_e(\sx)\gg n$, as trivially predicting with the null vector $\0\in \R^p$ will give asymptotically equivalent results. 
 This occurs, for instance, when $\sx$ is well conditioned and $p/n\ra \i$. Figure 2 in  \cite{hastie2019} depicts an example of this behavior: it plots   $\EE[\|\a - \alpha\|^2|\X]$ as a function of the ratio $\gamma =p/n$, where  $(X,y)$ follows the linear model $y=\alpha^\top X + \eps$ with $\sx =I_p$.

This motivates the study of $R(\a)$ when  the condition ${\rm r_e} (\sx) >   C\cdot n$   of Theorem \ref{thm:alpha null risk} fails.  The recent work \cite{bartlett2019} developed bounds for the excess risk $R(\wh \alpha) - \inf_{\alpha\in \R^p}R(\alpha)$ under the linearity assumption $\EE[y|X]= X^\top \theta$ (for some $\theta\in \R^p$), and used this to show that the excess risk goes to zero for a certain class of \textit{benign} covariance matrices that in particular satisfy ${\rm r_e}(\sx)/n\to 0$ and $\|\sx\|=1$.
%In this framework, \cite{bartlett2019} show that $R(\wh \alpha)$ %can in fact converge to the optimal risk $ \inf_{\alpha\in %\R^p}R(\alpha)$,  under additional assumptions. 

In this work we are interested in obtaining risk bounds for $R(\a)$ under a different model, the factor regression model  (\ref{model}) given  below. In this model, while ${\rm r_e} (\sx)/n$ remains bounded, $\|\sx\|$ typically grows with $p$ (see Lemma \ref{thm:spectrum} below), in contrast to the assumption $\|\sx\|=1$ of the definition of benign matrices in \cite{bartlett2019}. Furthermore, the results in \cite{bartlett2019} only apply to model (\ref{model}) when $(X,y)$ are assumed to be jointly Gaussian. In this case, their bound offers an alternative result, which we compare to our main result in Section \ref{sec:compare} below. We find that in this common regime, we obtain a tighter bound.

\section{Factor regression models}\label{sec:frm}

%%%\subsection{Factor Regression Model}
%The result and discussion of the previous section imply  that in order for the generalized least squares estimator $\a$  to have   asymptotically better prediction performance than the trivial estimator $\0\in \R^p$,  the ratio  ${\rm r_e} (\sx)/n$ must remain bounded as $n$ and $p$ grow, as a first requirement.

%We now introduce a class of models, and associated conditions,
%that guarantee under which this happens. 
%The model class is that  of factor regression models,
In this paper, we consider the factor regression model (FRM). % which 
This 
is a latent factor model in which we single out one variable, $y \in \R$, to emphasize its role as  the response relative to input covariates $X \in \R^p$, while both $X$ and $y$ are directly connected to a lower dimensional, unobserved, random vector $Z \in \R^K$, with mean zero and $K < n$.
Specifically, the factor regression model postulates that 
\begin{eqnarray}\label{model} 
    X =  AZ + E,  \ \quad \ 
    y = Z^\top \beta + \varepsilon, 
\end{eqnarray}
where $\beta\in \R^K$ is the latent variable regression vector, $A\in \R^{p\times K}$ is a unknown loading matrix, and $\eps \in \R$ and $E \in \R^p$ are mean zero additive noise terms independent of one another and of $Z$. We let $\se \coloneqq \cov(E)$, $\sz \coloneqq \cov(Z)$ and $\sep^2 \coloneqq \var(\eps)$.
For the remainder of the paper we will assume that the  
data consist of $n$ i.i.d.~pairs $(X_i, y_i)$ satisfying (\ref{model}), in that 
\begin{equation}
    X_i = AZ_i + E_i,\hspace{1cm} y_i = Z_i^\top \beta+\eps_i\hspace{1cm} \forall i \in [n],
\end{equation}
where  the latent factors $Z_1,\ldots,Z_n\in \R^{K}$ are i.i.d.~copies of $Z$,  and the error terms $E_i\in\R^p$ and $\eps_i\in \R$ for $i=1,\ldots,n$ are i.i.d.~copies of $E$ and $\eps$, respectively. We recall that 
$\X \in\R^{n\times p} $ is  the matrix with rows  $X_1,\ldots,X_n$ and $\y \in \R^n$ is  the vector with entries $y_1,\ldots,y_n$. We similarly let $\Z \in \R^{n \times K}$ be the matrix with rows $Z_1,\ldots,Z_n$.

 The remainder of this section is dedicated to deriving population-level properties of the factor regression model that are relevant to the performance of the GLS $\a$. In particular, we will (1) bound the effective rank of $\sx$, (2) bound the eigenvalues of $\sx$, (3) define two natural risk benchmarks and show when they are asymptotically equivalent, (4) show that the weight vector of the best linear predictor has vanishing norm, and (5) prove that, nonetheless, the null risk $R(\0)$ is clearly sub-optimal. The first two properties reflect the low-rank structure of the covariance matrix $\sx$ and are presented in Section \ref{sec:effrank}. The risk benchmarks are introduced and analyzed in Section \ref{sec:bench}.
 Section \ref{sec:low dim} investigates the properties of the best linear predictor $\ab=\sx^+ \sxy$ at the population level, showing properties (4) and (5).
%   We demonstrate the interesting phenomenon that under model (\ref{model}),  $\|\ab\|\to 0$ and yet $R(\ab)/R(\0)\not\to1$. 
The fourth property in particular is a consequence of the joint low-dimensional structure of $(X,y)$ via the vector of  covariances $\sxy$.
% It is a distinct property of the factor regression model
% The fifth property is also derived in Section \ref{sec:low dim}, and establishes that prediction is non-trivial in this model. 
%  The fact that the norm of the best linear predictor goes to zero is a 
 It is a distinct property of the factor regression model that sets it apart from the classical regression model where the response $y$ is linearly related to $X$ via $\EE [ y|X]= \theta^\top X$.
 We present a comparison between factor regression and classical linear regression in Section \ref{sec:lin model}.

\subsection{Effective rank and spectrum of $\sx$ in the FRM} \label{sec:effrank}
Theorem \ref{thm:alpha null risk} and its discussion above imply  that in order for the generalized least squares estimator
$\a$  to have  
asymptotically better prediction performance than the trivial estimator $\0\in \R^p$, 
the ratio  ${\rm r_e} (\sx)/n$ must remain bounded as $n$ and $p$ grow, as a first requirement.

% Lemma \ref{thm:k=k star} {\color{red} Notation ??}  %\ref{thm:k=k star} 
% in Appendix \ref{proof:k=k star} {\color{red} Notation ??}  %{sec:compare}
% below shows that under model  (\ref{model}),
% \begin{equation} \label{effb}    
% \frac{{\rm r_e} (\sx)}{n} \le 
% \frac{K}{n}(1 + \xi^{-1}) +  \frac{1}{\xi}\frac{{\rm r_e} (\se)}{n},
% \end{equation}
Using that $\sx = \sza + \se$ under (\ref{model}), we find
\begin{align*}
    {\rm r_e}(\sx) &= \frac{\tr(\sx) }{\|\sx\|}\\
    &\le \frac{\tr(\sza) + \tr(\se)}{\|\sza\|}&& (\text{since } \|\sx\|\ge \|\sza\|)\\
    &\le K + \frac{\tr(\se)}{\|\sza\|} && (\text{since } \tr(\sza)\le K\|\sza\|)\\
    &\le K + \frac{\|\se\|}{\lk(\sza)}\cdot\frac{\tr(\se)}{\|\se\|},&&(\text{since } \|\sza\|\ge \lk(\sza))
    % &= K + {\rm r_e}(\se)/\xi.
\end{align*}
where we use the convention that $\tr(\se)/\|\se\| = {\rm r_e}(\se) = 1$ if $\se=0$. We thus have
\begin{equation}\label{eqn:re x fm}
    \frac{{\rm r_e}(\sx)}{n} \le \frac{K}{n} + \frac{1}{\xi}\frac{ {\rm r_e}(\se)}{n},
\end{equation}
where
\begin{equation}
     \xi \coloneqq \lk(\sza)/\|\se\| \label{xi},
\end{equation}
can be viewed as a signal-to-noise ratio since $\Sigma_X = A \sz A^\top + \se$, and we use the convention that $\xi = \i$ and ${\rm r_e}(\se)/\xi = 0$ when $\se=0$.
 In standard factor regression models \cite{anderson1956}, $\se = I_p$, in which case ${\rm r_e} (\se) = p$, but in our analysis %derivations
we allow for a general   $\se$, with possibly smaller ${\rm r_e} (\se)$. 
The following simple result follows directly from (\ref{eqn:re x fm}).
\begin{lemma}\label{thm:SNR cond}
Under model (\ref{model}), we have ${\rm r_e} (\sx)/n \leq c_3$ whenever
\begin{eqnarray}\label{step1}
\frac{K}{n} \leq  c_1 \quad \  \text{and} \ \quad 
\xi \geq c_2\frac{{\rm r_e} (\se)}{n},
\end{eqnarray}
for positive absolute constants $c_1,c_2,c_3$. 
\end{lemma}
% The positive repercussion of Lemma \ref{thm:SNR cond} is that under condition \ref{step1} and for small enough constant $c_3$, Theorem \ref{thm:alpha null risk}   no longer applies. This in turn opens up the possibility of showing that, under the  data generating model (\ref{model}) with restrictions (\ref{step1}), the risk $R(\a)$ will approach optimal  risk benchmarks.  
%  We define the benchmark risks in terms of the best linear predictors of $y$ from $X$ and $Z$, respectively, in Section \ref{sec:bench}, and show that $R(\a)$ can indeed approach these benchmarks in Sections \ref{sec:noiseless} and \ref{main}. 
%  We also investigate the behaviour of the best linear prediction vector $\ab = \sx^+\sxy$ of $y$ from $X$ under the factor regression model in Section \ref{sec:low dim}, and use this in Section \ref{sec:lin model} to clarify
% the importance of
% the factor regression model, in which $(X, y)$ {\it jointly} have a low-dimensional  structure, in contrast to the classical linear model $y=X^\top \theta+\varepsilon$ with low-dimensional structure on $X$ alone. 
% We provide some examples of conditions on $A$, $\sz$, and $\se$ such that (\ref{step1}) holds. 
\begin{remark}\label{rem:een}
We remark on conditions under which (\ref{step1}) holds. Suppose that the eigenvalues of $\sz$ and $\se$ are constant, that is, 
$c_1 \le \lk(\sz) \le \| \sz\| \le C_1$ and $c_2<\lp(\se)\le \|\se\| <  C_2$, for some $c_1,c_2, C_1, C_2\in (0,\infty)$, both standard assumptions in factor models. Then, 
\begin{equation}\label{one} 
{\rm{r_e}} (\se) \asymp p, \hspace{0.5cm}\text{and}\hspace{0.5cm}
\xi= \frac{\lk(\sza)}{\|\se\|}  \asymp \lambda_K(A^\top A),
\end{equation} 
so the condition (\ref{step1}) reduces to $K/n \leq  c_1$ and
\begin{equation}\label{step2}
    \lambda_K(A^\top A) \gtrsim \frac{p}{n}.
\end{equation}
We give a few examples of $A$ that imply (\ref{step2}):
\begin{enumerate}
    \item For a well-conditioned matrix $A \in \R^{p \times K}$  with entries taking values in a bounded interval, $\lambda_K(A^\top A) \asymp  p$, and (\ref{step2}) holds.
    
    \item Treating $A$ as a realization of a random matrix with i.i.d.~entries  and $p \gg K$, then by standard concentration arguments (see \cite{verHDP}, for example) we once again have  $\lambda_K(A^\top A) \gtrsim   p$, with high probability, and (\ref{step2}) holds.
    \item In other situations, (\ref{step2})  is an assumption. It is a very natural, and mild, requirement in factor regression models, and if $A$ is structured and sparse, (\ref{step2}) can be given further interpretation. For instance, the model $X = AZ + E$ has been  used and analyzed in \cite{cord}   for clustering the $p$ components of $X$ around the latent $Z$-coordinates, via an assignment matrix $A\in \{0,1\}^{p\times K}$,  and when $\se$ is an approximately  diagonal matrix. Denoting    the size of the smallest of the $K$ non-overlapping clusters by $m$, for some integer $ 2 \leq m \leq p$, it is immediate to see (Lemma \ref{thm:snr cluster} in Appendix \ref{sec:snr cluster}) that $\lambda_K(A^\top A) \geq m$. Furthermore, when  these $K$ clusters are approximately balanced, then $m \approx p/K$ and (\ref{step2}) holds,  provided $K \lesssim n$. 
\end{enumerate}
\end{remark}
% {\marten Perhaps give examples of $A$ that give $\xi \gg p/n$?\\

% Also: compute $\|\sx\|$? $\|\sx\|\ge \| \sza\| - \|\se\|\ge \|\se\| ( \xi -1) $\\
% E.g. $\tr(\se)=p$ and $c< \|\se\|\le C$, so we need $\xi \gg p/n$, implying $\|\sx\| \gg p/n\to\i$ as well.}
% \begin{remark}
% The condition $\lk(A^\top A)\to \i$ has implications 
% \end{remark}
% {\seth We could compute
% \[\lambda_K(\sx)\ge \lambda_K(\sza)\ge \lk(\sz)\lk(A^\top A).\]
% For many examples of $A$ we have $\lk(A^\top A)\to \i$, and so assuming $\lk(\sz) > c>0$, we have $\lk(\sx)\to \i$, i.e., the first $K$ eigenvalues of $\sx$ diverge. Also, for any $i\in [p]$,
% \[\lambda_i(\sx) \ge \lambda_i(\se) \ge \lp(\se),\]
% so if $\lp(\se) > c>0$, the whole spectrum of $\sx$ is bounded below by $c>0$. Lastly, for $i > K$, 
% \[\lambda_i(\sx) \le \lambda_i(\sza) + \|\se\| = \|\se\|,\]
% since $ \lambda_i(\sza)=0$ for $i>K$. Thus, if $\|\se\|<C$, then $\lambda_i(\sx)<C$ for $i >K$.

% In summary, under the stated conditions,
% \begin{enumerate}
% \item All eigenvalues of $\sx$ are bounded below by $c>0$;
%     \item The first $K$ eigenvalues of $\sx$ diverge;
%     \item The last $p-K$ eigenvalues of $\sx$ are bounded above by $C$.
% \end{enumerate}
% }

The positive repercussion of Lemma \ref{thm:SNR cond} is that under condition (\ref{step1}) and for small enough constant $c_3$, Theorem \ref{thm:alpha null risk}   no longer applies. This in turn opens up the possibility of showing that, under the  data generating model (\ref{model}) with restrictions (\ref{step1}), the risk $R(\a)$ will approach optimal  risk benchmarks. We define the benchmark risks in terms of the best linear predictors of $y$ from $X$ and $Z$, respectively, in Section \ref{sec:bench}, and show that $R(\a)$ can indeed approach these benchmarks in Sections \ref{sec:noiseless} and \ref{main}.
%excess risk will indeed vanish in high dimensions. [excess risk is not defined yet]
%  We define the benchmark risks in terms of the best linear predictors of $y$ from $X$ and $Z$, respectively, in Section \ref{sec:bench}, and show that $R(\a)$ can indeed approach these benchmarks in Sections \ref{sec:noiseless} and \ref{main}. 
%  We also investigate the behaviour of the best linear prediction vector $\ab = \sx^+\sxy$ of $y$ from $X$ under the factor regression model in Section \ref{sec:low dim}, and use this in Section \ref{sec:lin model} to clarify
% the importance of
% the factor regression model, in which $(X, y)$ {\it jointly} have a low-dimensional  structure, in contrast to the classical linear model $y=X^\top \theta+\varepsilon$ with low-dimensional structure on $X$ alone. 

For completeness, we offer the following result characterizing the spectrum of $\sx$ under the factor regression model. In particular, as announced in Section \ref{sec:interpolation_null_risk}, we find that the operator norm $\|\sx\|$ diverges with $p$ under mild conditions. The proof can be found in Appendix \ref{proof:spectrum}.
\begin{lemma}\label{thm:spectrum}
Suppose that for some $c_1,c_2, C_1, C_2\in (0,\infty)$,
\begin{equation}
    c_1 \le \lk(\sz) \le \| \sz\| \le C_1 \quad  \text{ and } \quad c_2<\lp(\se)\le \|\se\| <  C_2.
\end{equation}
The spectrum of $\sx$ can then be characterized as follows:
\begin{enumerate}
    \item $\lambda_i(\sx) \ge c_2 > 0$ for all $i\in [p]$, i.e., the entire spectrum of $\sx$ is bounded below;
   \item $\lk(\sx)\ge c_1\lk(A^\top A)$, so the first $K$ eigenvalues of $\sx$ diverge if $\lk(A^\top A)\to \i$ as $p\to\i$;
    \item $c_2 \le \lambda_i(\sx)\le C_2$ for $i > K$, i.e., the last $p-K$ eigenvalues of $\sx$ are bounded above and below.
\end{enumerate}

\end{lemma}

After introducing the risk benchmarks below, we investigate the behaviour of the best linear prediction vector $\ab = \sx^+\sxy$ of $y$ from $X$ under the factor regression model in Section \ref{sec:low dim}, and use this in Section \ref{sec:lin model} to clarify
the importance of
the factor regression model, in which $(X, y)$ {\it jointly} have a low-dimensional  structure, in contrast to the classical linear model $y=X^\top \theta+\eta$ with low-dimensional structure on $X$ alone. 

\subsection{Risk benchmarks}\label{sec:bench}

% We will compare $R(\a)$ to two natural benchmarks. 
We introduce here two natural benchmarks for $R(\a)$ under the factor regression model, and characterize their relationship.
Under model (\ref{model}), if $Z \in \R^K$ were observed, the optimal risk of a linear oracle with access to $Z$ is %{\color{green}and using a linear predictor} is
\begin{equation}
    \min_{v\in\R^K} \EE\left[(Z ^\top v - y)^2\right] = \EE[\eps^2] = \sep^2,
\end{equation}
which we henceforth refer to as the oracle risk. 
Another natural benchmark to compare the risk $R(\a)$ to is the minimum risk possible for any linear predictor $\alpha^\top X$, namely $R(\ab)$, where 
\begin{equation}\label{star}
    \ab \in \arg\min_{\alpha\in\R^p} R(\alpha).
\end{equation}
%We will refer to $R(\alpha^*)$ as the  linear oracle risk,
%obtained by the best linear predictor $\alpha^*$.
Lemma \ref{thm:ab minimizer} in Appendix \ref{sec:supp} shows that for arbitrary zero-mean $(X,y)$ with finite second moments, $\ab=\sx^+\sxy$ is a minimizer of $R(\alpha)$, where $\Sigma_{Xy} \coloneqq \EE[Xy] \in \R^p$ is the vector of  component-wise  covariances.
% {\marten A simple calculation gives  $$R(\0)-\sep^2 = \| \beta\|^2_{\sz}.$$ This identity reveals that  the null risk $R(\0)$ is sub-optimal compared to the oracle risk $\sep^2$, unless $\|\beta\|_{\sz} \to0$. 
% In the next section, we will reach the conclusion that $R(\0)$ is   sub-optimal compared to $R(\ab)$, unless, again, $\|\beta\|_{\sz} \to0$. Clearly, $\|\beta\|_{\sz}\to0$ is unlikely, as $\beta$ is arbitrary element in $\R^K$, with $K$ possibly diverging.}

We can characterize the difference between these two benchmarks, $\sep^2$ and $R(\ab)$, as follows. See Appendix \ref{proof:bench compare} for the proof of this result.

\begin{lemma}[Comparison of risk benchmarks]\label{thm:bench compare}% \hspace{1cm}
  Suppose model (\ref{model}) holds and let $\xi$ be the signal-to-noise ratio defined in (\ref{xi}). We have
\begin{enumerate}
    \item $R(\ab) -\sep^2 \ge 0$ with equality if $\se=0$.
    {% \seth \item $R(\ab)= \sep^2$ if $\se=0$.

    \item Provided the matrices $\sz$, $\se$, and $A$ are full rank,
%\begin{equation}\label{eqn:bench compare lower}
   \[ \frac{\xi}{1+\xi} \beta^\top (A^\top \se^{-1}A)^{-1}\beta \le R(\ab) - \sep^2 
    \le \beta^\top (A^\top \se^{-1}A)^{-1}\beta,\]
    where
    \[\beta^\top (A^\top \se^{-1}A)^{-1} \beta\le \frac{1}{\xi}\|\beta\|^2_{\sz}.\]}
    In particular, $\|\beta\|^2_{\sz} / \xi \to0 $ implies $R(\ab)-\sep^2 \to0$, as $p\to\i$.

 %  \item {\marten $R(\0)-\sep^2 = \| \beta\|^2_{\sz}$.} {\seth seems out of place here...?}
%\end{equation}
%where $G \coloneqq \sz^{-1} + A'\se^{-1}A$.
%and $\xi \coloneqq \lk(\sza)/\|\se\|$ is the signal-to-noise ratio appearing in Theorem \ref{thm:upper bound}.
\end{enumerate}
\end{lemma}

Although the optimal  risk $R(\ab)$ is always greater than the oracle risk $\sep^2$ (part 1 of Lemma \ref{thm:bench compare}), 
the bound $ {\|\beta\|^2_{\sz}}/{\xi}$ on the difference $R(\ab)-\sep^2 $ in part 2 of Lemma \ref{thm:bench compare} is not a leading term in the excess risk bound given in Theorem \ref{thm:upper bound}. From this perspective, we can view these benchmarks as asymptotically equivalent, but  with different interpretations. 
%Part 3 of Lemma \ref{thm:bench compare} states that the null risk $R(\0)$ is sub-optimal compared to the oracle risk $\sep^2$, unless $\|\beta\|_{\sz} \to0$.
%In the next section, we will reach the conclusion that $R(\0)$ is   sub-optimal compared to $R(\ab)$, unless, again, $\|\beta\|_{\sz} \to0$.
 Interestingly, the condition $\lim_{p\to\i}  {\|\beta\|^2_{\sz}}/{\xi}= 0$ forces $\|\ab\|\to 0$, see Corollary \ref{thm:null vs opt} in the next section. This is an important feature of the FRM, and its  repercussions are discussed in Section \ref{sec:lin model}.\\

\subsection{Best linear prediction 
in factor regression models (population level)}

%Implications of the joint low-dimensional structure of   factor regression models to  best prediction}  
\label{sec:low dim}

In this section we investigate the properties of the population-level predictor $\ab$, defined in (\ref{star}), { under the factor regression model} (\ref{model}).
In particular, we prove that $\|\ab\|\to 0$ and yet $R(\0)-R(\ab)>0$ under the conditions
\begin{equation} \label{cond:snrbeta2}
     \lim_{p\to\i} \|\beta\|^2_{\sz} /\lk(\sza)=0 \ \text{ and } \liminf_{p\to\i}\| \beta\|_{\sz}>0.
\end{equation}
The property $\|\ab\|\to 0$ in particular is a consequence of the joint low-dimensional structure of $(X,y)$ via the covariance $\sxy= A\sz \beta$, which the vector $\ab = \sx^+\sxy$ depends on.
%   used in  Lemma  \ref{thm:bench compare} above to show $R(\ab)-\sep^2\to0  $ and $R(\0) -  \sep^2>0$.
Proofs for this section can be found in Appendix \ref{proofs:low dim}.
We first characterize the norms $\|\ab\|$ and $\|{\ab}\|_{\sx}$; the latter norm is of interest via the identity 
\begin{equation}\label{id:null}
    R(\0)-R({\ab})= \|{\ab} \|^2_{\sx}.
\end{equation}

It is instructive to first consider the simple case of noiseless features, $X = AZ$, with $E=0$. In this case, the best linear predictor of $y$ from $X$ is
${\ab}^\top X =  (A^\top {\ab})^\top Z.$
The following lemma states that $\ab = A^{+\top}\beta$, which by the identity $A^\top A^{+\top} = I_K$ when $A$ is full rank gives
\begin{equation}\label{z=x}
    {\ab}^\top X = (A^\top A^{+\top}\beta)^\top Z = \beta^\top Z,
\end{equation}
showing that the best linear predictor from $X$ reduces to the best linear predictor from $Z$. The lemma then uses this to derive explicit expressions for the norms of $\ab$.

\begin{lemma}\label{thm:ab norm bound e=0}
% {\marten Check if we need $A$ and $\sz$ of full rank} {\seth We do need $\sz$ full rank here. I'm thinking to change the presentation as follows, which then doesn't require $\sz$ full rank.}
%Let $\sx = \sza$ and let  $\ab = \sx^{+} \sxy = \sx^{+}A \sz \beta$.
%
Suppose model (\ref{model}) holds, that $\se=0$, and that $\sz$ and $A$ are full rank. Then, $\ab = A^{+\top}\beta$, and
\[\|{\ab}\|^2_{\sx} = \|\beta\|_{\sz}^2 \quad \text{ and }\quad \|{\ab}\|^2 = \beta^\top (A^\top A)^{-1}\beta.\]

\end{lemma}

We next find that in the more realistic case, when $\se \neq 0$,  even though identity (\ref{z=x}) no longer holds, we can recover the same identities for $\|{\ab}\|_{\sx}$ and $\|\ab\|$, up to constants, when the noise matrix $\se$ is well-conditioned.

\begin{lemma}\label{thm:ab norm bound main}
%Let $\sx = \sza + \se$ with
Suppose model (\ref{model}) holds and that $A$, $\sz$, $\se$ are all full rank. Then, when  
%\ab = \sx^{+} \sxy= \sx^{-1}A \sz \beta$ and 
 $\xi = \lk(\sza)/\|\se\|>c>1$ and $\ke<C<\infty$, 
\[\|{\ab}\|^2_{\sx} \asymp \|\beta\|_{\sz}^2 \quad \text{ and }\quad \|{\ab}\|^2 \asymp \beta^\top (A^\top A)^{-1} \beta .\]
% {\marten Q: we need in fact only upper bounds. Please decide what's easier..} {\seth [keep upper and lower bounds, as is? Also, note to self, check the proof of this.]}
% \begin{enumerate}
%{\tiny     \item \[(1-\xi^{-1})\cdot \|\beta\|_{\sz}^2  \le \|\ab\|^2_{\sx} \le %\|\beta\|_{\sz}^2,\]
%    where the upper bound holds even when $\sz$ and $\se$ are not invertible.
%    
%    \item 
%     
%     \[\left(\frac{\xi - 1}{\xi+1} \right)\cdot \frac{1}{\ke} \cdot \beta^\top (A^\top A)^{-1} \beta \le \|\ab\|^2 \le \ke\cdot  \beta^\top (A^\top A)^{-1} \beta .\]
%    
%    \item Thus, if $\xi > c > 1$ and $\ke < C < \i$, then
%    \[\|\ab\|^2_{\sx} \asymp \|\beta\|_{\sz}^2 \quad \text{ and }\quad \|\ab\|^2 \asymp \beta^\top (A^\top A)^{-1} \beta .\]
%}    
%    \item {\marten  Again, punchline   into cor 7): Provided $\|\beta\|^2_{\sz}\to \i$, $\|\beta\|^2_{\sz}/\xi\to0$, and $\ke<C<\i$, then  $\| \ab \|\to0$ while $R(\0)- R(\ab)= \| \ab\|^2_{\sx}\to\i$, as $p\to\i$.  }
%\end{enumerate}
\end{lemma}

\begin{remark}\label{rem:sparse} %[Sparsity of $\ab$] %and the Lasso]
We illustrate our findings in Lemmas \ref{thm:ab norm bound e=0} and \ref{thm:ab norm bound main} with the following example (that we will  use in our simulations in Section \ref{sec:pcr}),  where
%Note that by (\ref{eqn:ab formula}) in Appendix [blah],
%\[\ab = \se^{-1} A(\sz^{-1} + A^\top \se^{-1}A)^{-1}\beta.\]
%Thus, if
 $\sz = \sigma_Z^{2}I_K$, $\se = \sigma_E^{2}I_p$, and $A^\top A = a^2 I_K$. It can be verified that in this case,
\begin{align}\label{explicit}
    \ab & %\sx^+ \sxy %= \frac{1}{\sigma_E^2 + \sigma_Z^2a^2} \sxy 
    = \frac{\sigma_Z^2}{\sigma_E^2 + a^2\sigma_Z^2 } A\beta 
%\end{align}
%\begin{align}\label{explicit}
 \\   \|{\ab}\|^2  &
    = \frac{a^2\sigma_Z^2}{(\sigma_E^2 + a^2 \sigma_Z^2 )^2} \|\beta\|_{\sz}^2
%\end{align}
%and
%\begin{align}
\\
\| {\ab}\|^2_{\sx} %= {\ab}^\top \sx {\ab}
&= \frac{a^2 \sigma_Z^2 }{\sigma_E^2 + a^2 \sigma_Z^2}\|\beta\|_{\sz}^2. \end{align}
%\frac{\sigma_Z^2}{\sigma_E^2 + \sigma_Z^2 a^2} A\beta.\]
Since $\lk(\sza) =a^2\sigma_Z^2$ and $\xi=a^2\sigma_Z^2/ \sigma_E^2 $, it confirms that  $  \|\beta\|_{\sz}^2/ \lk(\sza)  \to 0$  forces $\|{\ab}\|\to0 $, while at the same time $\| {\ab}\|_{\sx}^2 \asymp  \|\beta\|_{\sz}^2$ when $\xi$ is bounded below (in fact, $\| {\ab}\|_{\sx}^2/\|\beta\|_{\sz}^2\to 1$ when $\xi\to\i$ in this example).
%
%This 
%\end{remark}
%
%\begin{remark} \label{rem:lasso}

 We note that while
 %Given that
 $\|\ab\|\to 0$, %one may ask if  $\ab$ is sparse. 
 there is no reason to assume $\ab$ to be sparse.
  %In the example in the previous Remark \ref{rem:sparse}, 
  In this example, we can see from the explicit formula (\ref{explicit})   that
${\ab}_i=0 \Longleftrightarrow A_{i\sbt}^\top \beta=0$, whence row-sparsity of the matrix $A$ induces sparsity of the vector ${\ab}$.  For a more general $A$, this isn't the case and ${\ab}$ isn't necessarily sparse or even approximately sparse. This observation is corroborated in our simulations  in Section \ref{sec:pcr}.\\

%{\tiny We stress that the low-dimensional structure of the factor regression model should not be confused with sparsity of $\ab$.
%It is the low-dimensional structure  that {\color{red} forces ?}  both $\|\ab\|\to0$ and $R(\ab)/R(\0)\to0$.{\color{red} You mean $\infty$  ?}  This means that off-the-shelf regression estimators designed to take sparsity of the vector of regression coefficients $\theta$ in the  model $y=X^\top \theta + \eps$ % (\ref{model linear}) in remark \ref{rem:lin model} 
%into account, such as the ubiquitous LASSO method, may  not be appropriate,
%even  in the Gaussian context. {\color{red} I do not understand (at all !) the previous sentence: a vector can be sparse and have  $\|\ab\|\to 0$  and have $R(\ab)/R(\0)\to \infty $. I'm lost as to the point of this comment ! }  This observation is corroborated in our simulations  in Section \ref{sec:compare}.}
\end{remark}

Identity (\ref{id:null}),
Lemma \ref{thm:ab norm bound e=0} and Lemma \ref{thm:ab norm bound main} imply the following conclusion.

\begin{cor}%[Comparison of null risk under factor regression model]
\label{thm:null vs opt}% \hspace{1cm}
  Suppose model (\ref{model}) holds with $A$, $\sz$, $\se$ all full rank, let $\xi = \lk(\sza)/\|\se\|>c>1$, and suppose $\ke<C<\infty$. Alternatively, suppose that under model (\ref{model}), $\se=0$ and $A$, $\sz$ are full rank.
%   Suppose model (\ref{model}) holds, $A$, $\sz$, and $\se$ are full rank, and $\ke<C<\i$. Letting $\xi$ be the signal-to-noise ratio defined in (\ref{xi}),  
  Then, in either case, condition (\ref{cond:snrbeta2})
 % \[ \|\beta\|^2_{\sz}\to \i, \ \  \|\beta\|^2_{\sz}/\xi\to0\]
  implies 
  \[ \lim_{p\to\i} \| \ab \|=0,\ \text{  while }\ \liminf_{p\to\i} \left\{  R(\0)- R(\ab) \right\} 
  \gtrsim \liminf_{p\to\i}  \|\beta\|^2_{\sz} >0 . \]
 %
 %
 % \[R(\0) - R(\ab) \ge \|\beta\|_{\sz}^2\cdot (1-\xi^{-1}).\]
%   If instead $\se = 0$ with $A$ and $\sz$ full rank, then condition (\ref{cond:snrbeta2}) with $ \lambda_K(\sza)$
%   replacing $\xi$ implies
%   the same conclusion.
  
%  \[R(\0) - R(\ab) = \|\beta\|^2_{\sz}.\]
 % Consequently, if $\|\beta\|_{\sz}^2\not\to0$ and $\xi \ge 1$, $\liminf _{p\to\i} R(\0) - R(\ab)>0$ 

 % {\marten Move the punchline in corollary below. Thus, if $\|\beta\|^2_{\sz}\to \i$ and $\|\beta\|^2_{\sz} / \lambda_K(\sza)\to0$, then
%$\|\ab\|\to0$ and $R(\0)- R(\ab)= \| \ab\|^2_{\sx}\to\i$ as $p\to\i$}.

%  {\marten  Again, punchline   into cor 7): Provided $\|\beta\|^2_{\sz}\to \i$, $\|\beta\|^2_{\sz}/\xi\to0$, and $\ke<C<\i$, then  $\| \ab \|\to0$ while $R(\0)- R(\ab)= \| \ab\|^2_{\sx}\to\i$, as $p\to\i$.  }
\end{cor}

This result shows that while the norm of $\ab$ converges to zero in the factor regression model, its risk is separated from the risk of the null predictor $\0$ by a constant times $\|\beta\|_{\sz}^2$. In fact, as $\beta$ is an arbitrary vector in $\R^K$, the gap $R(\0)-R(\ab)$ will typically grow as $K$ increases. 

The behaviour $\|\ab\|\to 0$ is a feature of the factor regression model that arises from the joint low-dimensional structure of the model, as encoded in the covariance $\sxy$. This is in stark contrast to the behaviour of the best linear prediction vector $\theta$ in a linear model $y= X^\top \theta+ \eta$, as we do not expect $\|\theta\|$ to vanish as $p$ grows. We discuss the important roles played by these quantities in the risk bound analysis in the next section.

\subsection{Prediction under  linear regression with conditions on the design versus prediction under latent factor regression}
%\begin{remark}
\label{sec:lin model}

The model (\ref{model}) can be said to have \textit{joint} low-dimensional structure, in that both the features $X$ and response $y$ are (noisy) functions of the low-dimensional latent vector $Z$. We would like to argue that this structure plays an important role in the behaviour of the GLS $\a$, which we will study in the next section. In particular, to understand the implications of this joint-low dimensional structure, we could compare model (\ref{model}) to a model
% In place of model (\ref{model}), 
% in which $y$ is connected to $X$ only via the latent factors $Z$, we could consider a model 
in which $X$ continues to follow a factor model, but $y$ is connected to $X$ via a linear model:
\begin{eqnarray}\label{model linear} 
    X =  AZ + E,  \ \quad \ 
    y = X^\top \theta + \eta,
\end{eqnarray}
where $\theta\in \R^p$ is a generic $p$-dimensional regression vector, and $\eta$ is zero-mean noise independent of $X$. Model (\ref{model linear}) captures the setting in which there is low-dimensional structure in the features alone.

When $(X,y)\in \R^p\times \R$ are jointly Gaussian,  Lemma \ref{thm:gaussian frm} in Appendix \ref{proofs:lin model} shows the simple fact that if the factor regression model (\ref{model}) holds, then  (\ref{model linear}) holds,  with regression coefficients $\theta= \ab$
%\begin{equation}
%    \theta= \ab = (\sza + \se)^+ A\sz \beta,
%\end{equation}
and error  $\eta \coloneqq y-  X^\top \ab$, independent of $X$. 
 Here $\ab$ is the best linear predictor \textit{under the factor regression model (\ref{model})}, which we studied the properties of in Section \ref{sec:low dim} above.

We can thus compare model (\ref{model}) and (\ref{model linear}) directly in the Gaussian case. We stress that we do not assume Gaussianity elsewhere in our paper, but use it here to facilitate this comparison. 

% \begin{lemma}
% Suppose $(X,y)\in \R^p\times \R$ is jointly normal and mean zero, but otherwise arbitrary. Then, defining $\ab \coloneqq \sx^+\sxy$ and
% \[\eta \coloneqq X^\top \ab - y,\]
% we have that $\eta$ and $X$ are independent. In particular, if $(X,y)$ are jointly normal and model (\ref{model}) holds, then model (\ref{model linear}) holds with 
% \[\theta = \ab = \sx^+\sxy = (\sza + \se)^+A\sz \beta.\]
% \end{lemma}
% \begin{proof}
% We first compute
% \[\EE[X\eta] = \EE[XX^\top ]\ab - \EE[Xy] = \sx\ab - \sxy.\]
% Using the fact that $\sx\ab = \sxy$ from (\ref{eqn:sx sxy identity}), we find $\EE[X\eta]=0$ so $X$ and $\eta$ are uncorrelated. Since $X$ and $y$ are jointly normal, it can be shown that $X$ and $\eta$ are jointly normal. Thus, they are independent.
% \end{proof}

In Section \ref{sec:low dim} %below we will analyze the best linear prediction vector $\ab$ under the factor regression model, and see that it has very distinctive properties compared to a generic regression vector $\theta\in \R^p$. In particular, we will find
we found that $\|\ab\|\to 0$, provided (\ref{cond:snrbeta2}) holds. 
 Thus, when the factor regression model (\ref{model}) 
is viewed as a particular case of  (\ref{model linear}), we have 
$\|\ab\| = \|\theta\| \to 0$.
%
%, and yet the risk $R(\ab)$ does not approach the null risk $R(\0)$, 
%provided $\|\beta\|^2_{\sz}/\xi\to0$ and $\ke<C<\infty$.
This behavior is in sharp contrast with the typical behavior of a  generic linear model  $y = X^\top\theta + \eta $  as in (\ref{model linear}), in which $\|\theta\|$ is usually fixed or growing with $p$. We argue that this difference
%highlights the difference between model (\ref{model linear}) and the factor regression model (\ref{model}), which, again, is a special case when the data is Gaussian.
%
%The fact that $\|\ab\|\to 0$ under the factor regression model 
  has important implications for the performance of the GLS predictor $\a$.
%   which we will study in the next section. 

  One way this can be seen is by considering the bound from the recent work \cite{bartlett2019} on the excess risk $R(\a)-R(\theta)$, proved under model $E(y|X) = X^T\theta$ 
%   {\color{red} erase this (\ref{model linear}), as they don't have a factor model specified on $X$}
  for sub-Gaussian $(X,y)$. In particular, the bound of \cite{bartlett2019} contains a bias term given by
\begin{equation}\label{bias bart linear}
     \| \theta\|^2 \| \sx\| \max \left\{\sqrt{\frac{{\rm r_e}(\sx)}{n}}, \frac{{\rm r_e}(\sx)}{n}\right\}.
\end{equation}
%  Whereas this bound is stated for general $\sx$, the main discussion in \cite{bartlett2019} focuses on the case $\|\sx\|=1$.
% In contrast, in the factor model $X=AZ+E$, %with $\se\ne0$, 
% $\| \sx\|$ is typically unbounded, since    $\| \sx\| \ge  \| A \sz A^\top\| \to\infty$ as $p/n\to\i$ for many matrices $A\in \R^{p\times K}$; see the discussion following Theorem \ref{thm:upper bound}. %{thm:ab norm bound main}.
% In this sense, the class of factor regression models studied here are not considered in the definition of \textit{benign} covariance matrices found in \cite{bartlett2019}.
% Indeed, 

 We examine this bound assuming further that model (\ref{model linear}) holds.
Since
%the trace of the $p\times p$ matrix $\sx = \se + \sza$ 
% is typically of order $p$, via mild assumptions on either  $\se$ or   $A$, 
%it is  natural to assume that  as $p/n\to \i$,
%\[\frac{\tr(\sx)}{n} = \frac{ \tr(\se)}{n} +  \frac{ \tr(A \sz A^\top)}{n}\to\i.\]
%$\tr(\sx)/n \to \i$, 
%This implies in turn that 
% the product
\begin{align}
    \| \sx\| \max \left\{\sqrt{\frac{{\rm r_e}(\sx)}{n}}, \frac{{\rm r_e}(\sx)}{n}\right\} &=\max \left\{\sqrt{\frac{\|\sx\|\tr(\sx)}{n}}, \frac{\tr(\sx)}{n}\right\}
    %\nonumber\\
    %&
    \ge \frac{\tr(\sx)}{n} %= \frac{ \tr(\se)}{n} +  \frac{ \tr(A \sz A^\top)}{n},
    %
    %
    %&\gtrsim  \max\left\{\sqrt{\frac{\rank(\se)}{n}}, \frac{\rank(\se)}{n}\right\}\to \i.
    \label{bias inf part}
\end{align} %diverges, 
and
\[\frac{\tr(\sx)}{n} = \frac{ \tr(\se)}{n} +  \frac{ \tr(A \sz A^\top)}{n}\to\i\]
%$\tr(\sx)/n \to \i$, 
under model (\ref{model linear}) with
mild assumptions on either  $\se$ (e.g., $\se\asymp I_p$)  or   $A$ (see Remark \ref{rem:een}),
 the bias term (\ref{bias bart linear}) will only converge to zero if $\|\theta\|\to 0$.
 %%More importantly, the bias term (\ref{bias bart linear}) typically will only converge to zero if $\|\theta\|\to 0$.
 
 As noted above,  $\|\theta\|\to0$ is rather unnatural in a generic model  (\ref{model linear}). However, we also noted that when $(X, y)$ are Gaussian
%  , we also argued that when
 %we are in the special case of model (\ref{model linear}) for which 
 and the factor regression model (\ref{model}) holds, then  (\ref{model linear}) holds with  $\|\theta\|=\|\ab\|\to 0$, which  means that the bias term (\ref{bias bart linear}) can   converge to zero when the data is generated by model (\ref{model}).
%  a factor regression model.
%  {\tiny from \cite{bartlett2019} 
% can   converge to zero, although the authors do not take  this route in their analysis (see Remark \ref{nozero}  below).}
% ,  
% in view of $\|\theta\|=\|\ab\|\to 0$. 
We take this as indication that the bias in prediction with $\a$ can be significantly lower in the factor regression model (\ref{model}) compared to a generic model (\ref{model linear})  as a result of the joint low-dimensional structure of model (\ref{model}).

We note that this discussion is only  based on an upper bound (\ref{bias bart linear})  on the bias term of the prediction risk. 
It nevertheless motivates a full investigation of an alternative upper bound to (\ref{bias bart linear}), directly derived under model (\ref{model}). This is the subject of Section \ref{section:min_ell_2} below,  with our main result presented in Theorem  \ref{thm:upper bound}.

% {\color{red} I suggest the blue above. This paragraph is a dangerous thing to say: While we have not rigorously proved this claim, since (\ref{bias bart linear}) is only an upper bound on the prediction risk, {\seth a simulation study we conducted further supported it.}} 
 \begin{remark} \label{nozero}  The authors of  \cite{bartlett2019} take a different route, complementary to ours,  in their analysis of the  bound (\ref{bias bart linear}). 
% {\tiny  As previously noted,  the bound (\ref{bias bart linear}) in  \cite{bartlett2019}}
Although  they derived it  with no assumptions on $\|\sx\|$, the desired convergence to zero is established under the assumption that   $\sx$ belongs to what is called in \cite{bartlett2019} a class of {\it  benign} covariance matrices, that in particular satisfy $\| \sx \| = 1$.

This assumption allows the authors to avoid making the unpleasant assumption that a generic $\theta$ would have $\ell_2$-norm converging to zero with $p$. To see why, note that when $\|\sx\|$ is bounded, working in the regime ${\rm r_e}(\sx)/n\to 0$ immediately implies
 \[\| \sx\| \max \left\{\sqrt{\frac{{\rm r_e}(\sx)}{n}}, \frac{{\rm r_e}(\sx)}{n}\right\}\to 0,\]
 which in turn means that under the assumption $\|\sx\|=1$, their bias term (\ref{bias bart linear}) can converge to zero even when $\|\theta\|\not \to 0$, for a generic $\theta$.  
 
 However, as we have shown in Lemma \ref{thm:spectrum} above, this class does not cover covariance matrices $\sx$ associated with a random vector that obeys a factor model $X = AZ + E$, as $\| \sx \| \rightarrow \infty$ with $p$ in this case.  Since in factor regression we argued that  $\|\theta\|=\|\ab\|\to 0$, one can still expect that (\ref{bias bart linear}) will vanish, in the regime ${\rm r_e}(\sx)/n\to 0$, even though $\| \sx \| \rightarrow \infty$.  The results of Section \ref{section:min_ell_2} can thus be viewed as complementary to  those in \cite{bartlett2019}.

% {\seth NOTE to possibly include (?): Under model (\ref{model linear}), the variance term in the bound of \cite{bartlett2019} scales with $\sigma_\eta^2$. So in the special case of Gaussian data and the FRM (\ref{model}), where $\eta \coloneqq y - X^\top \ab$, the variance scales with
% \[\sigma_\eta^2 = \EE[\eta^2]= R(\ab).\]
 % }
 
%  This is not the case in our setting, however, since $\|\sx\|\to \i$ is the natural setting for features with a factor regression structure (see Lemma \ref{thm:spectrum} above).
 \end{remark}

\section{Minimum $\ell_2$-norm prediction in factor regression}\label{section:min_ell_2}
In this section we analyze the GLS $\a$, and present our main contribution, namely, novel finite-sample bounds on the prediction risk $R(\a)$ relative to the benchmarks laid out in Section \ref{sec:bench}.
% \subsection{Exact adaptation in  factor regression models with noiseless features}\label{sec:noiseless}
% We begin our analysis
% % of the min-norm interpolator $\a$ by considering 
% an extreme case of model (\ref{model}), in which $E = 0$ almost surely, and thus $\Sigma_X$ is degenerate, with  ${\rm r_e} (\sx) \leq \rank(\sx) = K$. 
% In this case, the signal-to-noise ratio $\xi = \i$, and so assuming $n> C K$ for some $C$ large enough,  Lemma \ref{thm:SNR cond}  no longer forces $R(\a)/R(\0)\to 1$. In fact, we will show that $\a$ exhibits excellent behavior, as its risk approaches the optimal risk as $p\to\i$.

\subsection{Exact adaptation in  factor regression models with noiseless features}\label{sec:noiseless}

We begin our analysis by considering an extreme case of model (\ref{model}), in which $E = 0$ almost surely, and thus $\Sigma_X$ is degenerate, with  ${\rm r_e} (\sx) \leq \rank(\sx) = K$. 
% In this case, the signal-to-noise ratio $\xi = \i$, and so assuming $n> C K$ for some $C$ large enough,  Lemma \ref{thm:SNR cond}  no longer forces $R(\a)/R(\0)\to 1$. In fact, we will show that $\a$ exhibits excellent behavior, as its risk approaches the optimal risk as $p\to\i$.

Proofs for this section are contained in Appendix \ref{proofs:noiseless}. We make the following assumptions.
\begin{ass} \label{ass:fullrank} 
The $p\times K$ matrix $A$ and 
$K\times K$ matrix $\sz$ 
both have full rank equal to $K$.
\end{ass}
\begin{ass}\label{ass:subg fm}
$E = \se^{1/2}\tilde E$, where $\tilde E\in\R^p$ has independent entries with zero mean, unit variance, and sub-Gaussian constants bounded by an absolute constant. 

Furthermore, $Z = \sz^{1/2}\tilde Z$ and $\eps = \sep \tilde \eps$, where $\tilde Z\in\R^K$ and $\tilde \eps\in\R$ have zero mean and sub-Gaussian constants bounded by an absolute constant.
\end{ass}

We first analyze the norm of $\a$. In Lemma \ref{thm:ab norm bound e=0} above, we showed that $\|{\ab}\|^2 = \beta^\top (A^\top A)^{-1}\beta$ when $\se=0$,  and as a result, Corollary \ref{thm:null vs opt} states that $\|{\ab}\|\to 0$,  provided $\|\beta\|^2_{\sz} / \lambda_K(\sza)\to0$ as $p\to \i$. We now show that $\a$ mimics this behavior under 
 the additional condition that $ (\sep^2\log n) /\lk(\sza)\to0 \ \text{ as $n\to\i$} $. 
%almost the same condition.

% {\seth [NEW VERSION]

% \begin{lemma}[Factor regression with noiseless features]
% \label{thm:a hat norm e=0}
% Under model (\ref{model}) with $\se=0$, suppose that Assumptions \ref{ass:fullrank} \& \ref{ass:subg fm} hold and that
% $n> C\cdot  K$ for some large enough absolute constant $C>0$. Then, with probability at least $1-c/n$ for some absolute constant $c>0$,  $\a = A^{+\top}\Z^+\y$, and 
% \begin{equation}\label{eqn:a norm bnd e=0 simpl}
%     \|\a\|^2 \lesssim \frac{1}{\lk(\sza)}\left(\|\beta\|_{\sz}^2+ \sep^2\frac{K\log n}{n}\right).
% \end{equation}
% \end{lemma}
% }

\begin{lemma}
\label{thm:a hat norm e=0}
Under model (\ref{model}) with $\se=0$, suppose that Assumptions \ref{ass:fullrank} and \ref{ass:subg fm} hold, and that
$n> C\cdot  K$ for some large enough absolute constant $C>0$. Then, with probability at least $1-c/n$ for some absolute constant $c>0$,
\begin{equation}\label{eqn:a norm bnd e=0 simpl}
    \|\a\|^2 \lesssim \frac{1}{\lk(\sza)}\left(\|\beta\|_{\sz}^2+ \sep^2\frac{K\log n}{n}\right).
\end{equation}
\end{lemma}

The fact that $\a$ vanishes
%This 
does {\em not} imply that $R(\a)/R(\0)\to 1$, just like $R({\ab})/ R(\0) \not\to 1$ in Corollary \ref{thm:null vs opt}.
%{\marten delete: as in Theorem \ref{thm:alpha null risk} . }
%As such, $\a$ mimics the behavior of $\ab$, described in Corollary \ref{thm:null vs opt} above, and 
 We will now show that in fact the risk $R(\a)$ approaches the optimal risk $R({\ab})$ by adapting to the low-dimensional structure of the factor regression model.
 %, whereas $R(\0) - R(\ab) = \|\beta\|_{\sz}^2 \not\to 0$ by Corollary \ref{thm:null vs opt} above.
% {\tiny When $K < n$, the factor regression model (\ref{model}) closely resembles a low-dimensional classical regression model, with the caveat that $A$ is not known and therefore $Z$ cannot be treated as observed via $X = AZ$.  Nevertheless, the first part of Theorem \ref{thm:E=0 risk bound} below shows that $R(\a)$ does indeed mimic prediction in $K$ observed dimensions.}
Let $\wh y_z \coloneqq Z^\top \wh\beta$ be the predictor based on the least-squares   regression coefficients $\wh\beta\coloneqq \Z^+\y$   of $\y$ onto $\Z$;
% . The predictor $\wh y_z$ is the 
this is the classical least-squares prediction of $y$ under model (\ref{model}) that an oracle would use if it had access to the unobserved data matrix $\Z$, and the new, but unobservable, data point $Z$.
% , we call $\wh y_z$ the \textit{oracle} predictor. 
In contrast, let $\wh y_x = X^\top \a$ be the least-squares predictor of $y$ from $X$ based on $(\X,\y)$ only.
%that can be constructed on the basis of the observed new point $X$
Theorem \ref{thm:E=0 risk bound}.1 below shows that the realizable prediction $\wh y_x$ equals the oracle prediction $\wh y_z$. The second part of the theorem gives lower and upper bounds on the risk that hold with high probability over the training data.

\begin{thm}[Factor regression with noiseless features]% bound valid for any $A$]
\label{thm:E=0 risk bound}
Under model (\ref{model}) with $\se=0$, suppose that Assumption \ref{ass:fullrank} holds.
\begin{enumerate}
    \item Then, on the event that the matrix  $\Z$  has full rank $K$, we have, $\wh y_x = \wh y_z$ and $R(\a) = \EE_{(X,y)}[ (X^\top \a - y)^2] = \EE_{(Z,y)}[(Z ^\top \hat \beta- y)^2]$.
    % {\marten This event holds with probability at least $1-c/n$ for some absolute constant $c>0$.}
    \item Suppose that Assumption \ref{ass:subg fm} also holds and that
$n> C\cdot  K$ for some large enough absolute constant $C>0$. Then, with probability at least $1-c/n$ for some absolute constant $c>0$, $\Z$ has full rank $K$ and
\begin{equation}\label{eqn:E=0 risk bounds}
    R(\a) - \sep^2\lesssim \sep^2\frac{K\log n}{n}\hspace{0.5cm}\text{and}\hspace{0.5cm}\EE_{\Eps} [R(\a)] - \sep^2 \gtrsim \sep^2\frac{K}{n}.
\end{equation}
\end{enumerate}
%
%Then there exists $c_1>1$, $c_2 >0$ such that if $n > c_1K$ and Assumptions  \ref{ass:fullrank} \& \ref{ass:subg fm} hold, then with probability at least %$1-c_2/n$,
% \[R(\a) - \sep^2\lesssim \sep^2\frac{K\log n}{n}\hspace{0.5cm}\text{and}\hspace{0.5cm}\EE_{\Eps} [R(\a)] - \sep^2 \gtrsim \sep^2\frac{K}{n}.\]
\end{thm}
%{\color{red} Do we still have an expectation above ?} 
% \begin{proof}
% {\color{blue} See sections \ref{proof:a'x+=z+} } and \ref{proof:E=0 risk bound} for the first and second claims, respectively.
% \end{proof}

The risk bounds (\ref{eqn:E=0 risk bounds}) are the same as the standard risk bounds for prediction in linear regression in $K$ dimensions with observable design, despite  $A$  not being known under model (\ref{model}). We note that,  since $\rank(\X) = K < n$, $\y$ may not lie in the range of $\X$ and so $\a$ may not interpolate.
Nonetheless, under model (\ref{model}), with $E \neq 0$ and in the interpolating regime, we expect that the prediction performance of $\wh y_x$ will still approximately mimic that of $\wh y_z$ 
 as long as the signal, as measured by $\lk(A^\top \sz A)$, is strong relative to the noise, as measured by $\|\se\|$.
 The next section is devoted to the detailed study of  this fact.

Finally, another explanation of the perhaps surprisingly good performance of the GLS 
% {\tiny why the GLS 
%(hmm, it doesn't really interpolate here...)  (can just call it the GLS... yes, think so, should do it earlier as well. agreed. will have to go through from the beginning.  we will want to use min-norm interpolator as well. ok. especially in section where we prove when it interpolates. Page 1, we should say we study GLS, coincides with min-norm for p>n. Something like that. Ok, i'll make note and will do it in the last pass.) 
% $\wh \alpha$ performs so well 
is that it coincides with 
Principal Component Regression (PCR), see, e.g., \cite{SW2002_JASA},  in the case when $\se=0$.
Indeed, 
this is a  natural and practical prediction method when the covariance matrix $\sx$ has an approximately low rank. If $\se=0$, then $\sx= A\sz A^\top$ has rank of at most $K$ and so is exactly low rank.
%
%in the factor regression model and in fact 
%under the assumption that the covariance matrix $\sx$ has an approximately low rank, a natural and practical prediction method is Principal Component Regression (PCR).
In PCR, the response $\y$ is regressed onto the first $K$ principal components  of the data matrix $\X$ to estimate a vector of coefficients $(\X \wh U_{ K})^+\y$. %Typically $\wh K$ is estimated from the data and 
Here
$\wh U_{  K}\in \R^{p\times  K}$ has columns equal to the first $K$ eigenvectors of the sample covariance matrix $\X^\top \X/n$. A new response $y$ is then predicted by $\wh \alpha_{\rm PCR}^\top X$, where $\wh \alpha_{\rm PCR} \coloneqq \wh U_{K} (\X \wh U_{ K})^+\y$ and $X$ is the new feature vector.
The following lemma states that the PCR and GLS predictors coincide when $\se = 0$.
\begin{lemma}\label{thm:pcr e=0}
% On the event where $\rank(\X)\le K$, letting $\wh \alpha_{\rm PCR} \coloneqq \U_K (\X \U_K)^+\y$, we have $\a = \wh \alpha_{\rm PCR}$. In particular, when $\se = 0$, $\a = \wh \alpha_{\rm PCR}$ almost surely. [will edit to match your text after]
% I suggest to keep it simple: 
Define $\wh \alpha_{\rm PCR} \coloneqq \wh U_K (\X \wh U_K)^+\y$.
 On the event $\{\rank(\X)=K\}$, $\a = \wh \alpha_{\rm PCR}$.  In particular, when $\se=0$, $K > C\cdot n$, and Assumptions \ref{ass:fullrank} \& \ref{ass:subg fm} hold, $\a = \wh \alpha_{\rm PCR}$ with probability at least $1-c/n$
 for some absolute constant $c>0$. 
% Provided $\se=0$, $\a = \wh \alpha_{\rm PCR}$ almost surely.
\end{lemma}
% The following lemma states that the PCR and GLS predictors coincide when $\se = 0$.
Thus, the prediction $\wh \alpha_{\rm PCR}^\top X$ of $y$ based on PCR is exactly equal to the prediction $\a^\top X$ based on the GLS, in the case when $\se = 0$. Given that PCR is a natural and widely used prediction method in this setting, this further explains the performance of the GLS, at least when $\se =0$.
\subsection{Approximate adaptation of interpolating predictors in  factor regression}\label{main}

In this section we present our main results on the excess risk of prediction with $\a$, relative to the two benchmarks in Section \ref{sec:bench} above,  under the factor regression model (\ref{model}) with $E\neq 0$.

Our main result, Theorem \ref{thm:upper bound} below, shows that despite the fact that $\a$ interpolates, in that $\X\a = \y$ (Proposition \ref{in}), and that $\|\a\|\to 0$ (Lemma \ref{thm:gls norm bound}), the excess risks can vanish as a result of approximate adaptation to the embedded low-dimensional structure of (\ref{model}). 
The estimator $\a$ is guaranteed to interpolate the data whenever $\rank(\X)=n$, or equivalently, the smallest singular value $\sigma_n(\X) > 0$. The next proposition shows that the following set of   conditions in terms of $n$, $K$ and ${\rm r_e} (\se)$ guarantee this.
% We first show in Lemma \ref{thm:gls norm bound} that as in the noiseless case, $\|\a\|\to 0$ as $p,n\to \i$ with $p/n\to \i$, under a signal-to-noise condition. 
Proofs for this section are contained in Appendix \ref{proofs:main}.
%{\color{red}
%\begin{ass}\label{ass:K,n,r_e}
%For some absolute constant $c>1$ large enough,
%$n> c \cdot K$ and ${\rm r_e} (\se)>c \cdot n$.
%\end{ass}
%[move this assumption to earlier (since we use a similar assumption above), or use a %different approach to notation? Actually, I vote to put this assumption in the statement of %the theorem so it is clear. Otherwise the discussion in the following paragraph is a bit %confusing (seth)]}

 \begin{prop}\label{in} Under  model (\ref{model}), suppose that Assumptions \ref{ass:fullrank} and \ref{ass:subg fm} hold, and that ${\rm r_e} (\se)>C \cdot n$ for some $C>0$ large enough.
 %and $n> c \cdot K$ and ${\rm r_e} (\se)>c \cdot n$ hold,   for some universal constant $c>1$ large enough.
 Then, %there exists an absolute constant $c_1$ such that if $n > c_1 K$ and ${\rm r_e} (\se)>c_1n$, then 
 with probability at least $1 - c/n$, for some $c>0$,
\[\sigma_n^2({\X}) \gtrsim \text{\rm tr}(\se) > 0,\]
and thus,  in particular, $\a$  interpolates: $\X \a = \y$.
 \end{prop}

General existing bounds of the type  $\sigma_n({\bf X}) \gtrsim (\sqrt{p} - \sqrt{n})$ are by now well established in random matrix theory  \cite{rud2009}. When $p > C\cdot n$ for some $C>1$ and the entries of $\X$ are i.i.d.~sub-Gaussian with zero mean and unit variance, Theorem 1.1  in \cite{rud2009} implies  that $\sigma_n^2({\bf X}) \gtrsim p  $ with high probability. By comparison, Proposition \ref{in} holds for $\X$ with i.i.d.~sub-Gaussian rows with covariance matrix $\sx = \sza + \se$. 

The following result shows that as in the noiseless case $\se=0$ of Lemma \ref{thm:a hat norm e=0}, $\|\a\|\to 0$,
%when $E\neq 0$, 
mimicking the behavior of the best linear predictor $\ab$. We proved in Lemma \ref{thm:ab norm bound main} and Corollary \ref{thm:null vs opt}
that  $\|\ab\|\to0$ when $\lk(\sza)$ grows faster than $\|\beta\|^2_{\sz}$ as $p\to\i$; we will need here the additional assumption that $n\log n / {\rm r_e}(\se) \to0$ to guarantee $\|\a\|\to0$ as $n\to\i$.
The proof uses Proposition \ref{in}, which requires that the effective rank ${\rm r_e}(\se)$ is larger than a constant times $n$.
\begin{lemma}\label{thm:gls norm bound}
 Under model (\ref{model}),   suppose that Assumptions \ref{ass:fullrank} and \ref{ass:subg fm} hold and
 $n > 
 C\cdot K$ and ${\rm r_e} (\se)>C \cdot n$ hold, for some   $C>0$. 
Then, with probability exceeding $1-c/n$, for some   $c>0$, 
\begin{equation}
    \|\a\|^2\lesssim \frac{1}{\lk(\sza)}\|\beta\|^2_{\sz} + \sep^2\frac{n\log n}{{\rm r_e}(\se)}.
\end{equation}
% \begin{align}
%     \|\a\|^2 &\lesssim \beta^\top (A^\top A)^{-1}\beta + \sep^2\frac{n\log n}{\tr(\se)}\\
% &\lesssim \frac{1}{\lk(\sza)}\|\beta\|^2_{\sz} + \sep^2\frac{n\log n}{{\rm r_e}(\se)}.
% \end{align}
\end{lemma}

% We now move on to our main result, a bound on the prediction risk of $\a$. 
Despite the fact that $\|\a\|\to 0$ under the conditions stated, we now show that $\a$ can outperform the null predictor $\0$. If $\lk(\sza)$ grows faster than $\tr(\se)/n$ and $K/n\to0$, then Lemma \ref{thm:SNR cond} states that ${\rm r_e} (\sx)/n   $ remains bounded, and Theorem \ref{thm:alpha null risk} allows for the possibility that  $\a$ has asymptotically lower risk than $\0$.
Theorem \ref{thm:E=0 risk bound} above showed that $R(\a)- \sep^2$ can in fact approach $0$ under certain conditions when $E=0$.
The following result demonstrates that this can continue to hold even when $E\neq 0$.

\begin{thm}[Main result: Risk bound for factor regression]\label{thm:upper bound}
Under model (\ref{model}),   suppose that Assumptions \ref{ass:fullrank} and \ref{ass:subg fm} hold and
 $n > 
 C\cdot K$ and ${\rm r_e} (\se)>C \cdot n$ hold, for some   $C>0$. 
Then, with probability exceeding $1-c/n$, for some   $c>0$, 
%Under model \ref{model}   with Assumptions \ref{ass:fullrank} and \ref{ass:subg fm}, there exists an absolute constant $c_1$ such that if $n > c_1 K$ and ${\rm r_e} (\se)>c_1n$, then with probability at least $1 - c_2/n$, %$\a$ will interpolate, i.e. $\X \a = \y$, and
\begin{eqnarray}\label{bound:main}
%(i) \  
R(\a) - R(\alpha^*) \, \le\,  R(\a)- \sep^2 
&\lesssim& \frac{\|\beta\|^2_{\sz}}{\xi}\cdot\frac{{\rm r_e} (\se)}{n} + \sep^2\frac{n\log n}{{\rm r_e} (\se)} + \sep^2\frac{K\log n}{n}.
\end{eqnarray}
%,\]
%and also 
%\[ (ii) \  R(\a) - R(\alpha^*) \lesssim %\frac{\|\beta\|^2_{\sz}}{\xi}\cdot\frac{{\rm r_e} (\se)}{n} + \sep^2\frac{n\log %n}{{\rm r_e} (\se)} + \sep^2\frac{K\log n}{n},
%\] 
Recall    $\xi \coloneqq \lk(\sza)/\|\se\|$ is  the signal-to-noise ratio.
\end{thm}
% \begin{remark}
% % {\color{red} Do we still want a remark here, given the text directly above the theorem? I could just quickly explain that $R(\a)/R(\0)\to 0$. }{ \color{green} No, I agree.}
%
% \end{remark}
% \begin{proof}
% The first inequality in (\ref{bound:main}) is an immediate consequence of the second part of Lemma 
% \ref{thm:bench compare}.
% We prove the second inequality in (\ref{bound:main}) %Theorem \ref{thm:upper bound} 
% in Section \ref{proof:upper bound}. 
% %See section (\ref{proof:upper bound}).
% \end{proof}
%{\color{green} Keep only one display above, and write target risk inequality. The whole discussion that follows needs sharper editing.} \\
%
%We note that the conditions $n > c_1K$ and ${\rm r_e} (\se) > c_1n$ ensure that we are in the interpolating regime, and t
%The result holds for any value  $\xi$. However,  small values of the excess risk will only occur when $\xi \gg {\rm r_e} (\se)/n$, and thus when the effective rank of $\sx$ is smaller than $n$, as needed. 
%
\begin{remark}
%  In the scenario
% \[ \lim_{n\to \i} \left( \frac{K}{n}+  \frac{n\log n}{ {\rm r_e}(\se)} \right)=0,\]
Suppose $n \gg  \sep^2K\log n$ and ${\rm r_e}(\se) \gg \sep^2n\log n$. We then find that $\a$ interpolates by Proposition \ref{in}, and the behavior of $\a$ is determined by
the eigenvalue $\lk(\sza)$ or, equivalently, the signal-to-noise ratio $\xi=\lk(\sza))/\|\se\|$.
\begin{itemize}
    \item[(a)] If $\lk(\sza)\gg \tr(\se)/n$, then Lemma \ref{thm:SNR cond} implies that $R(\a)$ need no longer approach the trivial null risk $R(\0)$.
    \item[(b)] If $\lk(\sza)\gg \|\beta\|^2_{\sz}$, then Lemma \ref{thm:gls norm bound} implies $\|\a\|\to0$. 
\item[(c)] 
If $\lk(\sza)\gg \|\beta\|^2_{\sz}\tr(\se)/n$, then $R(\a)-\sep^2\to0$. Indeed, this assumption, together with $n \gg  \sep^2K\log n$ and ${\rm r_e}(\se) \gg \sep^2n\log n$, ensures that the right-hand side of the inequality (\ref{bound:main}) in Theorem \ref{thm:upper bound} is asymptotically negligible. \\
% When $\|\beta\|_{\sz}^2$ and $\tr(\se)/n$ are bounded below by constants, then, condition (c) implies both conditions (a) and (b), and their conclusions.  
\end{itemize}
\end{remark}

The first inequality in (\ref{bound:main}) is an immediate consequence of the first part of Lemma \ref{thm:bench compare} above.
We now discuss the three terms appearing in the upper bound (\ref{bound:main}) of Theorem \ref{thm:upper bound}. A comparison with  the risk bound in Theorem \ref{thm:E=0 risk bound} above, where the feature noise $E$ is equal to zero, reveals that the term $\sep^2K\log(n)/n$ in (\ref{bound:main})  
is equal to the risk of the oracle predictor $\wh y_z$ up to the multiplicative $\log n$ factor, and is small when $K \ll n$.
% , which can be viewed in the context of a bias-variance decomposition. 
The first two terms can be viewed as bias and variance components, respectively, that capture the impact of non-zero $\se$. 
The first term (bias) is proportional to 
the effective rank ${\rm r_e}(\se)$, while the second   term (variance) is inversely proportional to ${\rm r_e}(\se)$. As such, the variance term is implicitly regularized by the feature noise $E$, while for the bias to be small, we need the signal-to-noise ratio $\xi$ to be sufficiently large.
For example, suppose that the eigenvalues of $\sz$ and $\se$ are constant, that is, 
$c_1 \le \lk(\sz) \le \| \sz\| \le C_1$ 
%{\color{red} Here it's uncomfortable: we introduced $\sz$, and now we want to get rid of it when we actually check the assumptions... Anyway,  we should leave it like this :)}
% for some universal $c>0$, since it is reasonable to assume that the latent factors have a non-degenerate covariance matrix, 
and $c_2<\lp(\se)\le \|\se\| <  C_2$, for some $c_1,c_2, C_1, C_2\in (0,\infty)$, both standard assumptions in factor models. Then, %$\rm{r_e} (\se) \asymp p$ and  
\begin{equation}\label{one} 
{\rm{r_e}} (\se) \asymp p, \hspace{0.5cm}\text{and}\hspace{0.5cm}
\xi= \frac{\lk(\sza)}{\|\se\|}  \gtrsim \lambda_K(A^\top A).
\end{equation} 
Provided $\beta$ has uniformly bounded entries $|\beta_i|\le C$, % and $\|\sz\|\le  C$  so that 
$\|\beta\|^2_{\sz}\le C_1\cdot C^2 \cdot K $, % \lesssim K$, 
and the bias term in (\ref{bound:main}) can be bounded as 
\begin{equation}
    B_Z \coloneqq \frac{\|\beta\|^2_{\sz}}{\xi}\cdot \frac{{\rm r_e}(\se)}{n}\lesssim \frac{Kp}{n\cdot \lambda_K(A^\top A)};
\end{equation} 
it thus approaches zero whenever  \begin{equation}\label{bias} \lambda_K(A^\top A) \gg \frac{Kp}{n}. \end{equation} 
We mention that the examples of $A$ in Remark \ref{rem:een} of Section \ref{sec:effrank} all imply (\ref{bias}), provided $K\ll n$ in cases 1 and 2 (since there $\lambda_K(A^\top A) \gtrsim p$), and $K^2 \ll n$ in case 3 (since there $\lambda_K(A^\top A) \gtrsim p/K $).

We summarize this discussion in Corollary \ref{thm:risk bound purevar} below.
%,  on which Table \ref{table:results summary} in the Introduction is based. 

\begin{cor}\label{thm:risk bound purevar}
Under the same conditions as in Theorem \ref{thm:upper bound}, suppose, in particular, that $\lambda_K(\Sigma_Z)$ and  $\|\Sigma_E\|$ are constant, ${\rm r_e} (\Sigma_E) \asymp p$, and $\|\beta\|^2_{\sz} \lesssim K$. 
Then, with probability at least $1 - c/n$, for some absolute constant $c>0$,
\begin{equation}\label{cluster}R(\a) - R(\alpha^*) \leq R(\a) - \sep^2 \lesssim \frac{K}{\lambda_K(A^\top A)}\times \frac{p}{n} + \sep^2  \left( \frac{ n}{p} +  \frac{K}{n} \right)\log n .\end{equation} 
In particular, if $\lambda_K(A^\top A) \gtrsim p/K$, and  
 with probability at least $1-c/n$, for some absolute constant $c>0$, 
\begin{equation}\label{cluster}R(\a) - R(\alpha^*) \leq R(\a) - \sep^2 \lesssim \frac{K^2}{n} + \sep^2  \left( \frac{ n}{p} +  \frac{K}{n} \right)\log n.\end{equation} 
\end{cor}

  Figure \ref{fig:double descent}   illustrates the  risk behavior proved in Theorem \ref{thm:upper bound}. Note the descent towards zero in the regime $\gamma\coloneqq p/n>1$. For completeness, we  also provide a bound on the risk  $R(\a)$ for the low-dimensional case $p \ll n$,  under model (\ref{model}), in Appendix \ref{proof:LS}.
%%note that we have not only a double ?? descent of the risk, but also a descent to {zero}, in the interpolation regime.

\begin{figure}[h]
    \centering
    \includegraphics[width=16cm]{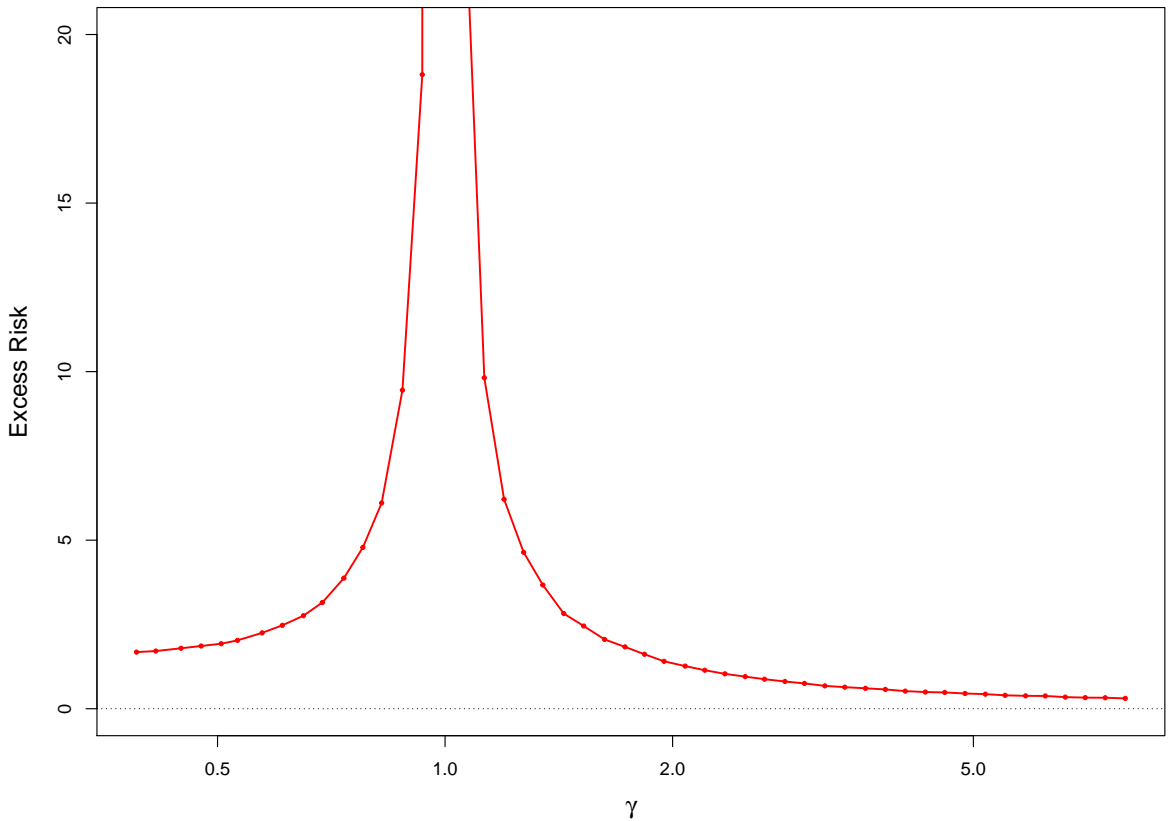}
    \caption{\small Excess prediction risk $R(\a) - \sep^2$ of the minimum-norm predictor under the factor regression model as a function of $\gamma = p/n$. %The dotted line denotes zero excess risk. 
    Here $K$ increases linearly from $16$ to $64$, $n=\lfloor K^{1.5}\rfloor$ and thus increases from $64$ to $512$, and $p$ increases from $33$ to $4066$.
    Further, $\se = I_p$, $\sz = I_K$, $\beta = (1 ,\ldots,1)^\top$, and $A = \sqrt{p}\cdot V_K$, where $V_K$ is generated by taking the first $K$ rows of a randomly generated $p\times p$ orthogonal matrix $V$.}
    \label{fig:double descent}
\end{figure}

\subsection{Comparison to existing results}%finite sample bounds}
\label{sec:compare}

The recent paper \cite{bartlett2019} 
gives a bias-variance type bound  on the excess prediction risk
of the minimum-norm predictor $\wh y_x = X^\top\a$ considered in this work.
In contrast to our study, \cite{bartlett2019} does not consider model (\ref{model}), and in fact assumes $\EE[y|X] = X^\top\theta$ for some $\theta\in \R^p$, which is typically not satisfied under (\ref{model}) when $(X,y)$ are sub-Gaussian, but not Gaussian. 

\begin{table}[h]
\centering
    \begin{tabular}{|c|c|c|c|}
    \hline
    Regime    &  Bias in Theorem \ref{thm:upper bound}  & Bias in Theorem 4 of \cite{bartlett2019} & Common variance\\
    \hline
    $p \ge n\cdot \xi$    & \rule{0pt}{0.45cm}$\|\beta\|^2_{\sz} \cdot p / (n\cdot {\xi}  )$ \rule[-0.3cm]{0pt}{0.3cm}& \rule{0pt}{0.45cm} $\|\beta\|^2_{\sz} \cdot p / (n\cdot {\xi}  )$\rule[-0.3cm]{0pt}{0.3cm} & 
    \multirow{4}{*}{\rule{0pt}{2cm}$\sep^2\log n \left\{  ({n}/{p}) + ({K}/{n}) \right\} $}
     \\
    \cline{1-3}
    $p\ll n\cdot \xi$ &\rule{0pt}{0.6cm} $\|\beta\|^2_{\sz} \cdot p / (n\cdot {\xi}  )$ \rule[-0.3cm]{0pt}{0.3cm}&\rule{0pt}{0.45cm} $\|\beta\|^2_{\sz} \cdot \sqrt{p / (n\cdot {\xi}  )}$\rule[-0.5cm]{0pt}{0.5cm}&\\
    \cline{1-3}
    $\xi \approx p$,\ $\|\beta\|^2_{\sz}\approx K$ &\rule{0pt}{0.6cm} $K/n$ \rule[-0.3cm]{0pt}{0.3cm}&\rule{0pt}{0.45cm} $K/\sqrt{n}$\rule[-0.5cm]{0pt}{0.5cm}&\\
    \cline{1-3}
    \makecell{$\xi \approx p$,\ $\|\beta\|^2_{\sz}\approx K$,\\ $K \approx n^{3/4}$ }&\rule{0pt}{0.8cm} $n^{-1/4}$ \rule[-0.3cm]{0pt}{0.3cm}&\rule{0pt}{0.45cm} $n^{1/4}$\rule[-0.5cm]{0pt}{0.5cm}&\\
    \hline
    \end{tabular}
    \caption{Comparison of risk bounds for Gaussian data.
    % In the regime of the third row, when $K\approx n^{3/4}$, for example, our bound in Theorem \ref{thm:upper bound} converges to zero, but 
    }\label{table:bart}
\end{table}

% \begin{table}[ht]
% \caption{Multi-row table}
% \begin{center}
% \begin{tabular}{|c|c|}
%     \hline
%     X&\multirow{2}{*}{Multirow}\\
%     \cline{1-1}
%     X&\\
%     \hline
% \end{tabular}
% \end{center}
% \label{tab:multicol}
% \end{table}

%
% This assumption is satisfied, however, whenever the data are jointly Gaussian; we thus offer a comparison in Table \ref{table:bart} between the respective bounds on the bias and variance terms corresponding to our 
%  Theorem \ref{thm:upper bound} and Theorem 4 of \cite{bartlett2019}, respectively, under the common case when $(X,y)$ are Gaussian and model (\ref{model}) holds.

When the data are jointly Gaussian this assumption is, however, satisfied under model (\ref{model}). For this common case, Table \ref{table:bart} compares the respective bounds on the bias and variance terms corresponding to our 
Theorem \ref{thm:upper bound} and Theorem 4 of \cite{bartlett2019}, respectively.  Again, we emphasize that the results from \cite{bartlett2019} do not hold in general for our modeling setup, but can be used to obtain the bounds in Table \ref{table:bart} in the Gaussian case. 
% For the common case when model (\ref{model}) holds and the data is jointly Gaussian,  Table  \ref{table:bart}   compares the respective bounds on the bias and variance terms corresponding to our 
%  Theorem \ref{thm:upper bound} and Theorem 4 of \cite{bartlett2019}, respectively. 
% {\tiny {\color{red} Delete: Do the entries in this Table correspond to $\|\sx\| = 1$ in \cite{bartlett2019} and  $\|\sx\| $ unbounded, for us ?  If so, say it !!!} }
  The entries in the second column of Table \ref{table:bart} correspond to  the bias in \cite{bartlett2019} under  model (\ref{model}), simplified in this table for ease of comparison\footnote{
For simplicity, we assume for this comparison that the matrices  $\sx$ and $\se$ are invertible and  that the condition numbers $\ke $ and $\kaz$ are bounded above by an absolute constant. Consequently, the effective rank ${\rm r_e} (\se)$ satisfies $c\cdot p \le {\rm r_e} (\se)\le  p $, for some $c\in (0,1)$.
}. 

In the setting of this comparison, the variance terms in our Theorem \ref{thm:upper bound} and the bound in \cite{bartlett2019} have the same rate, which we display in the third column of Table \ref{table:bart}. From the first row of Table \ref{table:bart} we see that when $p \geq n\cdot \xi $, the bias terms match as well. However,  this is not an interesting regime, as $p\ll n\cdot \xi$  is a necessary condition for either bound to converge to zero (assuming $\|\beta\|^2_{\sz}$ is bounded below). In this case, the second row of Table \ref{table:bart} shows that the bias in \cite{bartlett2019} becomes $\|\beta\|^2_{\sz}\sqrt{p /(n\cdot \xi)}$, which is larger than our bias bound in Theorem \ref{thm:upper bound} by a factor of  $\sqrt{n\cdot \xi / p }$.
From the second row we see that indeed, the upper bound on the excess risk in \cite{bartlett2019} can diverge while our bound in Theorem \ref{thm:upper bound} vanishes. For instance, if $\beta$ is a non-sparse vector in $\R^K$ with $\|\beta\|^2_{\sz} \approx K$, this phenomenon occurs if the signal-to-noise ratio $\xi$ lies in the range $ K p/ {n} \lesssim \xi \lesssim K^2 {p}/ {n}$. This illustrates that the general bound provided in \cite{bartlett2019} is not  always tight. 

The third row of Table \ref{table:bart} compares the bias rates in the simplified case when $\|\beta\|^2_{\sz}\approx K$ and $\xi\approx p$. The fourth row gives the rates under the further assumption that $K \approx n^{3/4}$, a concrete example of when our rate converges and that of \cite{bartlett2019} diverges.
Further details and discussion on the comparison of these two results are deferred to Appendix \ref{sec:finite sample}.\\

A latent factor regression model similar to (\ref{model}) has also been studied in  \cite[Section 7]{mei2019generalization} for the ridge regression estimator that minimizes the fit $\| \y - \X a \|^2+ \lambda \| a\|^2$ for any $\lambda>0$ (strict). Their model is a particular case of our model  (\ref{model}),  with $\se = \sigma_E^2 I_p$, $\sz = \sigma_Z^2 I_K$, up to an offset on $X$ so that in their case, $|\EE[X]|>0$.  Clearly, our estimator $\a$ can be viewed as the limiting case $\lambda=0$ of ridge regression.
Our results are difficult to compare directly since the  analysis in \cite{mei2019generalization} is asymptotic with $p/K\to \psi_1$ and $ n/K\to \psi_2$ for two absolute constants $\psi_1,\psi_2\in (0,\i)$.
Nevertheless, \cite[Theorem 7 and Figure 9]{mei2019generalization} also show that the excess risk $R(\wh \alpha)-\sigma_\eps^2$  is small in the large $\psi_1/\psi_2$ (corresponding to a  large $p/n$) regime, in line with our assessment.

\subsection{Comparison to other predictors}\label{sec:pcr}

In Lemma \ref{thm:pcr e=0} of Section \ref{sec:noiseless} above we showed that in the case of noiseless features, when $\se=0$, the regression vector $\wh\alpha_{\rm PCR}$ obtained by PCR is exactly equal to the GLS regression vector $\a$ on the event $\{\rank(\Z)=K\}$, which holds with probability at least $1-c/n$ for some universal constant $c>0$. In this section we show that when $\se \neq 0$, the minimum-norm estimator $\a$ is competitive even with the stylized %the population
version   $\wt \alpha_{\rm PCR} \coloneqq U_K(\X U_K)^+ \y$ of PCR under the factor regression model setting (\ref{model}) and in the high-dimensional regime $p\gg n$. This is a toy estimator as it uses the unknown dimension $K$ and  unknown matrix  $U_K$, composed of the first $K$ eigenvectors of the population covariance matrix $\sx$, in place of estimates $\wh K$ and $\wh U_{\wh K}$, respectively. We provide a simple proof, found in Appendix \ref{proof:TPCR}, of the following risk bound for $R(\wt \alpha_{\rm PCR})$. For a detailed comparison of PCR and the GLS, see \cite{bing2020prediction}, which analyzes the PCR predictor with the empirical matrix $\wh U_{\wh K}$, for a new, data  adaptive,  estimator $\wh K$ of $K$.

\begin{thm}\label{TPCR}
Under model (\ref{model}), suppose that $(X,y)$ are jointly Gaussian and that Assumption \ref{ass:fullrank} holds. % and that $\lp(\se)>0$. 
Then, if $n > C\cdot K\log n$ for some $C>0$ large enough, with probability at least $1-c/n$,
\begin{equation}\label{PCR} R(\wt \alpha_{\rm PCR}) - \sep^2 \lesssim \| \se\| \cdot \|{\ab}\|^2\frac{p}{n} + R({\ab}) \frac{ K \log (n) }{n}
\end{equation}
In particular, if 
$\se=0$, we obtain 
\begin{equation}\label{PCR2} R(\wt \alpha_{\rm PCR}) - \sep^2 \lesssim \sigma_\eps^2 \frac{ K \log (n) }{n}
\end{equation}
while, if $\lp(\se) > 0$,
%if the condition number  $\kappa(\se): = \lambda_1(\se)/ \lambda_p(\se)$  of the matrix  $\se$ is finite,
\begin{equation}\label{PCR3} R(\wt \alpha_{\rm PCR})- \sep^2 \lesssim \ke \frac{\|\beta\|^2_{\sz}}{\xi}\frac{p}{n} +\sep^2 \frac{K \log n }{n},\end{equation}
where $\kappa(\se): = \lambda_1(\se)/ \lambda_p(\se)$ is the condition number of the matrix  $\se$.
\end{thm}

%  I think the proof of (\ref{PCR}) is clean. (\ref{PCR2}) is an immediate consequence (with $R(\alpha)^*)= \sigma_\eps^2$ in this case). Finally, for the last bound, I wasn't successful. Essentially the proof is
% \begin{eqnarray*}
% {\ab} &=& \sx^+ \bar A \bar \beta\\
% &=& \se^{-1} \bar A G^{-1} \bar \beta\end{eqnarray*}
% and
% \begin{eqnarray*}
% \| {\ab} \|^2 &\le& \| \se^{-1} \bar A G^{-1} \|^2 \|\bar \beta\|^2\\
% &\le& \| \se^{-1/2}\|^2 \| \se^{-1/2}  \bar A G^{-1} \|^2 \|\bar \beta\|^2\\
% &\le& \lambda_p( \se ) \| G^{-1} \|^2 \|\bar \beta\|^2\\
% &\le& \ke \frac{ \| \beta\|_{\sz}^2 }{ \lambda_K ( A \sz A^\top) }
% \end{eqnarray*}

% The third step requires some thought - the argument you gave. I vote not to delete stuff, so we should keep this here, saying it is not the focus of the paper, but we provide a simple proof, and refer for detailed theoretical comparisons in different paper.\\

% We have the upper bound for $\| \ab\|$ in case $\se>0$ in the section about the paradox earlier. So (\ref{PCR3}) really would be an immediate consequence of (\ref{PCR}).\\

% Perhaps it is of interest to discuss if $\se=0$? I think so, even though we don't have interpolation; he story should be the same $\|\ab\| \to 0$, but $\| \ab\|_{\sx}\not\to0$.
% That should be moved to the paradox section.

Provided $\ke$ is bounded above by an absolute constant, the upper bounds for the minimum-norm and PCR predictors are comparable.
Indeed, when $\ke < C<\i$, the risk bound of Theorem \ref{thm:upper bound} for the GLS $\a$ takes the form
\begin{equation}
    R(\a) - \sep^2 \lesssim \frac{\|\beta\|^2_{\sz}}{\xi}\frac{p}{n} +\sep^2 \log n \left( \frac{K}{n} +  \frac{n}{p}\right).
\end{equation}
% Indeed,   Theorem \ref{thm:upper bound} implies the bound
% , when $\ke$ is bounded above by an absolute constant, the risk bound corresponding to  the GLS $\a$ takes the form
% \begin{equation}
%     R(\a) - \sep^2 \lesssim \frac{\|\beta\|^2_{\sz}}{\xi}\frac{p}{n} +\sep^2 \log n \left( \frac{K}{n} +  \frac{n}{p}\right)
% \end{equation}
% for the excess risk of the minimum-norm predictor.
The additional term $\sep^2  n\log n/ p$ in this bound  is  absent in the PCR prediction bound  (\ref{PCR3}) above, but in the  regime $p \gg n$ it can become negligible. It is perhaps surprising that under the factor regression model, the interpolator $\a$ can not only provide consistent prediction, but can in fact have excess risk comparable to a genuine $K$-dimensional predictor widely used in practice and tailored to the problem setting. This is despite the fact that the GLS interpolates the data (when $\rank(\X)=n$) and requires no tuning parameters or knowledge of the underlying dimension $K$. We emphasize that we do not claim that the GLS is necessarily a superior predictor to PCR in this setting. Rather, we observe the perhaps surprising fact that these two methods are comparable under the conditions stated. 

% {\color{green}[refer to other works that show this surprising fact].}
Figure \ref{fig:risk compare} plots the excess prediction risk of the GLS and PCR predictors. We also include the excess prediction risks of the LASSO, Ridge regression, and the null estimator $\0$ in this figure for comparison. 
% {\marten Each point is based on how many simulations?} 
The tuning parameters for LASSO and Ridge regression were chosen by cross-validation. We see that the peak in the GLS risk at $\gamma = p/n=1$ is not present in the PCR,  LASSO  and Ridge risks. This is due to the fact that these methods are regularized at this point, and in particular do not interpolate the training data. As $\gamma$ increases, and thus $p\gg n$, the GLS risk approaches the PCR risk, as indicated by the discussion above. The plot shows how the Ridge risk also approaches the common value of the PCR and GLS risks. Recalling that GLS is a limiting case of Ridge regression with regularization parameter $\lambda\to 0$, this suggests that for $p\gg n$, in our setting, the optimal choice of regularization parameter for ridge regression approaches zero \cite{mei2019generalization,hastie2019}.
%%{\color{red} Add: as previously observed in the Montanari/Tibshirani paper}. 
\begin{figure}[h!]
    \centering
    \includegraphics[width=12cm]{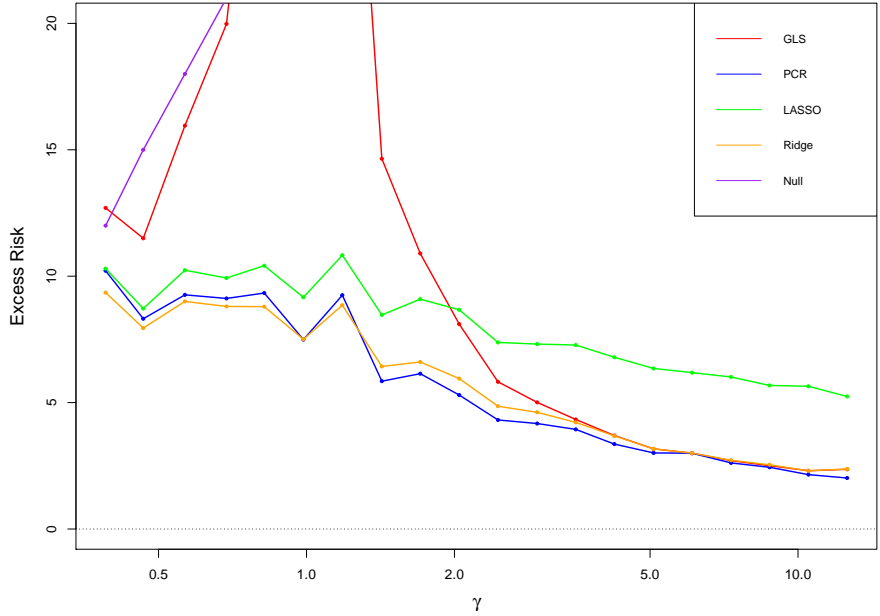}
    \caption{\small Excess prediction risk of GLS, PCR, LASSO, Ridge regression, and the null predictor as a function of $\gamma = p/n$. 
    Here $K$ increases linearly from $12$ to $69$, $n = \lfloor K^{1.5}\rfloor$ and thus increases from $41$ to $573$, and $p$ increases from $16$ to $7215$.
    Further, $\se = I_p$, $\sz = I_K$, $\beta = (1 ,\ldots,1)^\top$, and $A$ is generated by sampling each entry iid from $N(0,1/\sqrt{K})$. }
    \label{fig:risk compare}
\end{figure}
We plot the coefficients of $\ab$ in  Figure \ref{fig:alpha star} for the case $p = 7215$ and $K=69$. We can see that $\ab$ is clearly non-sparse, which explains  the inferior  performance of the LASSO in this setting.
For completeness, we contrast the above simulation setting in which $\ab$ is non-sparse with special case in which $\ab$ is in fact $K$-sparse. In this case, we
  take the matrix $A$ with   columns equal to the canonical basis vectors $e_1,\ldots,e_K\in \R^p$, multiplied by %$a=
  $\sqrt{p}$, and we set $\beta=(1,\ldots,1)^\top$,  $\sz= I_K$ and $\se=I_p$.
  %$\sz= \sigma_Z^2I_K$ and $\se= \sigma_E^2I_p$ with $\sigma_Z^2 = \sigma_E^2=1$. 
  Then $A^\top A=p I_K$ and $\ab$ is $K$-sparse
 since, by (\ref{explicit}) of Remark \ref{rem:sparse},
  \[{\ab}_i = \begin{cases} \sqrt{p}/ (p+1) & %\beta_i \{ a\sigma_Z^2/ (a^2 \sigma_Z^2+ \sigma_E^2) \} & 
  \text{for } i =1,\ldots,K\\  0 & \text{for } i = K+1,\ldots,p\end{cases}.\]
Figure \ref{fig:risk compare sparse} plots the excess risk of the GLS and other predictors for these model settings.
% , and Figure \ref{fig:alpha hist sparse} shows corresponding histograms of the components of $\ab$ and its estimates for this simulation at the point where $p=7215$.
We see that in this sparse setting the LASSO performs well, as expected, with its excess risk approximately equal to that of PCR for $p\gg n$, both of which do slightly better than GLS and Ridge. 
% The histograms in Figure \ref{fig:alpha hist sparse} corroborate this performance, showing that $\ab$, as well as its estimates by LASSO and PCR, are sparse, whereas the estimates of $\ab$ by GLS and Ridge are not, although they are more concentrated around zero compared to the case in Figure \ref{fig:alpha hist}. 
While LASSO and PCR outperform GLS in this case, we note that the excess risk of the GLS still decreases towards zero, and performs perhaps surprisingly well relative to the LASSO, given that the LASSO is specifically tailored to this exactly sparse setting. Moreover, we emphasize that for more generic choices of model parameters, $\ab$ will not necessarily be sparse or even approximately sparse, and we should expect the GLS to outperform the LASSO (see Remark \ref{rem:sparse} for further comment).

 The take-home message is that for $\gamma=p/n$ large enough, the GLS is a surprisingly competitive predictor, given its interpolating property, and in fact performs as well in the generic setting of Figure \ref{fig:risk compare} as the PCR predictor chosen with the unknown, optimal number of components $K$, in addition to Ridge regression with tuning parameter chosen by cross-validation. Even when the model parameters are carefully chosen so that the best linear predictor $\ab$ is $K$-sparse, the GLS performs not much worse than the LASSO, which is tailored to this setting, provided that $p$ is very large.  \\
 
\begin{figure}[h!]
    \centering
    \includegraphics[width=11cm]{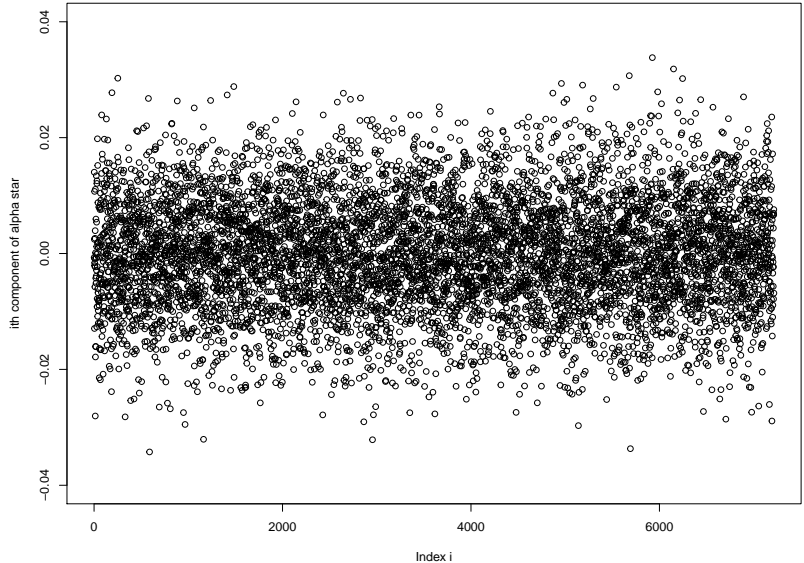}
    \caption{\small A scatter plot of the components of $\ab$, from the point in the simulation of Figure \ref{fig:risk compare} with the largest value of $\gamma$. Here $p = 7215$, $K=69$, $\se = I_p$, $\sz = I_K$, and $A$ is generated by sampling each entry iid from $N(0,1/\sqrt{K})$.}
    \label{fig:alpha star}
\end{figure}
\begin{figure}[h!]
    \centering
    \includegraphics[width=12cm]{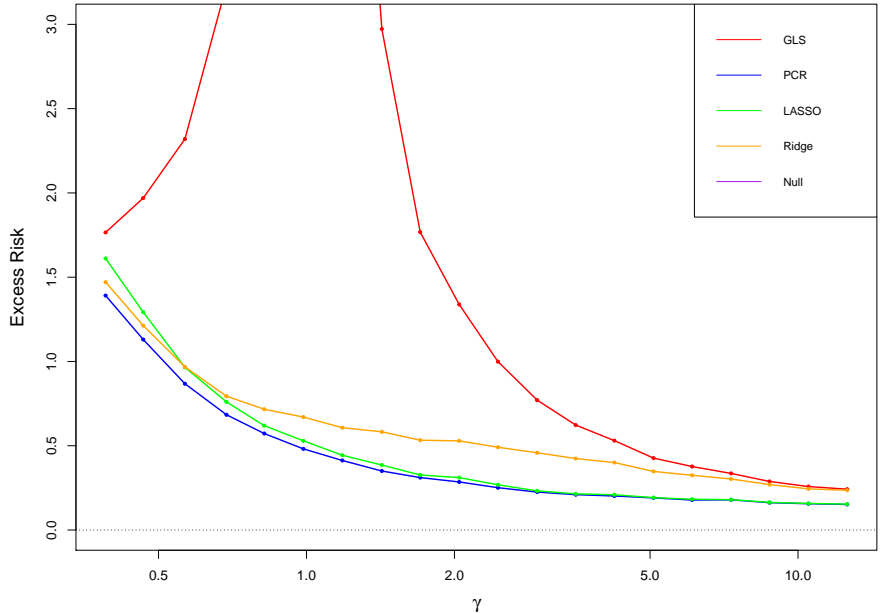}
    \caption{\small Excess prediction risk of GLS, PCR, LASSO, Ridge regression, and the null predictor as a function of $\gamma = p/n$. Null risk is not visible on plot since it is larger than the maximum plotted value. Here $K$ increases linearly from $12$ to $69$, $n = \lfloor K^{1.5}\rfloor$ and thus increases from $41$ to $573$, and $p$ increases from $16$ to $7215$.
    Further, $\se = I_p$, $\sz = I_K$, $\beta = (1 ,\ldots,1)^\top$, and $A$ has   columns equal to the canonical basis vectors $e_1,\ldots,e_K\in \R^p$, multiplied by $\sqrt{p}$. %{\seth NEED TO ADD UPDATED PLOT ONCE COMPLETE; this one has fewer values of gamma than final one.}
    }
    \label{fig:risk compare sparse}
\end{figure}
% \begin{figure}[h!]
%     \centering
%     \includegraphics[width=12cm]{Interpolation_factor_model/alpha_compare_estimates_k_sparse_i=20.pdf}
%     \caption{\small Histograms of the components of $\ab$ and its estimates from LASSO, GLS, PCR, and Ridge regression, from the simulation of Figure \ref{fig:risk compare sparse}. Here $p = 7215$, $K=69$, $\se = I_p$, $\sz = I_K$, and $A$ has   columns equal to the canonical basis vectors $e_1,\ldots,e_K\in \R^p$, multiplied by $\sqrt{p}$.}
%     \label{fig:alpha hist sparse}
% \end{figure}

\subsection*{Acknowledgements.}
	We thank the anonymous referees for their many insightful and helpful suggestions.
	Bunea and Wegkamp are supported in part by NSF grant DMS-1712709.  \\

\bibliography{refs} 
\bibliographystyle{plain}

%\printbibliography 

\newpage
\appendix

\section{Proofs for the main text}
\subsection{Proofs for Section \ref{sec:interpolation_null_risk}}\label{proofs:null}

\subsubsection*{Proof of Theorem \ref{thm:alpha null risk}}\label{proof:alpha null risk}
We work on the event
\begin{equation}
    \mathcal{K} \coloneqq \left\{\sn^2(\X)\gtrsim \tr(\sx),\ \|\y\|^2\lesssim n\sy^2 \right\}.
\end{equation}
On this event, recalling $\a = \X^+\y$ and invoking identity (\ref{eq:X+norm}) in Appendix \ref{sec:pseudo-inverse},
\begin{equation}\label{a norm bnd sx}
    \|\a\|^2\le \|\X^+\|^2\|\y\|^2= \frac{\|\y\|^2}{\sn^2(\X)}\lesssim \sy^2\frac{n}{\tr(\sx)}.
\end{equation}
By Lemma \ref{thm:theta null risk} below, 
\[\left |\frac{R(\theta)}{R(\0 )} - 1\right | \le \frac{\|\theta\|^2_{\sx}}{R(\0 )} + 2\sqrt{\frac{\|\theta\|^2_{\sx}}{R(\0 )}}\le \|\sx\|\frac{\|\theta\|^2}{R(\0 )} + 2\sqrt{\|\sx\|\frac{\|\theta\|^2 }{R(\0 )}} \]
for any vector $\theta\in \R^p$. Combining this with (\ref{a norm bnd sx}) and recalling that $\sy^2 = \EE[y^2]=R(\0)$, we find that on $\mathcal{K}$,
\[\left |\frac{R(\a)}{R(\0 )} - 1\right |\lesssim \frac{n}{{\rm r_e} (\sx)} + \sqrt{\frac{n}{{\rm r_e} (\sx)}}\]
% By Theorem \ref{thm:a norm}, if ${\rm r_e} (\sx)>C\cdot n$ for $C>0$ large enough, then with probability at least $1-ce^{-c'n}$,
% \begin{equation}
%     \frac{\|\a\|^2_{\sx}}{R(\0 )}\lesssim \frac{n}{{\rm r_e} (\sx)}.
% \end{equation}
% Thus, by Theorem \ref{thm:theta null risk}, with at least the same probability,
% \[\left |\frac{R(\a)}{R(\0 )} - 1\right | \le \frac{\|\a\|^2_{\sx}}{R(\0 )} + 2\sqrt{\frac{\|\a\|^2_{\sx}}{R(\0 )}}\lesssim \frac{n}{{\rm r_e} (\sx)} + \sqrt{\frac{n}{{\rm r_e} (\sx)}}.\]
Setting $C'=\max(C,1)$, when ${\rm r_e} (\sx)>C'n \ge n$, so $n/{\rm r_e} (\sx) > 1$, we find
\[\frac{n}{{\rm r_e} (\sx)} + \sqrt{\frac{n}{{\rm r_e} (\sx)}}\le 2 \sqrt{\frac{n}{{\rm r_e} (\sx)}}.\]
Thus, on $\mathcal{K}$,
\[\left |\frac{R(\a)}{R(\0 )} - 1\right | \lesssim \sqrt{\frac{n}{{\rm r_e} (\sx)}}.\]
All that remains is to bound the probability of $\mathcal{K}$. To this end, note that since we suppose Assumption \ref{ass:x} holds, we have $\X = \tX \sx^{1/2}$, and thus
\[\sn^2(\X) = \ln(\X\X^\top)= \ln(\tX \sx \tX),\]
where $\tX$ has i.i.d.~entries that have zero mean, unit variance, and sub-Gaussian constants bounded by an absolute constant. Theorem \ref{thm:concentration spectrum trace} below thus implies that if ${\rm r_e} (\sx)>C\cdot n$ for $C>0$ large enough, then with probability at least $1-2e^{-cn}$,
\[\sn^2(\X) \ge \tr(\sx)/2 - c_0 \|\sx\|n= \tr(\sx)\cdot[1/2 - c_0n/{\rm r_e}(\sx)].\]
Using that $n/{\rm r_e}(\sx)< 1/C$ and choosing $C$ large enough,
\begin{equation}\label{eqn:snx lb app for a norm}
    \PP(\sn^2(\X)\gtrsim \tr(\sx)) \ge 1 - 2e^{-cn}.
\end{equation}
By Assumption \ref{ass:x}, $\y = \sy\tilde \y$. Since $\tilde y_1,\ldots, \tilde y_n$ have zero mean and sub-Gaussian constants bounded by an absolute constant, Bernstein's inequality (Corollary 2.8.3 of \cite{verHDP}) implies that 
\[\PP(\|\tilde \y\|^2\gtrsim n) = \PP\left(\left |\sum_{i=1}^n {\tilde y_i}^2 \right|\gtrsim n\right)\le 2e^{-2cn}.\]
Thus,
\[\PP(\| \y\|^2\gtrsim \sy^2 n) = \PP(\sy^2\| \tilde \y\|^2\gtrsim \sy^2 n) = \PP(\| \tilde \y\|^2\gtrsim n)\le 2e^{-2cn}.\]
Combining this with (\ref{eqn:snx lb app for a norm}) establishes that $\PP(\mathcal{K}) \ge 1 - ce^{-c'n}$, thus completing the proof. \hfill $\blacksquare$

\subsubsection*{Lemma \ref{thm:theta null risk} and Theorem \ref{thm:concentration spectrum trace}}
The proof of Theorem \ref{thm:alpha null risk} above made crucial use of the following lemma and theorem. 
\begin{lemma}\label{thm:theta null risk}
For any vector $\theta \in \R^p$,
\begin{equation}
    \left |\frac{R(\theta)}{R(\0 )} - 1\right| \le \frac{\|\theta\|^2_{\sx}}{R(\0 )} + 2 \sqrt{\frac{\|\theta\|^2_{\sx}}{R(\0 )}}.
\end{equation}
\end{lemma}
\begin{proof}
We first show that $\sx \ab = \sxy$, where $\sxy \coloneqq \EE[Xy]$ and $\ab \coloneqq \sx^+\sxy$. To this end, observe that
\begin{align*}
    \cov( (I - \sx\sx^{+})X) &= (I_p - \sx\sx^{+})\EE[XX^\top](I_p - \sx\sx^{+})\\
    &= (I_p - \sx\sx^{+})\sx(I_p - \sx^+\sx)\\
    &= 0, %& (\text{since } \sx\sx^+\sx = \sx\text{, see Appendix \ref{sec:pseudo-inverse}})
\end{align*}
where we use that $\sx\sx^+\sx = \sx$ (see Appendix \ref{sec:pseudo-inverse}). Thus $(I_p - \sx\sx^{+})X = 0$ a.s., so
\begin{equation}\label{eqn:sx sxy identity}
    \sx\ab = \sx\sx^{+}\sxy = \EE[\sx\sx^{+}Xy] = \EE[Xy] = \sxy.
\end{equation}
Fixing $\theta \in \R^{p}$, we have
\begin{align*}
    R(\theta) - R(\0) &= \EE[(X^\top\theta - y)^2] - \EE[y^2]\\
    &= \theta^\top \EE[XX^\top ]\theta -2\theta^\top \EE[Xy]\\
    &= \|\theta\|^2_{\sx} - 2 \theta^\top \sxy\\
    &= \|\theta\|^2_{\sx} - 2 \theta^\top \sx \ab & (\text{by } (\ref{eqn:sx sxy identity})),
\end{align*}
so by the Cauchy-Schwarz inequality,
\begin{equation}\label{eqn:rtheta decomp2a}
    |R(\theta)- R(\0 )| \le \|\theta\|^2_{\sx} + 2\|\theta\|_{\sx}\|{\ab}\|_{\sx}.
\end{equation}
Next observe that
\[R(\0 ) = \EE[y^2] = \EE(y - X^\top \ab + X^\top {\ab})^2 = R({\ab}) + \|{\ab}\|^2_{\sx} \ge \|{\ab}\|^2_{\sx},\]
where we use that by (\ref{eqn:sx sxy identity}),
\[\EE(X^\top{{\ab}})(X^\top{\ab} -y) = {\alpha^{*\top}}\sx\ab - {\alpha^{*\top}}\sxy = 0.\]
Thus, $\|\ab\|^2_{\sx}\le R(\0 )$, so by (\ref{eqn:rtheta decomp2a}),
\begin{equation}\label{eqn:rtheta decomp2}
    |R(\theta)- R(\0 )| \le \|\theta\|^2_{\sx} + 2\|\theta\|_{\sx}\sqrt{R(\0 )}.
\end{equation}
Dividing both sides by $R(\0 )$ gives the final result.
\end{proof}

\begin{thm}\label{thm:concentration spectrum trace}
Suppose $\W$ is an $n\times r$ random matrix with independent subgaussian entries that have zero mean and unit variance. Then for any positive semi-definite matrix $\sg\in \R^{r\times r}$ and some $c' > 0$ large enough, with probability at least $1 -2e^{-cn}$,
\[\tr(\sg)/2 - c'(M^2+M^4)\|\sg\|n \le \ln(\W\sg \W^\top)\le \lambda_1(\W \sg \W^\top)\le 3\tr(\sg)/2 + c'(M^2+M^4)\|\sg\|n,\]
where $M \coloneqq \max_{i,j}\|\W_{ij}\|_{\psi_2}$\footnote{We define the sub-Gaussian norm of any real-valued random variable $U$ by $\|U\|_{\psi_2}\coloneqq \inf\{t>0: \EE \exp(U^2/t)<2\}$. We say $U$ is sub-Gaussian when $\|U\|_{\psi_2}<\i$.}.
\end{thm}

A similar result for diagonal  $\Sigma$  has been derived in Lemma 9 of \cite{bartlett2019}.
We make use of the Hanson-Wright inequality in our proof to deal with non-diagonal  $\Sigma$.
Theorem 4.6.1 in \cite{verHDP} provides similar two-sided bounds for the smallest and largest  eigenvalue of $\W\sg\W^\top$, when $\sg = I_r$. 
% Since we could not find a proof of Theorem \ref{thm:concentration spectrum trace}, we include it below to bridge the gap between these two results.
\begin{proof}
%\subsubsection*{Reduction to $M= 1$ case\footnote{\color{red} Is this section trivial, should we just remove it? Can just say ``We can take $M=1$ without loss of generality".}}
%\subsubsection*{Reduction to $M= 1$ case\footnote{\color{red} Is this section trivial, should we just remove it? Can just say ``We can take $M=1$ without loss of generality".}}
We will prove that for some $c' \ge 1$,
\begin{equation}\label{eqn:op norm conc}
\|\W\sg \W^\top - \tr(\sg)I_n\| \le c'(M^2+M^4)\|\sg\|n + \tr(\sg)/2
\end{equation}
with probability at least $1-2e^{-cn}$. Equation (\ref{eqn:op norm conc}) implies that for any $v\in\R^n$ with $\|v\|=1$,
\[|v^\top\W\sg \W^\top v - \tr(\sg)| \le c'(M^2+M^4)\|\sg\|n + \tr(\sg)/2,\]
and so
\[\tr(\sg)/2 - c'(M^2+M^4)\|\sg\|n\le v^\top\W\sg \W^\top v \le 3\tr(\sg)/2 + c'(M^2+M^4)\|\sg\|n.\]
Taking the minimum and maximum over $v\in S^{n-1}$ then gives the desired result.\\

We now prove (\ref{eqn:op norm conc}). Let $\N$ be a $1/4$-net of $S^{n-1}$ with $|\N|\le 9^n$, which exists by Corollary 4.2.13 of \cite{verHDP}. Then by Exercise 4.4.3 of \cite{verHDP},

\begin{equation}\label{eqn:eps net approx}
\|\W\sg \W^\top
- \tr(\sg)I_n\| = \sup_{v\in S^{n-1}}|v^\top\W\sg \W^\top v - \tr(\sg)| \le 2 \sup_{v\in \N}|v^\top\W\sg \W^\top v - \tr(\sg)|,
\end{equation}
where we use that $\W\sg \W^\top - \textrm{tr}(\Sigma) I_n$ is symmetric in the first step.% {\color{red}(check if that's true)}.

Now fix $v\in S^{n-1}$ and define $B = \W^\top v\in \R^r$. Observe that $B$ has mean zero entries that are independent because the columns of $\W$ are independent. Furthermore, by Proposition 2.6.1 of \cite{verHDP},
\[\|B_i\|_{\psi_2}^2 = \|\sum_{j}\W_{ji}v_j\|^2_{\psi_2}\le C\sum_{j}\|\W_{ji}\|_{\psi_2}^2v_j^2 \le \max_{l i}\|\W_{l i}\|_{\psi_2}^2\sum_{j}v_j^2 = CM^2,\] 
where we used $\|v\|^2=1$ in the last step.
Thus, by the Hanson-Wright inequality (Theorem 6.2.1 in \cite{verHDP}),
\begin{equation}\label{eqn:hanson-wright 1}
\PP\left(|B^\top\sg B - \EE B^\top\sg B| \ge c_1M^2t\right) \le 2 \exp\left\{-c_2\min\left(t/\|\sg\|, t^2/\|\sg\|^2_F\right)\right\},
\end{equation}
where we can choose  $c_1>0$ large enough such that $c_2\ge 12$.

Note that 
\begin{equation}\label{eqn:B expectation}
\EE B^\top\Sigma B = \sum_{i,j,k,l}\EE v_i\W_{ij}\sg_{jl}\W_{kl}v_k= \sum_{ij} v_i^2 \sg_{jj}\EE \W_{ij}^2 = \|v\|^2 \tr(\sg) = \tr(\sg),
\end{equation}
where in the second step we use that $\W$ has independent mean zero entries, in the third step we use that $\EE \W_{ij}^2 =1$ for all $i,j$, and in the final step we use that $\|v\|=1$.

Choosing $t=\|\sg\|n/2 + \sqrt{n\|\sg\|_F^2}/2$ in (\ref{eqn:hanson-wright 1}) and using that $c_2 \ge 12$, we observe that
\[c_2t/\|\sg\| = c_2n/2 + c_2\sqrt{n\|\sg\|_F^2}/(2\|\sg\|)\ge c_2n/2\ge 3n,\]
and
\[c_2t^2/\|\sg\|_F^2 = c_2\left[n\|\sg\|/(2\|\sg\|_F) + \sqrt{n}/2 \right]^2\ge c_2n/4\ge 3n.\]
Thus,
\begin{equation}\label{eqn:hanson-wright bf amgm}
\PP\left(|B^\top\sg B - \tr(\sg)| \ge c_1M^2\|\sg\|n/2 + c_1M^2\sqrt{n\|\sg\|_F^2}/2\right) \le 2e^{-3n},
\end{equation}
where we used (\ref{eqn:B expectation}). Finally, using
\[\|\sg\|_F^2 = \tr(\sg^2) \le \|\sg\|\tr(\sg),\]
and the inequality $ 2{ab}\le a^2+b^2$,
\[c_1M^2\sqrt{n\|\sg\|_F^2}/2 \le c_1M^2\sqrt{(c_1M^2 n\|\sg\|)(\tr(\sg)/c_1M^2)}/2\le c_1^2M^4n\|\sg\|/4 + \tr(\sg)/4.\]
Thus, by (\ref{eqn:hanson-wright bf amgm}), and for $c'>0$ large enough,
\begin{equation}\label{eqn:hanson-wright af amgm}
\PP\left(|B^\top \sg B - \tr(\sg)| \ge c'(M^2+M^4)\|\sg\|n + \tr(\sg)/4\right) \le 2e^{-3n}.
\end{equation}
Denoting $c'(M^2+M^4)\|\sg\|n + \tr(\sg)/4$ by $L$, we thus have
\begin{align*}
    \PP\left(\|\W\sg \W^\top - \tr(\sg)I_n\| \ge 2L\right)& \le  \PP\left(2\sup_{v\in \N}|v^\top\W\sg \W^\top v - \tr(\sg)| \ge 2L\right)&(\text{by } (\ref{eqn:eps net approx}))\\
    &\le \sum_{v\in\N}\PP\left(|v^\top \W\sg \W^\top v - \tr(\sg)| \ge L\right) &(\text{union bound})\\
    &\le 2\times 9^n e^{-3n} &(\text{by } (\ref{eqn:hanson-wright af amgm}))\\
    & = 2e^{ n\log(9) - 3n}\le 2e^{-cn},
\end{align*}
where we define $c = 3 - \log(9)>0$ in the last step. This shows (\ref{eqn:op norm conc}) and completes the proof.
\end{proof}

\subsection{Proofs for Section \ref{sec:frm}} 

\subsubsection{Proof of Lemma \ref{thm:spectrum} from Section \ref{sec:effrank} }\label{proof:spectrum}

We will use $\sx = \sza +\se$ and the min-max formula for eigenvalues,
 \begin{equation}\label{min-max spec}
     \lambda_{i}(\sx) = \min_{S:\text{dim}(S)=i} \max_{x\in S: \|x\|=1} x^\top \sx x,
 \end{equation}
 where the minimum is taken over all linear subspaces $S\subset \R^p$ with dimension $i$. We prove the three points one by one.
\begin{enumerate}
    \item Since for any $x\in \R^p$, $x^\top \sza x\ge 0$, we have
    \[x^\top \sx x \ge x^\top \se x,\]
    so by (\ref{min-max spec}), for any $i\in [p]$, 
    \[\lambda_i(\sx) \ge \lambda_i(\se) \ge \lp(\se)> c_2.\]
    \item For any $x\in \R^p$,
    \begin{align*}
    x^\top \sx x &= x^\top \sza x+ x^\top \se x\\
    &\ge x^\top \sza x\\
    &\ge \lk(\sz) x^\top AA^\top x\\
    &\ge c_1 \cdot x^\top AA^\top x.
    \end{align*}
    Plugging this into (\ref{min-max spec}) with $i=K$, we find $\lk(\sx)\ge c_1\lk(A^\top A)$ as claimed.
    \item For any $x\in \R^p$, $x^\top \se x\le \|\se\|$. Using this in (\ref{min-max spec}), we find for any $i > K$,
    \[\lambda_i(\sx)\le \|\se\| + \lambda_i(\sza) = \|\se\| < C_2,\]
    where in the second step we use that $\rank(\sza)\le K$, so $\lambda_i(\sza)=0$ for $i>K$. Combining this with $\lambda_i(\sx)>c_2$ from part 1 above completes the proof.
\end{enumerate}
\hfill $\blacksquare$
\subsubsection{Proof of Lemma \ref{thm:bench compare} from Section \ref{sec:bench}}\label{proof:bench compare}
Using $y=Z^\top\beta + \eps$ and the fact that $\eps$ is independent of $X$ and $Z$,
\begin{equation*}
    R(\ab) = \EE[({\alpha^{*\top}} X - y)]^2 = \EE[({\alpha^{*\top}} X - Z^\top\beta)]^2 + \sep^2\ge \sep^2,
\end{equation*}
which proves the first claim. Using $X = AZ+E$, we further find
\begin{equation}\label{eqn:ralpha exp}
    R(\ab) -\sep^2=\EE[({\alpha^{*\top}}  X - Z^\top\beta)]^2= {\alpha^{*\top}} \sx\ab + \beta^\top\sz\beta - 2{\alpha^{*\top}}  A\sz\beta.
\end{equation}
Now suppose $\se$ and $\sz$ are invertible as in the second claim. Then in particular, 
\[\lp(\sx)\ge \lp(\se)>0,\]
so $\sx$ is invertible and thus $\sx^+=\sx^{-1}$. Also, $\sxy = \EE[Xy] = A\sz\beta$, so 
\[\ab = \sx^+\sxy = \sx^{-1}A\sz\beta.\] 
Defining $\bA \coloneqq A\sz^{1/2}$ and $\bb \coloneqq \sz^{1/2}\beta$, we have $\ab = \sx^{-1}\bA\bb$. Plugging this into (\ref{eqn:ralpha exp}) and simplifying, we find
\begin{equation}\label{eqn:rstar - eps before woodbury}
    R(\ab) - \sep^2 = \bb^\top\left[I_K -  \bA^\top\sx^{-1}\bA\right]\bb.
\end{equation}
By the Woodbury matrix identity,
\[\sx^{-1} = (\bA \bA^\top + \se)^{-1} = \se^{-1} - \se^{-1} \bA (I_K + \bA^\top\se^{-1}\bA)^{-1}\bA^\top\se^{-1},\]
so letting $\bG\coloneqq I_K + \bA^\top\se^{-1}\bA$,
\[\bA^\top \sx^{-1} \bA =  \bA^\top\se^{-1} \bA  -  \bA^\top\se^{-1} \bA \bG^{-1} \bA^\top \se^{-1} \bA.\]
Now using $\bA^\top\se^{-1} \bA = \bG-I_K$, we find 
\begin{align*}
    \bA^\top \sx^{-1} \bA &=  (\bG-I_K) - (\bG-I_K)\bG^{-1}(\bG-I_K)\\
    &=  \bG - I_K - (I_K - \bG^{-1})(\bG - I_K)\\
    &= \bG  - I_K - [\bG - I_K- I_K + \bG^{-1}]\\
    &= I_K - \bG^{-1}.
\end{align*}
Using this to simplify (\ref{eqn:rstar - eps before woodbury}), we find
\begin{equation}\label{rab - sep}
    R(\ab) - \sep^2 = \bb^\top \bG^{-1}\bb = \bb^\top (I_K + \bA \se^{-1}\bA)^{-1}\bb.
\end{equation}
Letting $H \coloneqq \bA \se^{-1}\bA$, we find
\begin{equation}
    R(\ab) - \sep^2 = \bb^\top H^{-1/2}(I_K + H^{-1})^{-1}H^{-1/2}\bb.
\end{equation}
For the lower bound, first observe that
	\[R(\ab)-\sep^2 = \bb^\top H^{-1/2}(I_K + H^{-1})^{-1}H^{-1/2}\bb \ge\frac{\bb^\top H^{-1}\bb}{1+ \|H^{-1}\|}=  \frac{\beta^\top (A\se^{-1}A)^{-1}\beta}{1+ \lk^{-1}(H)}.\]
	Furthermore,
\begin{equation}\label{lkH lb}
    \lk(H) = \lk(\bA^\top \se^{-1}\bA)\ge \lk(\sza)/\|\se\|= \xi,
\end{equation}
	so using this in the previous display,
	\[R(\ab)-\sep^2  \ge  \frac{\beta^\top (A^\top \se^{-1} A)^{-1}\beta}{1+ \xi^{-1}} = \frac{\xi}{1+\xi} \cdot \beta^\top (A^\top \se^{-1} A)^{-1}\beta.\]
	
To obtain the upper bound on $R(\ab)$ we use
\[R(\ab) - \sep^2 = \bb^\top H^{-1/2}(I_K + H^{-1})^{-1}H^{-1/2}\bb\le \frac{\bb^\top H^{-1} \bb}{1 + \lk(H^{-1})}
    \le \bb^\top H^{-1} \bb
    =\beta^\top (A\se^{-1}A)^{-1}\beta,\]
    where in the last step we use $\sz^{1/2}H^{-1} \sz^{1/2}
    = (A\se^{-1}A)^{-1}$.
	Finally,
	\[\beta^\top (A\se^{-1}A)^{-1}\beta = \bb^\top H^{-1}\bb\le \|\beta\|^2_{\sz}/\lk(H)\le \|\beta\|^2_{\sz}/\xi,\]
	where we use (\ref{lkH lb}) in the last step.
% 	the claim 
% 	 \[\beta^\top (A^\top \se^{-1}A)^{-1} \beta\le \frac{1}{\xi}\|\beta\|^2_{\sz}\]
% 	 follows from $R(\ab) - \sep^2\le \bb^\top H^{-1}\bb$ and $lk(H) \ge \xi$.
% \begin{align*}
%     R(\ab) - \sep^2 &= \bb^\top H^{-1/2}(I_K + H^{-1})^{-1}H^{-1/2}\bb\\
%     &\le \frac{\bb^\top H^{-1} \bb}{1 + \lk(H^{-1})}\\
%     &\le \bb^\top H^{-1} \bb\\
%     &=\beta^\top (A\se^{-1}A)^{-1}\beta.
% \end{align*}
% \[\beta^\top G^{-1}\beta = \beta^\top\sz^{1/2} (\sz^{1/2}G\sz^{1/2})^{-1}\sz^{1/2}\beta \le \|\beta\|^2_{\sz}\|(\sz^{1/2}G\sz^{1/2})^{-1}\|= \|\beta\|^2/\lk(\sz^{1/2}G\sz^{1/2}),
% \]
% and that since $\sz^{1/2}G\sz^{1/2} = I_K + \sz^{1/2}A^\top\se^{-1}A\sz^{1/2}$,
% \[\lk(\sz^{1/2}G\sz^{1/2}) \ge \lk(\sz A^\top\se^{-1} A\sz ) \ge \lk(\sza)\lp(\se^{-1}) = \lk(\sza)/\|\se\|=\xi,\]
% which together give the desired bound. 

\hfill $\blacksquare$

% {\seth
% \subsubsection{Proof of Lemma \ref{thm:null vs opt} from Section \ref{sec:null vs opt}}\label{proof:null vs opt}

% Using $y = Z^\top\beta + \eps$ with $Z$ and $\eps$ independent and mean zero,
% \[R(\0) = \EE[y^2]= \|\beta\|_{\sz}^2 + \sep^2.\]
% Thus,
% \begin{align*}
%     R(\0) - R(\ab) &= R(\0) - \sep^2 + \sep^2 - R(\ab)\\
%     &\ge \|\beta\|_{\sz}^2 - \|\beta\|_{\sz}^2\cdot \xi^{-1} & (\text{by Lemma } \ref{thm:bench compare})\\
%     &= \|\beta\|_{\sz}^2\cdot (1-\xi^{-1}).
% \end{align*}
% \hfill $\blacksquare$
% }

\subsubsection{Proofs for Section \ref{sec:low dim}}\label{proofs:low dim}

\subsubsection*{Proof of Lemma \ref{thm:ab norm bound e=0}}

Let $\bA = A\sz^{1/2}$ and $\bb \coloneqq \sz^{1/2}\beta$. Using $\sx = \sza = \bA \bA^\top$, we find
\begin{equation}\label{ab formula proof}
    \ab = \sx^{+} \bA \bb = (\bA \bA^\top)^{+}\bA \bb = \bA^{+\top}\bb,
\end{equation}
where we use Lemma \ref{thm:psuedo lemma} in the last step. Using this formula, we obtain
\[\|\ab\|^2_{\sx} = \bb^\top \bA^+ (\bA\bA^\top) \bA^{+\top}\bb= \bb^\top\bb = \|\beta\|^2_{\sz},\]
where we use that $\bA$ is full rank since $A$ and $\sz$ are full rank, and thus $\bA^+ \bA=I_K$ by Lemma \ref{thm:psuedo lemma}.

Next, by identity (\ref{eqn:bc}) in Lemma \ref{thm:psuedo lemma}, and the fact that $A^+A=I_K$ and $\sz$ is invertible,
\[\bA^+ = (A\sz^{1/2})^+ = \sz^{-1/2}A^+.\]
Using this in (\ref{ab formula proof}) we find that $\ab = A^{+\top }\beta$, and thus
\[\|\ab\|^2 = \beta^\top A^+A^{+\top}\beta= \beta^\top (A^\top A)^{-1}A^\top A^{+\top}\beta,\]
where we use $A^+=(A^\top A)^{-1}A^\top$ by Lemma \ref{thm:psuedo lemma}. Thus, again using $A^+ A = A^\top A^{+\top}=I_K$, we find
\[\|\ab\|^2 = \beta^\top(A^\top A)^{-1}\beta,\]
as claimed.\hfill $\blacksquare$

\subsubsection*{Proof of Lemma \ref{thm:ab norm bound main}}
Defining $\bA = A\sz^{1/2}$ and $\bb = \sz^{1/2}\beta$, we have $\ab = \sx^{-1}\bA\bb$. Now recall that since $A$ and $\sz$ are full rank, so is $\bA$ and thus $\bA^+\bA = \bA^\top\bA^{+\top} = I_K$ (see Appendix \ref{sec:pseudo-inverse}). Thus,
\begin{align*}
    \ab &= \sx^{-1}\bA\bb \\ 
    &= \sx^{-1}\bA\bA^\top \bA^{+\top}\bb \\
    &= \sx^{-1}(\sx - \se)\bA^{+\top}\bb & (\text{since } \sx = \bA\bA^\top + \se)\\
    &= (I_p- \sx^{-1}\se)\bA^{+\top}\bb.
\end{align*}
By the Woodbury matrix identity applied to $\sx^{-1}=(\bA\bA^\top + \se)^{-1}$,
\[I_p -\sx^{-1} \se = \se^{-1}\bA \bG^{-1}\bA^\top,\]
where $\bG\coloneqq I_K+ \bA^\top\se^{-1}\bA$. Using this in the previous display,
\begin{equation}\label{eqn:ab formula}
    \ab = \se^{-1} \bA \bG^{-1}\bA^\top\bA^{+\top}\bb = \se^{-1} \bA \bG^{-1}\bb,
\end{equation}
where we again use $\bA^+\bA = \bA^\top\bA^{+\top} = I_K$ in the second step.

\paragraph{Bounds on $\|\ab\|_{\sx}^2$:}
% From (\ref{excess risk weighted norm}) we find $R(0) - R(\ab) = \|\ab\|^2_{\sx}$, and thus,
% \begin{align*}
%     \|\ab\|^2_{\sx} &= R(0) - R(\ab)\\
%     &= R(0) - \sep^2 - (R(\ab) - \sep^2)\\
%     &= \|\beta\|^2_{\sz} - (R(\ab) - \sep^2),
% \end{align*}
% where in the final step we use $R(0) = \EE y^2$ and $y = Z^\top\beta +\eps$ with $Z$ and $\eps$ independent.
By (\ref{eqn:ab formula}), we find
\begin{align*}
    \|\ab\|^2_{\sx} &= \bb^\top \bG^{-1}\bA^\top \se^{-1}(\bA\bA^\top +\se) \se^{-1}\bA \bG^{-1}\bb\\
    &= \bb^\top \bG^{-1}(\bA^\top\se^{-1}\bA)^2 \bG^{-1}\bb + \bb^\top \bG^{-1}(\bA^\top\se^{-1}\bA) \bG^{-1}\bb\\
    &= \bb^\top \bG^{-1}(\bG-I_K)^2 \bG^{-1}\bb + \bb^\top \bG^{-1}(\bG-I_K) \bG^{-1}\bb.
\end{align*}
Expanding the above and simplifying, we find 
\begin{equation}\label{ab sx bgb}
    \|\ab\|^2_{\sx} = \bb^\top [I_K - \bG^{-1}]\bb = \|\beta\|^2_{\sz} - \bb^\top \bG^{-1}\bb.
\end{equation}
Recalling that $R(\ab) - \sep^2 = \bb^\top \bG^{-1}\bb$ from (\ref{rab - sep}) above, Lemma \ref{thm:bench compare} implies that
\[0\le \bb^\top \bG^{-1}\bb \le \|\beta\|^2_{\sz}/\xi.\]
Combining this with (\ref{ab sx bgb}) yields
\[(1 - \xi^{-1})\cdot \|\beta\|^2_{\sz} \le \|\ab\|^2_{\sx} \le \|\beta\|^2_{\sz}.\]
Thus, when $\xi > c  >1$, $\|\ab\|^2_{\sx}\asymp \|\beta\|_{\sz}^2$, as claimed.

\paragraph{Bounds on $\|\ab\|^2$:} Using (\ref{eqn:ab formula}), we find
\begin{equation}
    \|\ab\|^2 = \bb^\top \bG^{-1} \bA^\top \se^{-2} \bA \bG^{-1}\bb.
\end{equation}
Thus,
\begin{align}
    \|\ab\|^2 &\le \frac{1}{\lp(\se)}\bb^\top \bG^{-1} \bA^\top \se^{-1} \bA \bG^{-1}\bb\nonumber\\
    &= \frac{1}{\lp(\se)}\bb^\top \bG^{-1} (\bG - I_K)\bG^{-1}\bb\nonumber\\
    &=\frac{1}{\lp(\se)}(\bb^\top \bG^{-1}\bb-\bb^\top \bG^{-2}\bb)\nonumber\\
    &\le\frac{1}{\lp(\se)}\bb^\top \bG^{-1}\bb.\label{ab ub bgb}
\end{align}
We also have
\begin{align}
    \|\ab\|^2 &\ge \frac{1}{\|\se\|}\bb^\top \bG^{-1} \bA^\top \se^{-1} \bA \bG^{-1}\bb\nonumber\\
    &= \frac{1}{\|\se\|}\bb^\top \bG^{-1} (\bG - I_K)\bG^{-1}\bb\nonumber\\
    &=\frac{1}{\|\se\|}[\bb^\top \bG^{-1}\bb-\bb^\top \bG^{-2}\bb]\nonumber\\
    &\ge\frac{1}{\|\se\|}\bb^\top \bG^{-1}\bb \cdot [1 - 1/\lk(\bG)]\\
    &\ge \frac{1}{\|\se\|}\bb^\top \bG^{-1}\bb \cdot [1 - 1/\xi],\label{ab lb bgb}
\end{align}
where in the final step we used
\[\lk(\bG) = 1 + \lk(\bA^\top \se^{-1}\bA) \ge \lk(\bA^\top \bA)/\|\se\| = \xi.\]
Combining (\ref{ab ub bgb}) and (\ref{ab lb bgb}), 
\[\left(\frac{\xi-1}{\xi}\right) \frac{1}{\|\se\|} \bb^\top \bG^{-1}\bb \le \|\ab\|^2 \le \frac{1}{\lp(\se)}\bb^\top \bG^{-1}\bb.\]
Recalling that $R(\ab) - \sep^2 = \bb^\top \bG^{-1}\bb$ from (\ref{rab - sep}) above, Lemma \ref{thm:bench compare} implies
\begin{equation}\label{ab u l se}
    \left(\frac{\xi-1}{\xi + 1}\right) \frac{1}{\|\se\|} \beta^\top (A^\top \se^{-1}A)^{-1} \beta \le \|\ab\|^2 \le \frac{1}{\lp(\se)}\beta^\top (A^\top \se^{-1}A)^{-1} \beta.
\end{equation}
As shown at the end of this proof using the singular value decomposition of $A$, we have that
\[\lp(\se)\cdot \beta^\top(A^\top A)^{-1}\beta \le \beta^\top(A^\top \se^{-1}A)^{-1}\beta \le \|\se\|\cdot \beta^\top(A^\top A)^{-1}\beta.\]
Combining this with (\ref{ab u l se}) proves that
\begin{equation}\label{aea}
    \left(\frac{\xi - 1}{\xi+1} \right)\cdot \frac{1}{\ke} \cdot \beta^\top (A^\top A)^{-1} \beta \le \|\ab\|^2 \le \ke\cdot  \beta^\top (A^\top A)^{-1} \beta .
\end{equation}
Thus, when $\xi > c>1$ and $\ke < C$, $\|\ab\|^2 \asymp \beta^\top (A^\top A)^{-1} \beta$, as claimed.

\paragraph{Proof of (\ref{aea}).} Write the singular value decomposition $A = U_A S_A V_A^\top$, where $U_A$ is an $p\times K$ matrix with satisfying $U_A^\top U_A=I_K$, $V_A$ is a $K\times K$ orthogonal matrix, and $S_A$ is a $K\times K$ diagonal matrix with positive entries (since we assume $\rank(A)=K$).
% Also write $\se = U_E D_E U_E^\top$ where $U_E$ is a $p\times p$ orthogonal matrix and $D_E$ is a $p\times p$ diagonal matrix with non-negative real entries. 
Then,
\begin{equation}\label{a svd use}
     (A^\top \se^{-1}A)^{-1} = (V_AS_AU_A^\top \se^{-1}U_AS_AV_A^{\top})^{-1}
    =  V_AS_A^{-1}(U_A^\top \se^{-1}U_A)^{-1}S_A^{-1}V_A^\top.
\end{equation}
Thus,
% \begin{equation}
%     \beta^\top (A^\top \se^{-1}A)^{-1}\beta = \beta^\top V_AS_A^{-1}(U_A^\top \se^{-1}U_A)^{-1}S_A^{-1}V_A^\top\beta
%     \ge \beta^\top V_AS_A^{-2}V_A^\top\beta\cdot\frac{1}{\|U_A^\top \se^{-1}U_A\|},
% \end{equation}
\begin{align*}
    \beta^\top (A^\top \se^{-1}A)^{-1}\beta &= 
    \beta^\top V_AS_A^{-1}(U_A^\top \se^{-1}U_A)^{-1}S_A^{-1}V_A^\top\beta\\
    &\ge \beta^\top V_AS_A^{-2}V_A^\top\beta\cdot\frac{1}{\|U_A^\top \se^{-1}U_A\|},
\end{align*}
so using 
\[\|U_A^\top \se^{-1}U_A\|\le \|\se^{-1}\| = 1/\lp(\se),\]
we find
\begin{equation}\label{lb baeab 1}
    \beta^\top (A^\top \se^{-1}A)^{-1}\beta \ge  \lp(\se)\cdot \beta^\top V_AS_A^{-2}V_A^\top\beta.
\end{equation}
We next observe that since $U_A^\top U_A = I_K$
\begin{equation}\label{aa id}
    (A^\top A)^{-1} = (V_A S_A U_A^\top U_A S_A V_A^\top)^{-1} = V_A S_A^{-2} V_A^\top,
\end{equation}
and thus, by (\ref{lb baeab 1}),
\[\beta^\top (A^\top \se^{-1}A)^{-1} \beta \ge \lp(\se) \cdot \beta^\top (A^\top A)^{-1} \beta,\]
which proves the lower bound in (\ref{aea}). To prove the upper bound, we use that by (\ref{a svd use}), 
\begin{align*}
    \beta^\top (A^\top \se^{-1}A)^{-1}\beta &= 
    \beta^\top V_AS_A^{-1}(U_A^\top \se^{-1}U_A)^{-1}S_A^{-1}V_A^\top\beta\\
    &\le \beta^\top V_AS_A^{-2}V_A^\top\beta\cdot\frac{1}{\lk(U_A^\top \se^{-1}U_A)}.
\end{align*}
Thus, since 
\[\lk(U_A^\top \se^{-1}U_A)\ge \lk(U_A^\top U_A)\lp(\se^{-1}) = 1/\|\se\|,\]
we have
\[ \beta^\top (A^\top \se^{-1}A)^{-1}\beta\le \|\se\|\cdot \beta^\top V_AS_A^{-2}V_A^\top\beta =  \|\se\|\cdot  \beta^\top (A^\top A)^{-1} \beta,\]
where in the last step we use (\ref{aa id}). This establishes the upper bound of (\ref{aea}), completing the proof. 

\hfill $\blacksquare$

\subsubsection{Proof of Corollary \ref{thm:null vs opt}}

Under the conditions stated, by either Lemma \ref{thm:ab norm bound e=0} or Lemma \ref{thm:ab norm bound main}, $\|\ab\|^2\lesssim \beta^\top(A^\top A)^{-1}\beta$. Thus, using that $\sz$ is invertible, 
\begin{equation}\label{ab beta lk}
    \|\ab\|^2\lesssim \beta^\top(A^\top A)^{-1}\beta = \beta^\top\sz^{1/2}(\sz^{1/2}A^\top A\sz^{1/2})^{-1}\sz^{1/2}\beta\le \|\beta\|_{\sz}^2/\lk(\sza),
\end{equation}
so $\|\ab\|\to 0$ when  $\|\beta\|^2_{\sz} /\lk(\sza)\to 0$.

For the second claim, we have
\begin{align*}
    R(\0) - R(\ab) &= \|\ab\|^2_{\sx} & (\text{by } (\ref{id:null}))\\
    &\gtrsim \|\beta\|^2_{\sz}. & (\text{by  either Lemma } \ref{thm:ab norm bound e=0} \text{ or Lemma } \ref{thm:ab norm bound main})
\end{align*}
The claim follows by taking the  limit inferior as $p\to \i$ on both sides of the inequality and using condition (\ref{cond:snrbeta2}). \hfill $\blacksquare$
\subsection{Proofs for Section \ref{section:min_ell_2}}
\subsubsection{Proofs for Section \ref{sec:noiseless}}\label{proofs:noiseless}

In the proofs of Lemma \ref{thm:a hat norm e=0} and Theorem \ref{thm:E=0 risk bound}, we will use the event
\begin{equation}\label{event Z eps}
    \A \coloneqq \left\{\|\tZ^+\tilde \Eps\|^2 \lesssim \log(n)\tr(\tZ^{+\top}\tZ^+),\ c_1n\le \sk^2(\tZ)\le \|\tZ\|^2\le c_2 n\right\},
\end{equation}
which occurs with probability at least $1-c/n$, as shown in Lemma \ref{thm:event Z eps bnd} below, where $\Z=\tZ\sz^{1/2}$ and $\Eps = \sep \tilde \Eps$ by Assumption \ref{ass:subg fm}. 

\subsubsection*{Proof of Lemma \ref{thm:a hat norm e=0}}\label{proof:a hat norm e=0}
On the event $\A$ defined in (\ref{event Z eps}), and using $\lk(\sz)>0$ by Assumption \ref{ass:fullrank},
\begin{equation}\label{eqn:skz > 0}
    \sk^2(\Z) = \lk(\Z\Z^\top )= \lk(\tZ \sz \tZ^\top )\ge \lk(\sz)\cdot \sn^2(\tZ) \gtrsim \lk(\sz)\cdot n > 0,
\end{equation}
so $\rank(\Z)=K$ and thus $\Z^+\Z=I_K$ by Lemma \ref{thm:psuedo lemma} in Appendix \ref{sec:pseudo-inverse}. Similarly, since $A$ is of dimension $p\times K$ and $\rank(A)=K$ by Assumption \ref{ass:fullrank}, 
\[A^\top A^{+\top} = (A^+A)^\top = I_K.\]
% By (\ref{eqn:bc}) of Lemma \ref{thm:psuedo lemma} in Appendix \ref{sec:pseudo-inverse} below,
% \begin{equation}\label{eqn:x+ E=0}
%     \X^+ = (\Z A^\top)^+ = (\Z^+ \Z A^\top)^+ (\Z A^\top A^{+\top})^+ .
% \end{equation}
% Since $K < n$, $\Z^+\Z=I_K$ on the event where $\Z$ is of full rank $K$ by Lemma \ref{thm:psuedo lemma}. Similarly, since $A$ is of dimension $p\times K$ and $\rank(A)=K$ by Assumption \ref{ass:fullrank}, 
% \[A^\top A^{+\top} = (A^+A)^\top = I_K.\]
Using these two results together with (\ref{eqn:bc}) of Lemma \ref{thm:psuedo lemma},
% in (\ref{eqn:x+ E=0}), 
we find 
\begin{equation}\label{eqn:x+ E=0}
    \X^+ = (\Z A^\top)^+ = (\Z^+ \Z A^\top)^+ (\Z A^\top A^{+\top})^+ =  A^{+\top} \Z^+.
\end{equation}
Thus, on the event $\A$,
\begin{equation}\label{aazy}
    \a = \X^+\y = A^{+\top} \Z^+\y,
\end{equation}
so
% On the event $\A$ defined in (\ref{event Z eps}), and using $\lk(\sz)>0$ by Assumption \ref{ass:fullrank},
% \begin{equation}\label{eqn:skz > 0}
%     \sk^2(\Z) = \lk(\Z\Z^\top )= \lk(\tZ \sz \tZ^\top )\ge \lk(\sz)\cdot \sn^2(\tZ) \gtrsim \lk(\sz)\cdot n > 0,
% \end{equation}
% so $\rank(\Z)=K$ and thus $\Z^+\Z=I_K$ by Lemma \ref{thm:psuedo lemma} below. Thus, using (\ref{aazy}), we find
\begin{align*}
    \|\a\|^2 &= \|A^{+\top}\Z^+\y\|^2\\
    &= \|A^{+\top}\Z^+\Z\beta + A^{+\top}\Z^+\Eps\|^2 & (\text{by } \y = \Z \beta + \Eps)\\
    &\le 2\|A^{+\top}\beta\|^2 + 2\|A^{+\top}\Z^+\Eps\|^2 & (\text{since } \Z^+\Z=I_K \text{ on }\A)\\
    &=  2\|A^{+\top}\beta\|^2 + 2\|(A\sz^{1/2})^{+\top}\tZ^+\Eps\|^2,
\end{align*}
where in the last step we used that
% $\Z^+\Z=I_K$ since $\Z$ is full rank on the event $\A$, and that 
by Lemma \ref{thm:psuedo lemma},
\[A^{+\top}\Z = A^{+\top}(\tZ \sz^{1/2})^+ =A^{+\top} \sz^{-1/2}\tZ^+ = (A\sz^{1/2})^{+\top}\tZ^+.\]
Continuing, and using
\[A^+A^{+\top} = (A^\top A)^{-1}A^\top A^{+\top} = (A^\top A)^{-1},\]
we find
\begin{align*}
     \|\a\|^2 &\lesssim \beta^\top (A^\top A)^{-1}\beta + \|(A\sz^{1/2})^+\|^2 \cdot \sep^2\cdot \|\tZ^+\tEps\|^2\\
     &\lesssim \beta^\top (A^\top A)^{-1}\beta + \frac{1}{\lk(\sza)}\sep^2 \log(n)\tr(\tZ^{+\top}\tZ^+) & (\text{on } \A)\\
     &\le \beta^\top (A^\top A)^{-1}\beta + \frac{1}{\lk(\sza)}\sep^2 \log(n)K \|\tZ^+\|^2\\
     &= \beta^\top (A^\top A)^{-1}\beta + \frac{1}{\lk(\sza)}\sep^2 \log(n)K\frac{1}{\sk^2(\tZ)}\\
     &\lesssim \beta^\top (A^\top A)^{-1}\beta+\frac{1}{\lk(\sza)} \sep^2 \log(n)\frac{K}{n}  & (\text{on } \A)\\
     &\le \frac{1}{\lk(\sza)} \left(\|\beta\|^2_{\sz}+\sep^2 \log(n)\frac{K}{n}\right). & (\text{by } (\ref{ab beta lk}))
\end{align*}
Under the assumptions of this Lemma, the event $\A$ holds with probability at least $1-c/n$ by Lemma \ref{thm:event Z eps bnd}, so the proof is complete. \hfill $\blacksquare$
\subsubsection*{Proof of Theorem \ref{thm:E=0 risk bound}}\label{proof:E=0 risk bound}

\paragraph{Part 1:} By (\ref{aazy}), $\a = A^{+\top}\Z^+\y$ on the event $\A$ defined in (\ref{event Z eps}). Thus, using $X = AZ$ and $A^\top A^{+\top} = I_K$ since $A$ is full rank by Assumption \ref{ass:fullrank},
\begin{equation}\label{eqn:yx=yz}
    \wh y_x = X^\top\a = Z^\top A^\top A^{+\top}\Z^+\y= Z^\top \Z^+\y = Z^\top\wh\beta = \wh y_z.
\end{equation}
\paragraph{Part 2:}
% \textbf{\textit{Part 1:}}
% By (\ref{eqn:bc}) of Lemma \ref{thm:psuedo lemma} in Appendix \ref{sec:pseudo-inverse} below,
% \begin{equation}\label{eqn:x+ E=0}
%     \X^+ = (\Z A^\top)^+ = (\Z^+ \Z A^\top)^+ (\Z A^\top A^{+\top})^+ .
% \end{equation}
% Since $K < n$, $\Z^+\Z=I_K$ on the event where $\Z$ is of full rank $K$ by Lemma \ref{thm:psuedo lemma}. Similarly, since $A$ is of dimensions $p\times K$ and $\rank(A)=K$ by Assumption \ref{ass:fullrank}, 
% \[A^\top A^{+\top} = (A^+A)^\top = A^+A = I_K,\]
% where we also use that $A^+A$ is symmetric by (\ref{eqn: b+b sym}) of Appendix \ref{sec:pseudo-inverse}.
% Using these two results in (\ref{eqn:x+ E=0}), we find
% \[A^\top\X^+ = A^\top[A^{+\top} \Z^+] = (A^\top A^{+\top})\Z^+ = \Z^+.\]
% Thus, recalling $\a = \X^+\y$, we find that on the event where $\Z$ is full rank,
% \begin{equation}\label{eqn:yx=yz}
%     \wh y_x = X^\top\a = Z^\top A^\top \X^+\y = Z^\top \Z^+\y = Z^\top\wh\beta = \wh y_z.
% \end{equation}
% We will work on the event
% \[\A \coloneqq \left\{\|\tZ^+\tilde \Eps\|^2 \lesssim \log(n)\tr(\tZ^{+\top}\tZ^+),\ c_1n\le \sk^2(\tZ)\le \|\tZ\|^2\le c_2 n\right\},\]
% which occurs with probability at least $1-c/n$, as shown below, where $\Z=\tZ\sz^{1/2}$ and $\Eps = \sep \tilde \Eps$ by Assumption \ref{ass:subg fm}. 
Using the independence of $\eps$ and $Z$ together with (\ref{eqn:yx=yz}), the excess risk can be written as
\begin{equation}\label{eqn:risk decomp e=0}
    R(\a) - \sep^2 = \EE[( X^\top\a- Z^\top\beta)^2] = \EE[(Z^\top \wh\beta - Z^\top \beta)^2] = \|\wh\beta - \beta\|^2_{\sz}.
\end{equation}
% On the event $\A$ defined in (\ref{event Z eps}), and using $\lk(\sz)>0$ by Assumption \ref{ass:fullrank},
% \begin{equation}\label{eqn:skz > 0}
%     \sk^2(\Z) = \lk(\Z\Z^\top )= \lk(\tZ \sz \tZ^\top )\ge \lk(\sz)\cdot \sn^2(\tZ) \gtrsim \lk(\sz)\cdot n > 0,
% \end{equation}
% so $\rank(\Z)=K$ and thus $\Z^+\Z=I_K$ by Lemma \ref{thm:psuedo lemma} below.
By (\ref{eqn:skz > 0}), $\rank(\Z)=K$ and $\Z^+\Z=I_K$ on the event $\A$ defined in (\ref{event Z eps}). Thus,
\[\wh\beta = \Z^+\y = \Z^+\Z\beta + \Z^+\Eps = \beta+ \Z^+\Eps,\]
so by (\ref{eqn:risk decomp e=0}), 
\begin{equation}\label{eqn:risk decomp e=0 2}
    R(\a) - \sep^2 = \|\Z^+\Eps\|^2_{\sz} = \|\sz^{1/2}\Z^+\Eps\|^2.
\end{equation}
By (\ref{eqn:bc}) of Lemma \ref{thm:psuedo lemma},
\begin{equation}\label{eqn:z+ tz}
    \sz^{1/2}\Z^+ = \sz^{1/2}(\tZ\sz^{1/2})^+ = \sz^{1/2}(\tZ^+\tZ \sz^{1/2})^+ (\tZ \sz^{1/2}\sz^{-1/2})^+ = \sz^{1/2}\sz^{-1/2}\tZ^+= \tZ^+,
\end{equation}
\[\]
where we used that $\tZ^+\tZ = I_K$ since $\rank(\tZ)=K$ on $\A$. Thus by (\ref{eqn:risk decomp e=0 2}), we find that on $\A$,
\begin{equation}\label{eqn:risk decomp e=0 final}
    R(\a) - \sep^2 = \|\tZ^+ \Eps\|^2 = \sep^2\|\tZ^+ \tilde \Eps\|^2 \lesssim \sep^2 \log(n)\tr(\tZ^{+\top}\tZ^+).
\end{equation}
We then use that $\rank(\tZ^+)=K$ and that $\|\tZ^+\|=1/\sk(\tZ)$ from Lemma \ref{thm:psuedo lemma} in Appendix \ref{sec:pseudo-inverse} below to find that on $\A$,
\[\tr(\tZ^{+\top }\tZ^+)\le K \|\tZ^{+\top}\tZ^+\| =K \|\tZ^+\|^2= \frac{K}{\sk^2(\tZ)}\lesssim \frac{K}{n}.\]
 Plugging this into (\ref{eqn:risk decomp e=0 final}) completes the proof of the upper bound.

For the lower bound, first observe that on $\A$,
\[\EE_{\Eps}R(\a) - \sep^2 = \EE_{\Eps}\|\tZ^+\Eps\|^2 = \sep^2\tr(\tZ^{+\top}\tZ^+) \ge \sep^2 K \lk(\tZ^{+\top }\tZ^+) = \sep^2K\sk^2(\tZ^+),\]
so using $\sk(\tZ^+) = 1/\|\tZ\|$ by Lemma \ref{thm:psuedo lemma} again,
\[\EE_{\Eps}R(\a) - \sep^2 \ge \sep^2 \frac{K}{\|\tZ\|^2}\gtrsim \sep^2\frac{K}{n}.\]
\hfill $\blacksquare$

\begin{lemma}\label{thm:event Z eps bnd}
Suppose that Assumptions \ref{ass:fullrank} \& \ref{ass:subg fm} hold and that
$n> C\cdot  K$ for some large enough absolute constant $C>0$. Then there exists $c>0$ such that
\[\PP\left\{\|\tZ^+\tilde \Eps\|^2 \lesssim \log(n)\tr(\tZ^{+\top}\tZ^+),\ c_1n\le \sk^2(\tZ)\le \|\tZ\|^2\le c_2 n\right\} \ge 1-c/n.\]
\end{lemma}
\begin{proof}
Since $\tZ$ has independent rows with entries that are zero mean, unit variance, and have sub-Gaussian constants bounded by an absolute constant, Theorem 4.6.1 of \cite{verHDP} gives that with probability at least $1-2/n$,
    \[\sqrt{n} - c''(\sqrt{K} + \sqrt{\log n})\le \sn(\tZ) \le \|\tZ\| \le \sqrt{n} + c''(\sqrt{K} + \sqrt{\log n}).\]
    and thus
    \[\sqrt{n}\cdot [1 - c''(\sqrt{K/n} + \sqrt{\log (n)/n})] \le \sn(\tZ)\le \|\tZ\| \le \sqrt{n}\cdot [1 - c''(\sqrt{K/n} + \sqrt{\log (n)/n})].\]
    Using that $n > CK$ we can choose $C$ large enough such that 
    \[c''(\sqrt{K/n} + \sqrt{\log (n)/n}) < c_0 < 1,\]
    and thus
    \begin{equation}\label{eqn:z double bound}
        \PP\left(c_3n \le \sigma_{K}^2(\tZ)\le \|\tZ\|^2 \le c_4n\right) \ge 1-2/n.
    \end{equation}
    The bound
    \[\PP\left(\|\tZ^+\tilde \Eps\|^2 \lesssim \log(n)\tr[\tZ^{+\top}\tZ^+] \right) \ge 1 - e^{-cn}\]
    follows from Lemma \ref{thm:trace lemma}, which we state below. Combining this with (\ref{eqn:z double bound}) proves that $\A$ occurs with probability at least $1-c/n$. 
\end{proof}

The following result is a slightly adapted version of Lemma 19 from \cite{bartlett2019} and the discussion that follows.
\begin{lemma}\label{thm:trace lemma}
Suppose $\tEps \in \R^n$ has independent entries with sub-Gaussian constants bounded by an absolute constant, and suppose $M\in \R^{n\times n}$ is a positive semidefinite matrix independent of $\tEps$. Then, with probability at least $1-e^{-cn}$,
\[\tEps^\top M \tEps \lesssim \log(n)\cdot \tr(M).\]
\end{lemma}

\subsubsection*{Proof of Lemma \ref{thm:pcr e=0}}\label{proof:pcr e=0}

% Since $\se = 0$, $\X = \Z A^\top$ almost surely. Thus,
Suppose $\rank(\X) = K$. We can then write the singular value decomposition of $\X$ as $\X = \wh V_K\wh D\wh U_K^\top$, where $\wh V_K\in \R^{n\times K}$, $\wh U_K\in \R^{p\times K}$, and $\wh D\in \R^{K\times K}$ are full rank, and $\wh V_K^\top \wh V_K =\wh U_K^\top \wh U_K = I_K$. Thus,
\[(\X \wh U_K)^+ = (\wh V_K\wh D\wh U_K^\top \wh U_K)^+ = (\wh V_K\wh D)^+.\]
By Lemma \ref{thm:psuedo lemma} of Appendix \ref{sec:pseudo-inverse}, we thus have
\begin{align*}
    (\X \wh U_K)^+ &= (\wh V_K^+\wh V_K \wh D)^+ (\wh V_K \wh D\wh D^+)^+\\
    &= \wh D^+ \wh V_K^+ &(\text{since } \wh V_K \text{ and } \wh D \text{ full rank})\\
    &= \wh D^+ (\wh V_K^\top \wh V_K)^+ \wh V_K^\top\\
    &= \wh D^+\wh V_K^\top. &(\text{by } \wh V_K^\top \wh V_K = I_K)
\end{align*}
We thus find
\[\wh \alpha_{\rm PCR} = \wh U_K (\X \wh U_K)^+\y = \wh U_K \wh D^+ \wh V_K^\top \y = \X^+\y = \a,\]
where we recognize $\wh U_K \wh D^+ \wh V_K^\top$ as the pseudoinverse of $\X$ in the third step.

Now suppose that Assumptions \ref{ass:fullrank} \& \ref{ass:subg fm} hold and $K > C\cdot n$. Then by Lemma \ref{thm:event Z eps bnd} above, $\PP\{\sk^2(\tZ)\gtrsim n\} \ge 1-c/n$. Thus, using 
\[\sk^2(\Z) = \sk^2(\tZ \sz^{1/2})\ge \lk(\sz)\sk^2(\tZ)\]
and that $\lk(\sz)>0$ by Assumption \ref{ass:fullrank}, 
\[\PP\{\rank(\X)=K\}\ge \PP\{\sk^2(\Z) \gtrsim n\}\ge \PP\{\sk^2(\tZ) \gtrsim n\} \ge 1-c/n,\]
which completes the proof.\hfill $\blacksquare$

\subsubsection{Proofs for Section \ref{main}}\label{proofs:main}
In this section we begin with the proof of Lemma \ref{thm:gls norm bound} and our main result, Theorem \ref{thm:upper bound}, which rely on Proposition \ref{in}, proved subsequently. The proofs of  Lemma \ref{thm:gls norm bound} and Theorem \ref{thm:upper bound} use the event
\begin{equation}\label{event main section}
    \Ec \coloneqq \Ec_1\cap \Ec_2\cap \Ec_3,
\end{equation}
where for positive absolute constants $c_1$ to $c_6$,
\[\Ec_1 \coloneqq \left \{\sigma_{n}^2(\X)\ge c_1\tr(\se), \|\E\|^2\le c_2\tr(\se),\ c_3n \le \sigma_{K}^2(\tZ)\le \|\tZ\|^2 \le c_4n \right\},\]
\[\Ec_2 \coloneqq \left\{ \tEps^\top \X^{+\top}\sx\X^+ \tEps\le c_5 \log(n)\tr(\X^{+\top}\sx\X^+) \right\},\]
\[\Ec_3 \coloneqq \left\{ \tEps^\top \X^{+\top}\X^+ \tEps\le c_6 \log(n)\tr(\X^{+\top}\X^+) \right\}.\]
We will show in Lemma \ref{thm:event main} below that $\Ec$ occurs with probability at least $1-c/n$ for an absolute constant $c>0$.

\subsubsection*{Proof of Theorem \ref{thm:gls norm bound}}\label{proof:gls norm bound}
Using $\a = \X^+\y$, $\y = \Z\beta +\Eps$, and that $A$ is full rank by Assumption \ref{ass:fullrank}, we find
\begin{align*}
    \a &= \X^+\y\\
    &= \X^+ \Z\beta + \X^+\Eps\\
    &= \X^+\Z A^\top A^{+\top}\beta + \X^+\Eps & (A^+A=I_K \text{ since } \rank(A)=K)\\
    &= \X^+(\X -\E)A^{+\top}\beta + \X^+\Eps & (\text{using } \X = \Z A^\top +\E)\\
    &= \X^+\X A^{+\top}\beta - \X^+\E A^{+\top}\beta + \X^+\Eps.
\end{align*}
Thus, using $(a+b+c)^2\le 3(a^2+b^2+c^2)$,
\begin{align*}
    \|\a\|^2 &\le 3\|\X^+\X A^{+\top}\beta\|^2+ 3\|\X^+\E A^{+\top}\beta\|^2 + 3\|\X^+\Eps\|^2\\
    &\lesssim \|\X^+\X\|^2 \|A^{+\top }\beta\|^2 + \frac{\|\E\|^2}{\sn^2(\X)}\|A^{+\top}\beta\|^2 + \sep^2 \tEps^\top \X^{+\top}\X^+ \tEps\\
    &\le  \|A^{+\top }\beta\|^2 + \|A^{+\top}\beta\|^2 + \sep^2 \log(n)\tr(\X^{+\top}\X^+),
\end{align*}
where in the last step holds on the event $\Ec$, and uses that $\|\X^+\X\|\le 1$ since $\X^+\X$ is a projection matrix. Recalling that by (\ref{ab beta lk}),
\[\|A^{+\top}\beta\|^2 = \beta^\top(A^\top A)^{-1}\beta\le \|\beta\|_{\sz}^2/\lk(\sza),\]
and using that $\rank(\X) \le n$, we find that on $\Ec$,
\begin{align*}
    \|\a\|^2 &\lesssim \frac{1}{\lk(\sza)}\|\beta\|^2_{\sz} + \sep^2\log(n) \cdot n\cdot  \|\X^+\|^2\\
    &= \frac{1}{\lk(\sza)}\|\beta\|^2_{\sz} + \sep^2 \frac{n\log n}{\sn^2(\X)}\\
    &\lesssim \frac{1}{\lk(\sza)}\|\beta\|^2_{\sz} + \sep^2 \frac{n\log n}{\tr(\se)}.
\end{align*}
By Lemma \ref{thm:event main}, $\Ec$ holds with probability at least $1-c/n$, so the proof is complete.\hfill $\blacksquare$
\subsubsection*{Proof of Theorem \ref{thm:upper bound}}\label{proof:upper bound}
% \subsubsection*{Step 1: Decomposing the risk}\label{sec:decomp}

Using that $Z$, $E$ and $\eps$ are independent of one another and of $\a$, we have
\begin{align*}
    R(\a) &= \EE[(X^\top \a - y)^2] \\
    &= \EE[(Z^\top A^\top \a - Z^\top \beta - \eps + E^\top \a)^2]\\
    &= \sep^2 + \|\Sigma_E^{1/2}\a\|^2 + \|\sz^{1/2}(A^\top \a- \beta)\|^2.
\end{align*}
Since $\wh\alpha = \X^+\y = \X^+\Z\beta + \X^+\Eps$,
\begin{align*}
    \|\Sigma_E^{1/2}\a\|^2 \le 2 \|\Sigma_E^{1/2}\X^+\Z\beta\|^2 + 2\|\Sigma_E^{1/2}\X^+\Eps\|^2
    \coloneqq 2B_{1} + 2V_{1}.
\end{align*}
Similarly,
\begin{align*}
    \|\sz^{1/2}(A^\top \a - \beta)\|^2 \le2 \|\sz^{1/2}(A^\top \X^+\Z - I_K)\beta\|^2 + 2\|\sz^{1/2}A^\top \X^+\Eps\|^2\coloneqq 2B_2 + 2V_2.
\end{align*}
We thus have $R(\a) - \sep^2 \lesssim B + V$, where we view $B \coloneqq B_1 + B_2$ as a bound on the bias component of the risk and $V\coloneqq V_1 + V_2$ as a bound on the variance component. In what follows, we bound the four terms
\begin{align*}
    &B_1 = \|\Sigma_E^{1/2}\X^+\Z\beta\|^2 \\
    &B_2=\|\sz^{1/2}(A^\top \X^+\Z - I_K)\beta\|^2  \\
    & V_1 =\|\Sigma_E^{1/2}\X^+\Eps\|^2 \\
    &V_2=\|\sz^{1/2}A^\top \X^+\Eps\|^2.
\end{align*}

% \subsubsection*{Step 2: Bounding the risk}
% Recall that $\Z = \tZ\sz^{1/2}$ and $\Eps = \sep\Eps_w$ from Assumption \ref{ass:subg fm}.
% We will bound the bias and variance on the event $\Ec \coloneqq \Ec_1\cap \Ec_2$, where
% \[\Ec_1 \coloneqq \left \{\sigma_{n}^2(\X)\ge c_1\tr(\se), \|\E\|^2\le c_2\tr(\se),\ c_3n \le \sigma_{K}^2(\tZ)\le \|\tZ\|^2 \le c_4n \right\},\]
% \[\Ec_2 \coloneqq \left\{ \tEps^\top \X^{+\top}\sx\X^+ \tEps\lesssim \log(n)\tr(\X^{+\top}\sx\X^+) \right\},\]
% where $c_1$ to $c_4$ are absolute constants such that, as shown in the last step of the proof, $\Ec$ occurs with probability at least $1 - c/n$ for some $c>0$. 
% \subsubsection*{\textit{Bounding the bias component}}
\paragraph{Bounding the bias component:}
On the event $\Ec$ defined in (\ref{event main section}), $\sn(\X)>0$ and by Assumption \ref{ass:fullrank} and (\ref{eqn:skz > 0}) above, $\sn^2(\Z)\gtrsim \lk(\sz)n > 0$. Thus $\X$ and $\Z$ are of rank $n$ and $K$ respectively, so by  Lemma \ref{thm:psuedo lemma} of Appendix \ref{sec:pseudo-inverse}, $\X\X^+ = I_n$ and $\Z^+\Z = I_K$. It follows that
\begin{align}
    % \Z^+ - A^\top \X^+\Z &=  \Z^+\Z - A^\top \X^+\Z & (\text{since } \Z^+\Z = I_K)\nonumber\\
    \Z^+ - A^\top \X^+\nonumber
    &= \Z^+\X\X^+ - A^\top \X^+& (\text{since } \X\X^+=I_n)\nonumber\\
    &= (\Z^+\X - A^\top)\X^+\nonumber\\
    &= (\Z^+[\Z A^\top + \E] - A^\top) \X^+& (\text{since } \X=\Z A^\top + \E)\nonumber\\
    &= \Z^+\E\X^+, &(\text{since } \Z^+\Z=I_K) \label{eqn:i-axz}
\end{align}
% \[\Z^+\E\X^+ = \Z^+(\X - \Z A^\top)\X^+ = (\Z^+\X - A^\top)\X^+ = \Z^+ - A^\top \X^+.\]
and thus again using $\Z^+\Z = I_K$
\[B_2 = \|\sz^{1/2}(A^\top\X^+\Z - I_K)\beta\|^2 =\|\sz^{1/2}(A^\top \X^+ - \Z^+)\Z\beta\|^2 = \|\sz^{1/2}\Z^+\E\X^+\Z\beta\|^2.\]
By (\ref{eqn:z+ tz}) above and the fact that $\Z$ is full rank on $\Ec$, $\sz^{1/2}\Z^+ = \tZ^+$, so on $\Ec$,
\[B_2 = \|\tZ^+\E\X^+\Z\beta\|^2 \le \frac{\|\E\|^2}{\sk^2(\tZ)} \|\X^+\Z\beta\|^2\lesssim \frac{\tr(\se)\|\X^+\Z\beta\|^2}{n},\]
where we also used that $\|\tZ^+\|^2 = 1/\sk^2(\tZ)$. Since $B_1 = \|\se^{1/2}\X^+\Z\beta\|^2\le \|\se\|\|\X^+\Z\beta\|^2$, and
\[\|\se\| = \tr(\se)\frac{\|\se\|}{\tr(\se)}= \frac{\tr(\se)}{n}\cdot \frac{n}{{\rm r_e} (\se)} \lesssim \frac{\tr(\se)}{n},\]
where we used the assumption ${\rm r_e} (\se) > c_1 n$ in the last step, we also have that on $\Ec$,
\begin{equation}\label{eqn:r3 apply lemma}
    B = B_1 + B_2\lesssim \frac{\tr(\se)\|\X^+\Z\beta\|^2}{n}.
\end{equation}
To bound $\|\X^+\Z\beta\|^2$, we first use $A^\top A^{+\top} = I_K$ and $\Z A^\top = \X - \E$ to find
\[
    \|\X^+\Z\beta\|^2 = \|\X^+\Z A^\top A^{+\top}\beta\|^2 \le 2\|\X^+\X A^{+\top}\beta\|^2 + 2\|\X^+\E A^{+\top}\beta\|^2.
\]
The second term can be bounded, on the event $\Ec$, by
\[\frac{\|\E\|^2\|A^{+\top}\beta\|^2}{\sn^2(\X)} {\lesssim}\|A^{+\top}\beta\|^2.\]
On the other hand, the first term can be bounded as $\|\X^+\X A^{+\top}\beta\|^2\le \|A^{+\top}\beta\|^2$ using the fact that $\X^+\X$ is a projection matrix, so we find that on $\Ec$,
\begin{equation}\label{eqn:r3 xzb}
    \|\X^+\Z\beta\|^2\lesssim \|A^{+\top}\beta\|^2.
\end{equation}
Finally, we have
\begin{equation}\label{eqn:A+b ub}
   \|A^{+\top}\beta\|^2 = \beta^\top(A^\top A)^{-1}\beta =\beta^\top\sz^{1/2}(\sz^{1/2}A^\top A\sz^{1/2})^{-1}\sz^{1/2}\beta  \le \frac{\|\beta\|^2_{\sz}}{\lk(\sza)}. 
\end{equation}
Combining this with (\ref{eqn:r3 xzb}) and plugging into (\ref{eqn:r3 apply lemma}), we find that on the event $\Ec$,
\begin{equation}\label{eqn:b}
    B \lesssim \frac{\|\beta\|^2_{\sz}}{\lk(\sza)}\frac{\tr(\se)}{n}= \frac{\|\beta\|^2_{\sz}\|\se\|}{\lk(\sza)}\cdot\frac{\tr(\se)}{\|\se\|n} = \frac{\|\beta\|^2_{\sz}}{\xi
    }\frac{{\rm r_e} (\se)}{n}.
\end{equation}
% \subsubsection*{\textit{Bounding the variance component}}
\paragraph{Bounding the variance component:}
First note that
\[V = V_1 + V_2 = 
\| \se^{1/2} \X^+ \Eps\|^2 + \| \sz^{1/2} A^\top \X^+ \Eps \|^2 =
\Eps^\top \X^{+\top}\sx\X^+\Eps = \sep^2 \tEps\X^{+\top}\sx\X^+\tEps,\]
so on the event $\Ec$, 
\begin{equation}\label{eqn:v tr}
    V\lesssim \sep^2\log(n)\tr(\X^{+\top}\sx\X^+)=\sep^2\log(n)\left\{\tr(\X^{+\top}\se\X^+)+\tr(\X^{+\top}\sza\X^+)\right \},
\end{equation}
where we use $\sx = \sza + \se$ in the second step. The first term in (\ref{eqn:v tr}) can by bounded as
\begin{equation}\label{eqn:tr x se x bound}
    \tr(\X^{+\top}\se \X^{+}) \le \|\se\|\cdot n\|\X^{+\top}\X^+\| = \|\se\|\frac{n}{\sn^2(\X)} \lesssim \frac{n}{{\rm r_e} (\se)} ,
\end{equation}
where in the first step we used that $\rank(\X^+)=\rank(\X) = n$ and in the last step that $\sn^2(\X)\gtrsim \tr(\se)$ on $\Ec$.

For the second term in (\ref{eqn:v tr}),
\begin{align*}
    \tr(\X^{+\top}\sza \X^{+}) &\le K\|\sz^{1/2}A^\top\X^+\|^2 & (\text{since } \text{rank}(\sza)=K)\\
    &=  K \|\sz^{1/2}(\Z^+ - \Z^+\E\X^+)\|^2 & (\text{by } (\ref{eqn:i-axz}) \text{ above})\\
    &\le 2K\|\tZ^+\|^2 + 2K \|\tZ^+\|^2\|\E\|^2\|\X^+\|^2,
\end{align*}
where we use that $\sz^{1/2}\Z^+ = \tZ^+$ from (\ref{eqn:z+ tz}) in the final step. Continuing, we find
\begin{equation}\label{eqn:tr sza}
        \tr(\X^{+\top}\sza \X^{+}) \lesssim  \frac{K}{\sk^2(\tZ)}\left(1 + \frac{\|\E\|^2}{\sn^2(\X)}\right)\lesssim  \frac{K}{n},
\end{equation}
where we use the bounds defining $\Ec_1$ in the last inequality.
Combining (\ref{eqn:tr sza}) and (\ref{eqn:tr x se x bound}) with (\ref{eqn:v tr}), we conclude that on $\Ec$,
\[V \lesssim \sep^2\frac{n\log n}{{\rm r_e} (\se)} + \sep^2\frac{K\log n}{n}.\]
Combining this with the bias bound (\ref{eqn:b}) gives the bound in the statement of the theorem. By Lemma \ref{thm:event main} below, $\PP(\Ec)\ge 1-c/n$, so the proof is complete.\hfill $\blacksquare$\\

% \subsubsection*{Step 3: Bounding $\PP(\Ec)$}
\begin{lemma}\label{thm:event main}
Under model (\ref{model}),   suppose that Assumptions \ref{ass:fullrank} and \ref{ass:subg fm} hold and
 $n > 
 C\cdot K$ and ${\rm r_e} (\se)>C \cdot n$ hold, for some   $C>0$. Then $\PP(\Ec)\ge 1-c/n$, where $\Ec \coloneqq \Ec_1\cap \Ec_2\cap \Ec_3$ and
\[\Ec_1 \coloneqq \left \{\sigma_{n}^2(\X)\ge c_1\tr(\se), \|\E\|^2\le c_2\tr(\se),\ c_3n \le \sigma_{K}^2(\tZ)\le \|\tZ\|^2 \le c_4n \right\},\]
\[\Ec_2 \coloneqq \left\{ \tEps^\top \X^{+\top}\sx\X^+ \tEps\le c_5 \log(n)\tr(\X^{+\top}\sx\X^+) \right\},\]
\[\Ec_3 \coloneqq \left\{ \tEps^\top \X^{+\top}\X^+ \tEps\le c_6 \log(n)\tr(\X^{+\top}\X^+) \right\},\]
for positive constants $c_1$ to $c_6$.
\end{lemma}
\begin{proof}
We have $\PP(\Ec^c)\le \PP(\Ec_1^c)+\PP(\Ec_2^c) + \PP(\Ec_3^c)$. The bounds $\PP(\Ec_2^c)\le e^{-cn}$ and $\PP(\Ec_3^c)\le e^{-cn}$ follow immediately from Lemma \ref{thm:trace lemma} in Appendix \ref{proofs:noiseless} above, using the fact that $\tEps$ has independent entries with sub-Gaussian constants bounded by an absolute constant. Considering $\PP(\Ec_1^c)$, we have
\begin{align*}
    \PP(\Ec_1^c) \le \PP\{\sigma_{n}^2(\X)\le c_1 \tr(\se)\} +\PP\{ \|\E\|^2\ge c_2\tr(\se)\}+ \PP\{c_3n \le \sigma_{K}^2(\tZ)\le \|\tZ\|^2 \le c_4n\}
\end{align*}
The three terms above can be bounded as follows. Recall that we assume $n > CK$ and ${\rm r_e} (\se) > C n$ for some $C >1$ large enough.
\begin{enumerate}
    \item Since ${\rm r_e} (\se) > Cn$, Proposition \ref{in} can be applied to conclude 
    \[\PP\{\sigma_{n}^2(\X)\le c_1 \tr(\se)\}\le 2e^{-cn}.\]
    
    \item By Assumption \ref{ass:subg fm}, $\E = \tE \se^{1/2}$, where $\tE$ has independent entries with zero mean, unit variance, and sub-Gaussian constants bounded by an absolute constant. Thus,
    \[\|\E\|^2 = \|\E\E^\top \| = \|\tE\se \tE^\top \|,\]
    and by applying Theorem \ref{thm:concentration spectrum trace} with $\tE$ and $\se$ we find that with probability at least $1-2e^{-cn}$,
    \[\|\E\|^2 \le \tr(\se) + c'\|\se\|n= \tr(\se)\cdot (1 + c'n/{\rm r_e} (\se))\lesssim \tr(\se),\]
    where the last inequality holds since $n/{\rm r_e} (\se) < 1/C$. Thus for $c_2>0$,
    \[\PP\{ \|\E\|^2\ge c_2\tr(\se)\} \le 2e^{-cn}.\]
    
    \item By (\ref{eqn:z double bound}) we have that with probability at least $1-2/n$,
    \[c_3n \le \sigma_{K}^2(\tZ)\le \|\tZ\|^2 \le c_4n.\]
\end{enumerate}
Combining the previous three steps shows that $\PP(\Ec_1^c)\le c/n$.
\end{proof}

% {\color{red} We next bound $\PP(\Ec_2^c)$. For this we use that $\tilde \Eps$ are $\X$ are independent and that $\tilde \Eps$ is sub-Gaussian with sub-Gaussian constant bounded by an absolute constant. This allows us to apply Lemma 9 of \cite{bartlett2019} (and the discussion immediately following it), which implies that for all $t\ge 0$, with probability at least $1-e^{-t}$,
% \[\tEps' \X^{+}^\top \sx\X^+ \tEps\lesssim (4t+2)\tr(\X^{+}^\top\sx\X^+).\]
% Choosing $t = \log(2n)$, $\PP(\Ec_2^c)\le 1/2n$, which completes the proof.}
%%{\color{red}Do we need more detail here? NOTE: I think I saw an answer saying we need a lot more detail? It was mixed up a bit with other text below.}
%%{\color{blue}{ No, but I don't understand for $n$ large enough - your choice of $t$ should guarantee it for all $n$ :)}} {\color{red}Fixed now.}

\subsubsection*{Proof of Proposition \ref{in}}\label{proof:in}
We will work on the event 
\[\F \coloneqq \{\sn^2(\E U_{(K+1):p})\ge c_4\tr(\se),\ \|\tZ\|^2\le c_5 n\},\]
where $U_{(K+1):p}\in \R^{p\times (p-K)}$ has columns equal to the orthonormal eigenvectors of $\sx$ corresponding to the smallest $p-K$ eigenvalues. 

% \subsubsection*{Bounding $\PP(\F)$}
\paragraph{Bounding $\PP(\F)$:}By Assumption \ref{ass:subg fm}, $\E = \tE \se^{1/2}$, where $\tE$ has independent sub-Gaussian entries with zero mean, unit variance, sub-Gaussian constants bounded by an absolute constant. Thus, letting 
\[Q = U_{(K+1):p}U_{(K+1):p}',\]
we have
\[\sn^2(\E U_{(K+1):p}) = \ln(\E Q\E^\top) = \ln(\tE\se^{1/2}Q\se^{1/2} \tE^\top).\]
We can now apply Theorem \ref{thm:concentration spectrum trace}, stated and proved above in Section \ref{proofs:null}, with $\tE$ and $\se^{1/2}Q\se^{1/2}$. Noting that $M = \max_{ij} \|\tE\|_{\psi_2}$ is bounded by an absolute constant by Assumption \ref{ass:subg fm}, this implies that with probability at least $1-2e^{-cn}$,
\begin{equation}\label{eqn:sne lb with P}
    \sn^2(\E U_{(K+1):p}) \ge \tr(\se^{1/2}Q\se^{1/2})/2 - c'\|\se^{1/2}Q\se^{1/2}\|n.
\end{equation}
Since $Q$ is a projection matrix, $\|\se^{1/2}Q\se^{1/2}\|\le \|\se\|\|Q\|= \|\se\|$. Furthermore,
\begin{align*}
    \tr(\se^{1/2}Q\se^{1/2}) &= \tr(\se Q)\\
    &= \tr(\se) - \tr(\se (I-Q))\\
    &\ge \tr(\se) - K\|\se(I-Q)\| &(\text{since } \rank(I-Q)=K)\\
    &\ge \tr(\se) - K\|\se\|\|I-Q\|\\
    &= \tr(\se) - K\|\se\| & (\text{since } \|I-Q\|=1)\\
    &\ge \tr(\se) - n\|\se\|. &(\text{since } n\ge K)
\end{align*}
Plugging these two results into (\ref{eqn:sne lb with P}), we find that with probability at least $1-2e^{-cn}$,
\begin{equation}\label{eqn:sneu lb final}
    \sn^2(\E U_{(K+1):p}) \ge \tr(\se)/2 - (1/2+c')n\|\se\| = \tr(\se) \cdot[1/2 - (1/2+c')n/{\rm r_e} (\se)]\gtrsim \tr(\se),
\end{equation}
where in the last inequality we use that $n/{\rm r_e}(\se) < 1/C$ and choose $C$ large enough.

Also, since $\tZ$ has independent rows with entries that have zero mean, unit variance, and sub-Gaussian constants bounded by an absolute constant, we have that by Theorem 4.6.1 of \cite{verHDP},
\[\|\tZ\|^2\le c_2 n,\]
with probability at least $1-e^{-c'n}$. Combining this with \ref{eqn:sneu lb final} we conclude that
\[\PP(\F) \ge 1 - c e^{-c'n}.\]
\paragraph{Bounding $\sn(\X)$ on $\F$:}
% \subsubsection*{Bounding $\sn(\X)$ on $\F$ }
We now show that $\sn^2(\X) \gtrsim \tr(\se)$ holds on the event $\F$. Let $\sx = UDU^\top $ with $U\in \R^{p\times p}$ orthogonal and $D = \diag(\lambda_1(\sx),\ldots,\lambda_p(\sx))$. Define $U_K\in \R^{p\times K}$ to be the sub-matrix of $U$ containing the first $K$ columns, and define $U_{(K+1):p}$ to be composed of the last $p-K$ columns of $U$. Then
\[I_p = UU^\top  = U_KU_K^\top  + U_{(K+1):p}U_{(K+1):p}^\top ,\]
so
\[\ln(\X\X^\top ) = \ln(\X U_KU_K^\top \X^\top + \X U_{(K+1):p}U_{(K+1):p}^\top \X^\top )\ge \ln(\X U_{(K+1):p}U_{(K+1):p}^\top \X^\top ), \]
where we use the min-max formula for eigenvalues in the last step. This implies 
\begin{equation}\label{eqn:x lb xu}
    \sn(\X) \ge \sn(\X U_{(K+1):p}).
\end{equation}
By Weyl's inequality for singular values, and using $\X = \Z A^\top  + \E$,
\[|\sn(\X  U_{(K+1):p}) - \sn(\E U_{(K+1):p})| \le \|\Z A^\top   U_{(K+1):p}\|,\]
so by (\ref{eqn:x lb xu}),
\begin{equation}\label{eqn:lb x weyl}
    \sn(\X)\ge \sn(\X U_{(K+1):p}) \ge \sn(\E U_{(K+1):p}) - \|\Z A^\top   U_{(K+1):p}\| \gtrsim \sqrt{\tr(\se)} - \|\Z A^\top   U_{(K+1):p}\|,
\end{equation}
where the last inequality holds on the event $\F$. We show below that $\|\Z A^\top   U_{(K+1):p}\|\lesssim \sqrt{n\|\se\|}$ on $\F$, which implies that
\[\sn(\X) \gtrsim \sqrt{\tr(\se)} - c\sqrt{n\|\se\|} = \sqrt{\tr(\se)}\cdot (1-  c\sqrt{n/{\rm r_e} (\se)})\gtrsim \sqrt{\tr(\se)},\]
%{\color{blue} make it $\ge 1/2 \cdots$}{\color{red} why?}
where in the last inequality we use that $n/{\rm r_e}(\se) < 1/C$ and choose $C$ large enough.

% \subsubsection*{Upper bound of $\|\Z A^\top   U_{(K+1):p}\|$}
\paragraph{Upper bound of $\|\Z A^\top   U_{(K+1):p}\|$:}
On the event $\F$,
\begin{equation}\label{eqn:zau ub 1}
    \|\Z A^\top   U_{(K+1):p}\|^2 = \|\tZ \sz^{1/2} A^\top   U_{(K+1):p}\|\le \|\tZ\|^2 \|\sz^{1/2} A^\top   U_{(K+1):p}\|^2 \lesssim n \|\sz^{1/2} A^\top   U_{(K+1):p}\|^2.
\end{equation}
Furthermore, using $\sx = \sza+\se$, and that $U_{(K+1):p}^\top \sx U_{(K+1):p} = D_{(K+1):p}$ where we define $D_{(K+1):p} \coloneqq \diag(\lambda_{K+1}(\sx),\ldots,\lambda_{p}(\sx))$,
\begin{align*}
    \|\sz^{1/2} A^\top   U_{(K+1):p}\|^2 &= \|U_{(K+1):p}^\top \sza U_{(K+1):p}\|\\
    &= \|U_{(K+1):p}^\top \sx U_{(K+1):p} - U_{(K+1):p}^\top \se U_{(K+1):p}\|\\
    &= \|D_{(K+1):p}- U_{(K+1):p}^\top \se U_{(K+1):p}\|\\
    &\le \lambda_{K+1}(\sx) + \|U_{(K+1):p}^\top \se U_{(K+1):p}\|\\
    &\le \lambda_{K+1}(\sx) + \|\se\|\|U_{(K+1):p}^\top  U_{(K+1):p}\|\\
    &= \lambda_{K+1}(\sx) + \|\se\|,
\end{align*}
where we use $U_{(K+1):p}^\top  U_{(K+1):p} = I_{p-K}$ in the last step. Thus, using that 
\[\lambda_{K+1}(\sx) = \lambda_{K+1}(\sx) - \lambda_{K+1}(\sza) \le \|\se\|\]
by Weyl's inequality and the fact that $\lambda_{K+1}(\sza)=0$, we find
\[\|\sz^{1/2} A^\top   U_{(K+1):p}\|^2\le 2\|\se\|.\]
Combining this with (\ref{eqn:zau ub 1}), we find that on $\F$,
\[\|\Z A^\top   U_{(K+1):p}\| \lesssim \sqrt{n\|\se\|}.\]
\hfill $\blacksquare$

%\subsubsection

\subsubsection{Proof of Theorem \ref{TPCR} from Section \ref{sec:pcr}}\label{proof:TPCR}

% \subsubsection*{Step 1: Formulation as a linear model}
% Define $\ab \coloneqq \sx^{-1}\sxy$ and $\eta \coloneqq y - X^\top \ab$, and note that 
% \[\EE[X\eta] = \EE[Xy] - \EE[XX']\ab= \sxy - \sx \sx^+\sxy = 0,\]
% where in the last step we use the fact that $\sx \sx^+\sxy = \sxy$ from Lemma \ref{thm:xxxy}. Since we assume $(X,y)$ is Gaussian, and $X$ and $\eta$ are uncorrelated, they are independent. We thus have the linear model
% \begin{equation}\label{eqn:pcr lin model}
%     y = X^\top\ab + \eta,
% \end{equation}
% with $X$ and $\eta$ independent, and noise variance 
% \begin{equation}\label{eqn:eta var}
%     \EE[\eta^2] = \EE[(X^\top\ab - y)^2] = R(\ab).
% \end{equation}

% \subsubsection*{Step 2: Decomposition of risk}
% We aim to bound the risk
% \[\rpcr(\bpcr) \coloneqq \EE\left[ ( X^\top U_K \bpcr - y)^2 \right],\]
% with $\bpcr \coloneqq (\X U_K)^+\y$ and where $U_K$ is the $p\times K$ matrix with columns equal to the first $K$ (orthonormal) eigenvectors of $\sx$. Using the linear model formulation \ref{eqn:pcr lin model} and the independence of $X$ and $\eta$, we find
Let $D_K = U_K^\top \sx U_K = \diag(\lambda_1(\sx),\ldots,\lk(\sx))$ and note that since $A$ and $\sz$ are rank $K$ by Assumption \ref{ass:fullrank},
\[\lk(\sx) \ge \lk(\sza) \ge \lk(\sz)\lk(AA^\top)>0,\]
and thus $D_K$ is invertible. Furthermore, define $\eta = y - X^\top\ab$ with variance $\seta^2 = \EE[\eta^2]$, and the sample version $\Eta = \y - \X\ab$. We work on the event $\D \coloneqq \D_1\cap \D_2$, where
\[\D_1 \coloneqq \left\{\sk^2(\X U_K D_K^{-1/2}) \gtrsim n,\ \|\X\sx^{-1/2}\|^2 \lesssim p\right\},\]
and
\[\D_2 \coloneqq \left\{\|(\X U_K D_K^{-1/2})^+\Eta\|^2\lesssim \log(n)\cdot \seta^2\cdot \tr[(\X U_K D_K^{-1/2})^{+\top}(\X U_K D_K^{-1/2})^+]\right\}.\]
As the last step of this proof, we will show that $\PP(\D) \ge 1 - c'/n$. 

Letting $\eta \coloneqq y - X^\top \ab$, we have
\begin{equation}\label{eqn:x eta}
    \EE[X\eta] = \EE[Xy] - \EE[XX^\top ]\ab= \sxy - \sx \sx^+\sxy = 0,
\end{equation}
where we used (\ref{eqn:sx sxy identity}) in the last step. Thus,
\begin{align}\label{eqn:pcr risk init decomp}
     R(\wt \alpha_{\rm PCR}) &\coloneqq \EE[(X^\top \wt \alpha_{\rm PCR}  - y)^2]\nonumber\\
    &= \EE\left[ ( X^\top \wt \alpha_{\rm PCR} - X^\top \ab - \eta)^2 \right] \nonumber\\
    &= \EE\left[ ( X^\top \wt \alpha_{\rm PCR}  - X^\top\ab)^2 \right]  + \EE[\eta^2] & (\text{by } \ref{eqn:x eta})\nonumber\\
    &= \|\wt \alpha_{\rm PCR} - \ab\|^2_{\sx} + R(\ab).
\end{align}
Defining the projection matrix $P = U_KU_K^\top$, and writing
\[\y = \X\ab + \Eta = \X P\ab + \X (I_p-P)\ab + \Eta, \]
we find
\begin{align*}
    \wt \alpha_{\rm PCR}  &= U_K(\X U_K)^+\y\\
    &= U_K(\X U_K)^+\X P \ab + U_K(\X U_K)^+\X (I_p-P)\ab + U_K(\X U_K)^+\Eta.
\end{align*}
From the fact that $\X U_K$ is an $n\times K$ matrix with $K< n$ and $\rank(\X U_K) = K$ on the event $\D_1$, we have $(\X U_K)^+\X U_K = I_K$ by Lemma \ref{thm:psuedo lemma} of Appendix \ref{sec:pseudo-inverse} below. Thus, using $P = U_K U_K^\top$ we have $(\X U_K)^+\X P = U_K^\top$. Applying this in the previous display, we find
\[\wt \alpha_{\rm PCR} = P \ab + U_K(\X U_K)^+\X (I_p-P)\ab + U_K(\X U_K)^+\Eta.\]
It thus follows from the decomposition (\ref{eqn:pcr risk init decomp}) that
\begin{align}
    R(\wt \alpha_{\rm PCR}) - R(\ab) &=  \|\wt \alpha_{\rm PCR} - \ab\|^2_{\sx} \nonumber\\
    &\lesssim \|(I_p-P)\ab\|^2_{\sx} + \|U_K(\X U_K)^+ \X(I_p-P)\ab\|^2_{\sx} + \|U_K (\X U_K)^+\Eta\|^2_{\sx}\nonumber\\
    &=: B_1 + B_2 + V.
\end{align}

\paragraph{Bounding $B_1$:} We find
\begin{equation}\label{eqn:B_1 init pcr}
    B_1 = \|\sx^{1/2}(I_p - P)\ab\|^2 \le \|\sx^{1/2}(I_p-P)\|^2\|\ab\|^2= \|(I-P)\sx (I-P)\|\|\ab\|^2.
\end{equation}
Since $I-P$ is a projection onto the span of the last $p-K$ eigenvectors of $\sx$ with eigenvalues $\lambda_{K+1}(\sx),\ldots,\lambda_p(\sx)$, we have $\|(I-P)\sx (I-P)\| = \lambda_{K+1}(\sx)$. By Weyl's inequality,
\[\lambda_{K+1}(\sx) = \lambda_{K+1}(\sx) - \lambda_{K+1}(\sza) \le \|\se\|,\]
where we used that $\lambda_{K+1}(\sza)=0$ in the first step since $\rank(\sza)= K$. Thus 
\[\|\sx^{1/2}(I_p-P)\|^2\le \|\se\|,\] 
and combining this with (\ref{eqn:B_1 init pcr}) we find
\begin{equation}\label{pcr b1 bnd}
    B_1 \le \|\se\|\|\ab\|^2.
\end{equation}
% Using that  $\sx = \sza +\se$, Lemma \ref{thm:bart bias} of Appendix \ref{sec:finite sample} can be applied to find
% \[\|\ab\|^2 \le \ke \|\beta\|^2_{\sz}/\lk(\sza),\] 
% so
% \begin{equation}\label{eqn:B_1 pcr bound}
%     B_1 \le \ke\frac{\|\se\|}{\lk(\sza)}\cdot\|\beta\|^2_{\sz} = \ke \frac{\|\beta\|^2_{\sz}}{\xi}.
% \end{equation}

\paragraph{Bounding $B_2$:} Recalling $D_K = U_K^\top \sx U_K$,
\begin{align}
B_2 &= {\ab}^{\top} (I_p-P) \X^\top (\X U_K)^{+\top} U_K^\top\sx U_K (\X U_K)^+ \X(I-P)\ab\nonumber\\
&= \|D_K^{1/2}(\X U_K)^+ \X(I_p-P)\ab\|^2.\label{eqn:B2 dk}
\end{align}
Observe that by Lemma \ref{thm:psuedo lemma} of Appendix \ref{sec:pseudo-inverse}, 
\begin{equation}\label{eqn:xukdk}
    (\X U_K D_K^{-1/2})^+ = [ (\X U_K)^+(\X U_K) D_K^{-1/2}]^+ \cdot [\X U_K D_K^{-1/2}D_K^{1/2}]^+ = D_K^{1/2} (\X U_K)^+,
\end{equation}
where we used that $\X U_K$ is a full rank $n\times K$ matrix with $K< n$ so $(\X U_K)^+(\X U_K)=I_K$. Using this in (\ref{eqn:B2 dk}) yields
\begin{align*}
    B_2 &= \|(\X U_K D_K^{-1/2})^+ \X (I_p - P)\ab\|^2\\
    &\le \frac{\|\X (I_p - P)\ab\|^2}{\sk^2(\X U_K D_K^{-1/2})}\\
    &\le \frac{\|\X \sx^{-1/2}\|^2 }{\sk^2(\X U_K D_K^{-1/2})}\cdot \|\sx^{1/2}(I_p - P)\ab\|^2\\
    &\lesssim \frac{p}{n} \|\sx^{1/2}(I_p - P)\ab\|^2,
\end{align*}
where the last step holds on $\D$. Recalling that $\|\sx^{1/2}(I_p - P)\ab\|^2 = B_1$ and using (\ref{pcr b1 bnd}), we find that
\begin{equation}\label{eqn:b2 final pcr}
    B_2 \lesssim  \|\se\|\cdot \|\ab\|^2\frac{p}{n}.
\end{equation}

\paragraph{Bounding $V$:} We have on $\D$,
\begin{align*}
    V &= \Eta^\top (\X U_K)^{+\top} U_K^\top \sx U_K  (\X U_K)^{+} \Eta\\
    &= \Eta^\top (\X U_K)^{+\top} D_K (\X U_K)^{+} \Eta\\
    &= \|D^{1/2}_K(\X U_K)^{+} \Eta\|^2\\
    &= \|(\X U_K D_K^{-1/2})^{+} \Eta\|^2 & (\text{by } (\ref{eqn:xukdk}))\\
    &\lesssim \seta^2 \cdot \log(n)\cdot \tr[(\X U_K D^{-1/2})^{+\top}(\X U_K D^{-1/2})^+] &(\text{on } \D_2)\\
    &\le \seta^2 \cdot \log(n)\cdot K\cdot \|(\X U_K D^{-1/2})^+\|^2 &(\text{since } \rank(\X U_K D^{-1/2})=K)\\
    &= \seta^2\cdot \frac{K\log n}{\sk^2(\X U_K D^{-1/2})}\\
    &\lesssim \seta^2\cdot \frac{K\log n}{n}. & (\text{on } \D_1).
\end{align*}
Recalling $\eta = y - X^\top \ab$ so $\seta^2 = R(\ab)$,
\begin{equation}\label{pcr v bnd}
    V\lesssim R(\ab)\cdot \frac{K\log n}{n}.
\end{equation}
Combining this with (\ref{pcr b1 bnd}) and (\ref{eqn:b2 final pcr}) proves (\ref{PCR}). \\

In the case $\se=0$, the bound (\ref{PCR2}) follows immediately from (\ref{PCR}). When $\lp(\se)>0$, Lemma \ref{thm:bench compare} of Section \ref{sec:bench} implies
\[R(\ab)\le \sep^2 + \frac{\|\beta\|^2}{\xi}.\]
% we use from Lemma \ref{thm:bench compare} that
% \[\seta^2 \le \sep^2 + \frac{\|\beta\|^2_{\sz}}{\xi},\]
% so
% \begin{equation}\label{pcr v se}
%     V\lesssim \frac{\|\beta\|^2_{\sz}}{\xi}\frac{K\log n}{n} + \sep^2\frac{K\log n}{n}.
% \end{equation}
When $\lp(\se)>0$, we also have that \[\|\ab\|^2\le \ke \beta^\top(A^\top A)^{-1}\beta\le \frac{1}{\lp(\se)}\cdot \frac{\|\beta\|^2_{\sz}}{\xi}.\]
Plugging the last two displays into (\ref{PCR}) gives
\begin{align*}
    \rpcr(\bpcr) - R(\ab) &\lesssim  \ke\frac{\|\beta\|^2_{\sz}}{\xi}\cdot \frac{p}{n}+  \frac{\|\beta\|^2_{\sz}}{\xi}\frac{K\log n}{n} + \sep^2\frac{K\log n}{n}\\
    &\lesssim \ke \frac{\|\beta\|^2_{\sz}}{\xi}\cdot \frac{p}{n}+\sep^2\frac{K\log n}{n},
\end{align*}
where in the second step we use that 
\[K\log n < c\cdot n \lesssim p\le \ke p.\] 
This proves (\ref{PCR3}). All that remains is to bound the probability of the event $\D$.

\paragraph{Bounding $\PP(\D)$:}
% \[\D_1 \coloneqq \left\{\sk^2(\X U_K D_K^{-1/2}) \gtrsim n,\ \|\X\sx^{-1/2}\|^2 \lesssim p\right\},\]
We first bound the probability $\PP(\D_1)$. Note that the matrix $\X U_K D_K^{-1/2}$ has independent Gaussian rows $D_K^{-1/2} U_K^\top X_i$, with covariance
\[\EE[D_K^{-1/2} U_K^\top X_i X_i^\top  U_K D_K^{-1/2}] = D_K^{-1/2} U_K^\top \Sigma_X U_K D_K^{-1/2} = D_K^{-1/2} D_K D_K^{-1/2} = I_K,\]
and so $\X U_K D_K^{-1/2}$ i.i.d.~$N(0,1)$ entries. Thus, by Theorem 4.6.1 of \cite{verHDP}, with probability at least $1-2/n$,
\begin{equation}\label{eqn:sk lb pcr}
    \sk(\X U_K D_K^{-1/2}) \ge \sqrt{n} - c(\sqrt{K} + \sqrt{\log n})= \sqrt{n}\cdot [1 - c\sqrt{K/n} - c\sqrt{\log(n)/n}]\gtrsim \sqrt{n},
\end{equation}
where in the last step we use the assumption that $n > CK>C$ and choose $C$ large enough.

Similarly, $\X\sx^{-1/2}$ is a $n\times p$ matrix with i.i.d.~$N(0,1)$ entries, so again by by Theorem 4.6.1 of \cite{verHDP}, with probability at least $1-2e^{-n}$,
\begin{equation}
    \|\X \sx^{-1/2}\| \le \sqrt{n} + c(\sqrt{p}+\sqrt{n}) \lesssim \sqrt{p}.
\end{equation}
Using a union bound to combine this with (\ref{eqn:sk lb pcr}), we find
\[\PP(\D_1) \ge 1 - c'/n,\]
for some $c' > 0$.

To bound $\PP(\D_2)$, first note that by (\ref{eqn:x eta}) and the assumption that $(X,y)$ are Gaussian, $\X$ and $\Eta$ are independent. Furthermore, $\tilde \Eta = \Eta/\seta$ has independent $N(0,1)$ entries. We can thus apply Lemma \ref{thm:trace lemma} from Appendix \ref{proofs:noiseless} above with 
\[M = (\X U_K D_K^{-1/2})^{+\top}(\X U_K D_K^{-1/2})^+\]
to conclude that with probability at least $1-e^{-cn}$,
\[\|(\X U_K D_K^{-1/2})^+\Eta\|^2 = \Eta^\top M\Eta
    =\seta^2 {\tilde \Eta}^\top M {\tilde \Eta}
    \lesssim \seta^2\cdot\log(n)\cdot \tr(M),\]
and so $\PP(\D_2^c)\le e^{-cn}$.\hfill $\blacksquare$

\subsection{ Detailed comparison of the bias and variance terms in Section  \ref{sec:compare} }\label{sec:finite sample}
In this sections we give a detailed comparison between our Theorem \ref{thm:upper bound} and Theorem 4 in \cite{bartlett2019}. We assume throughout this section that the matrices $\sx$ and $\se$ are invertible and the condition number $\ke $ of the matrix $\se$ is bounded above by an absolute constant $c_1$.

%{\color{red} [note: should look if there are any cases when Bartlett's bound is trivial, for example if $k^*=\i$, and see if our bound could hold in this case.]}
 
First define the effective ranks 
\[r_k(\sx) \coloneqq \frac{\sum_{i>k}\lambda_i(\sx)}{\lambda_{i+1}(\sx)},\hspace{1cm} R_k(\sx) \coloneqq \frac{\left(\sum_{i>k} \lambda_i(\sx)\right)^2}{\sum_{i>k} \lambda_i^2(\sx)}.\]
The bound of \cite{bartlett2019} is stated to hold for probability at least $1-\delta$ for a general $\delta<1$ such that $\log(1/\delta)>n/c$ for an absolute constant $c>1$. Taking $\delta = e^{-c'n}$ (for an appropriate $c'$) to ease comparison with our results, the bound then states that with when model (\ref{model}) holds, $(X,y)$ are jointly Gaussian, $\rank(\sx) \ge n$, and $n$ is large enough, with probability at least $1-e^{-c'n}$,
 \[R(\a) - R(\ab) \lesssim B + V,\]
 where
 \begin{equation}
     B \coloneqq \|\ab\|^2\|\sx\|\max\left\{\sqrt{\frac{r_0(\sx)}{n}},\frac{r_0(\sx)}{n}, 1\right\},
 \end{equation}
 and
 \begin{equation}
     \V \coloneqq \sep^2\log(n)\left( \frac{n}{R_{K^*}(\sx)}+\frac{K^*}{n}\right)
 \end{equation}
 are bounds on the bias and variance respectively, and 
 \begin{equation}\label{eqn:k star}
     K^* = \min\{k \ge 0: r_k(\sx)/n \ge b\},
 \end{equation}
 where $b >1$ is an absolute constant. 
 
 We now compare these two terms to the corresponding terms in our bound in Theorem \ref{thm:upper bound}. 
 \subsubsection{Comparison of variance terms}
 We first compare the variance term $\V$ to corresponding variance term in our Theorem \ref{thm:upper bound}, display (\ref{bound:main}). Note that as long as the SNR 
 \[\xi \coloneqq \lk(A\sz A^\top )/\|\se\|\]
 grows fast enough, $K^* = K$ for large enough $n$, where $K$ is the dimension of the latent variables $Z\in \R^K$ in the factor regression model.
\begin{lemma}\label{thm:k=k star}
If $K/n = o(1)$, ${\rm r_e} (\se)/n\ra \i$, and $\xi\ra \i$, such that $\xi^{-1}{\rm r_e} (\se)/n = o(1)$, then $K^* = K$ for all $n$ large enough.
 \end{lemma}
Thus, under the conditions stated in Lemma \ref{thm:k=k star} and for $n$ large enough,
\[\V \coloneqq \sep^2\log(n)\left(\frac{n}{R_{K}(\sx)}+\frac{K}{n} \right).\]
Using the convexity of $x\mapsto x^2$, we can bound $R_K(\sx)$ above via
%  \[R_K(\sx)= \frac{\left(\sum_{i=K+1}^p \lambda_i(\sx)\right)^2}{\sum_{i=K+1}^p \lambda_i^2(\sx)}\le \frac{(p-K)^2\|\se\|^2}{(p-K)\lp^2(\se)} \le p\ke^2\lesssim p,\]
 \[R_K(\sx)= \frac{\left(\sum_{i=K+1}^p \lambda_i(\sx)\right)^2}{\sum_{i=K+1}^p \lambda_i^2(\sx)} \le \frac{(p-K) \sum_{i=K+1}^p \lambda_i^2(\sx)}{\sum_{i=K+1}^p \lambda_i^2(\sx)} \le p.\] 
 Thus,
 \begin{equation}\label{eqn:v lb}
     V\ge \sep^2\log(n)\left(\frac{n}{p}+\frac{K}{n}\right).
 \end{equation}
When $\ke<c_1$, $p\lesssim {\rm r_e} (\se)\le p$, and so the variance term in the bound of our Theorem \ref{thm:upper bound} is
 \[\sep^2\log(n)\left(\frac{n}{r_0(\se)}+\frac{K}{n}\right) \lesssim \sep^2\log(n)\left(\frac{n}{p}+\frac{K}{n}\right).\]
 Thus, comparing with (\ref{eqn:v lb}), we see that under the stated conditions our variance bound is the same as that of \cite{bartlett2019}, up to absolute constants.

\begin{proof}[Proof of Lemma \ref{thm:k=k star}]
We will prove that
 \begin{equation}\label{eqn:rl/n ub}
     \frac{r_\ell(\sx)}{n}\le \frac{K}{n}(1 + \xi^{-1}) +  \frac{1}{\xi}\frac{{\rm r_e} (\se)}{n},\hspace{1cm}\text{for } 0\le \ell\le K-1
 \end{equation}
and that
 \begin{equation}\label{eqn:rl/n lb}
     \frac{r_K(\sx)}{n} \ge \frac{{\rm r_e} (\se)}{n} - \frac{K}{n}.
\end{equation}
Together with the definition of $K^*$ in (\ref{eqn:k star}), these two bounds imply Lemma \ref{thm:k=k star}.

First note that for $0\le \ell \le K$,
 \begin{align}
     \sum_{i=\ell+1}^p\lambda_i(\sx) &= \tr(\sx) - \sum_{i=1}^\ell\lambda_i(\sx)\nonumber\\
     &= \tr(\se) + \tr(\sza) - \sum_{i=1}^\ell\lambda_i(\sx)\nonumber\\
     &=\tr(\se) + \sum_{i=\ell+1}^K\lambda_i(\sza) + \sum_{i=1}^\ell(\lambda_i(\sza) - \lambda_i(\sx)),\label{eqn:sum l ub before weyl}
 \end{align}
 where the sums from $\ell+1$ to $K$ and from $1$ to $\ell$ are defined to be zero when $\ell = K$ and $\ell = 0$, respectively.\\
 
 {\bf Proof of (\ref{eqn:rl/n ub}).} By Weyl's inequality, 
 \begin{equation}\label{eqn:sza sx weyl}
     |\lambda_i(\sza) - \lambda_i(\sx)| \le \|\se\|,
 \end{equation}
 so by (\ref{eqn:sum l ub before weyl}),
 \begin{align}
     \sum_{i=\ell+1}^p\lambda_i(\sx) &\le \tr(\se) + (K-\ell)\lambda_{\ell+1}(\sza) + \ell \|\se\|\nonumber\\
     &\le \tr(\se) + K\lambda_{\ell+1}(\sza) + K \|\se\|\label{eqn:sum l ub before weyl 2}.
 \end{align}
%  From the min-max formula for eigenvalues we also have
%  \begin{equation}\label{eqn:lk lb min max}
%      \lambda_{\ell +1}(\sx) = \lambda_{\ell+1}(\sza + \se) \ge \lambda_{\ell+1}(\sza).
%  \end{equation}
 From the min-max formula for eigenvalues we have
 \begin{equation*}\label{min-max}
     \lambda_{\ell +1}(\sx) = \min_{S:\text{dim}(S)=\ell+1} \max_{x\in S: \|x\|=1} x^\top \sx x,
 \end{equation*}
 where the minimum is taken over all linear subspaces $S\subset \R^p$ with dimension $\ell+1$. Since $x^\top \sx x \ge x^\top \sza x$ for any $x\in \R^p$, this implies
 \begin{equation}\label{eqn:lk lb min max}
     \lambda_{\ell +1}(\sx)\ge \lambda_{\ell +1}(\sza).
 \end{equation}
 Combining (\ref{eqn:sum l ub before weyl 2}) and (\ref{eqn:lk lb min max}), we find
 \begin{align*}
     r_{\ell}(\sx) &= \frac{\sum_{i=\ell+1}^p\lambda_i(\sx)}{\lambda_{\ell+1}(\sx)} \\
     &\le K\left(1 + \frac{\|\se\|}{\lambda_{\ell+1}(\sza)}\right) + \frac{\tr(\se)}{\lambda_{\ell+1}(\sza)}\\
     &\le K\left(1 + \frac{\|\se\|}{\lambda_{K}(\sza)}\right) + \frac{\tr(\se)}{\lambda_{K}(\sza)}\\
     &= K(1+\xi^{-1}) + \xi^{-1} {\rm r_e} (\se),
 \end{align*}
 which completes the proof of (\ref{eqn:rl/n ub}).\\
 
 \textbf{Proof of (\ref{eqn:rl/n lb}).} Equation (\ref{eqn:sum l ub before weyl}) for $\ell = K$ is
 \[\sum_{i=K+1}^p\lambda_i(\sx) = \tr(\se) + \sum_{i=1}^K(\lambda_i(\sza)-\lambda_i(\sx)).\]
 Again using (\ref{eqn:sza sx weyl}), \begin{equation}\label{eqn:sum k+1 lb}
     \sum_{i=K+1}^p\lambda_i(\sx) \ge \tr(\se) - K\|\se\|.
 \end{equation}
Since
%Another application of Weyl's inequality gives
 \begin{eqnarray}
\label{eqn:k+1 weyl}
     \lambda_{K+1}(\sx) &=& \lambda_{K+1}(\sx) - \lambda_{K+1}(\sza) \qquad\text{ (since $\lambda_{K+1}(\sza)=0$)}\nonumber\\
     &\le&  \|\se\| \qquad\text{(Weyl's inequality)}.
\end{eqnarray} 
% where in the first step we use that %$\sza$ has rank at most $K$, and thus 
% $\lambda_{K+1}(\sza)=0$. 
Combining (\ref{eqn:sum k+1 lb}) and (\ref{eqn:k+1 weyl}), we find
 \[r_K(\sx) =\frac{\sum_{i=K+1}^p\lambda_i(\sx)}{\lambda_{K+1}(\sx)} \ge {\rm r_e} (\se) - K,\]
 which proves (\ref{eqn:rl/n lb}).
\end{proof}

\subsubsection{Comparison of bias terms}
A more interesting comparison arises  between the bias term  $B$ and the corresponding bias term in Theorem \ref{thm:upper bound}, display (\ref{bound:main}). Here we will see how the approach we take in this paper, explicitly taking advantage of the structure of the factor regression model, leads to a stronger bound under certain conditions

\begin{lemma}\label{thm:bart bias}Suppose $\xi\coloneqq \lk(\sza)/\|\se\| > 1$ and $A$, $\sz$, $\se$ are all full rank. Then
% \begin{equation}\label{eqn:bart bias lb full}
%     B \ge \frac{\|\beta\|^2_{\sz}}{2\ke}(1-\xi^{-1})\max\left(\sqrt{\frac{r_0(\sx)}{n}},\frac{r_0(\sx)}{n}\right),
% \end{equation}
\begin{equation}\label{eqn:bart bias lb full}
    B \ge\left(\frac{\xi - 1}{\xi+1} \right)\cdot \frac{1}{\ke}\|\beta\|_{\sz}^2\max\left(\sqrt{\frac{r_0(\sx)}{n}},\frac{r_0(\sx)}{n}\right),
\end{equation}
where
\begin{equation}\label{eqn:r0 lb for lemma}
    \frac{r_0(\sx)}{n} \ge \frac{1}{2}\frac{r_0(\sza)}{n} + \frac{1}{2\kappa({\sza})}\frac{1}{\xi}\frac{{\rm r_e} (\se)}{n}.
\end{equation}
In particular, if $\xi > c_1>1$ and $\ke< c_2$, $\kappa({\sza})< c_2$ for absolute constants $c_1,c_2$,
\begin{equation}\label{eqn:bart bias lb}
    B \gtrsim \|\beta\|^2_{\sz}\max\left(\sqrt{\frac{1}{\xi}\frac{p}{n}},\frac{1}{\xi}\frac{p}{n}\right).
\end{equation}
\end{lemma}
%We make several observations based on this Lemma. Firstly,  c
Compared to our bias bound $\|\beta\|^2_{\sz} p / (n\cdot \xi) $ in Theorem \ref{thm:upper bound}, there is an additional quantity $r_0(\sza)/n$ %which can be
of order $O(K/n)$. Ignoring this quantity, provided both $\ke$ and $\kappa({\sza})$ are uniformly  bounded, we obtain %{\color{red}[comment on what happens when $\kappa_{\sza}$ is not bounded]}, we find 
the lower bound (\ref{eqn:bart bias lb}). When $p/(n\cdot \xi) < 1$, this rate is worse by a factor $\sqrt{p/(n\cdot \xi)}$, compared to
 the bias term $\|\beta\|^2_{\sz} p / (n\cdot \xi) $ in Theorem \ref{thm:upper bound}.
 %, we find that under these assumptions, our bound is better by a factor of $\sqrt{\xi^{-1}p/n}$.
 
%  We will use the following simple result in the proof of Lemma \ref{thm:bart bias}, which we isolate here for reference.

\begin{proof}[Proof of Lemma \ref{thm:bart bias}]
% By Lemma \ref{thm:ab norm bound}, which we isolate below for reference elsewhere,
% \[\|\sx\|\|\ab\|^2 \ge \frac{\|\beta\|^2_{\sz}}{2\ke}(1-\xi^{-1}), \]
% which implies (\ref{eqn:bart bias lb full}). 
Using that $A$, $\sz$, $\se$ are all full rank, by (\ref{aea}) above,
% Lemma \ref{thm:ab norm bound main} in Section \ref{sec:low dim},
\[\|\ab\|^2 \ge \left(\frac{\xi - 1}{\xi+1} \right)\cdot \frac{1}{\ke} \cdot \beta^\top (A^\top A)^{-1} \beta\ge \left(\frac{\xi - 1}{\xi+1} \right)\cdot \frac{1}{\ke} \frac{\|\beta\|_{\sz}^2}{\|\sza\|}.\]
Thus, using $\|\sx\| = \|\sza+\se\|\ge \|\sza\|$,
\[\|\sx\|\|\ab\|^2 \ge \left(\frac{\xi - 1}{\xi+1} \right)\cdot \frac{1}{\ke}\|\beta\|_{\sz}^2,\]
which implies (\ref{eqn:bart bias lb full}).

To prove (\ref{eqn:r0 lb for lemma}), we first recall that $r_0(\sx) = \tr(\sx)/\|\sx\|$ and $\sx = \sza + \se$, which implies that
\[\frac{r_0(\sx)}{n}  = \frac{ \tr(A\sz A^\top )}{n\|\sx\|} + \frac{\tr(\se)}{n\|\sx\|}.\]
Observing that $\|\sx\|\le \|\sza\|+\|\se\|\le 2\|\sza\|$, where we use that $\|\se\|\le \|\sza\|$ by the assumption $\xi >1$, we find
% \[\frac{r_0(\sx)}{n} \ge \]
% To lower bound the second term, we use
\begin{align}
    \frac{r_0(\sx)}{n} 
    & \ge \frac{1}{2}\frac{r_0(\sza)}{n} + \frac{1}{2}\frac{\tr(\se)}{n\|\sza\|}\nonumber\\
    &=\frac{1}{2}\frac{r_0(\sza)}{n} + \frac{1}{2}\frac{\lk(\sza)}{\|\sza\|}\frac{\|\se\|}{\lk(\sza)}\frac{\tr(\se)}{n\|\se\|}\nonumber\\
    &= \frac{1}{2}\frac{r_0(\sza)}{n} + \frac{1}{2\kaz}\frac{1}{\xi}\frac{{\rm r_e} (\se)}{n},\nonumber
\end{align}
which proves (\ref{eqn:r0 lb for lemma}).
\end{proof}

\section{Supplementary Results}\label{sec:supp}

% \subsection{A concentration result for quadratic forms}

% We use the following result several times throughout the paper. 

\subsection{Closed form solutions of min-norm estimator and minimizer of $R(\alpha)$}\label{sec:closed form a}
\begin{lemma}\label{thm:ab minimizer}
For zero mean random variables $X\in \R^p$ and $y\in \R$, suppose $\sx \coloneqq\EE[XX^\top]$ and $\sy^2 \coloneqq \EE[y^2]$ are finite, and let $\sxy = \EE[Xy]$. Then $\ab \coloneqq \sx^+\sxy$ is a minimizer of $R(\alpha)$:
\[R(\ab) = \min_{\alpha \in \R^p}R(\alpha).\]
\end{lemma}
\begin{proof}
We have 
\[R(\alpha) = \EE[(X^\top \alpha - y)^2] = \alpha^\top \sx \alpha + \sy^2 - 2\alpha^\top\sxy,\]
so since $R(\alpha)$ is convex, $\alpha$ is a minimizer if and only if
\[\nabla_\alpha R(\alpha) = 2\sx \alpha - 2\sxy = 0.\]
By (\ref{eqn:sx sxy identity}), $\sx\ab = \sxy$, so the claim is proved.

\end{proof}
For $\X \in \R^{n\times p}$ and $\y\in \R^n$, let 
\[\a \coloneqq \arg\min\left\{\|\alpha\|:\ \|\X\alpha - \y\| = \min_u \|\X u - \y\|\right\}.\]
We then have the following result.
\begin{lemma}
$\a = \X^+\y$.
\end{lemma}
\begin{proof}

\textbf{Step 1: Existence and uniqueness of $\a$.} Since
\[\nabla_u\|\X u- y\|^2= 2\X^\top \X u - 2\X^\top y,\]
and $\|\X u- y\|^2$ is convex in $u$, $u$ is a minimizer of $u\mapsto \|\X u- y\|^2$ if and only if
\begin{equation}\label{eqn:min cond}
    \X^\top \X u = \X^\top  \y.
\end{equation}
 By the properties of the pseudo-inverse, $\X^\top \X\X^+ = \X^\top $, so
\[\X^\top \X (\X^+ \y) = \X^\top  \y,\]
and thus $\X^+ \y$ is a minimizer of $\|\X u - \y\|$. The set of vectors $u$ satisfying $\X^\top \X u = \X^\top  \y$ is also convex, so $\a$ is a minimizer of a strictly convex function $\|\cdot\|$ over a non-empty convex set. Such a minimizer exists and is unique, so $\a$ exists and is unique.

\textbf{Step 2: formula for $\a$.} Since $\a$ is a minimizer of $\|\X u-\y\|$, it must satisfy \ref{eqn:min cond}, i.e. 
\begin{equation}\label{eqn:a min cond}
    \X^\top \X \a = \X^\top  y.
\end{equation}
We can write 
\[\a = \X^+\X \a + (I - \X^+\X)\a,\]
and using $\X\X^+\X = \X$ as well as the fact that $\X^+\X$ is symmetric (see Appendix \ref{sec:pseudo-inverse}), a quick calculation gives
\[\|\a\|^2 = \|\X^+\X\a\|^2 + \|(I - \X^+\X)\a\|^2.\]
Thus $\|\X^+\X \a\|\le \|\a\|^2$, and also
\[\X^\top \X (\X^+\X\a) = \X^\top \X \a = \X^\top \y,\]
where we used $\X\X^+\X = \X$ in the first step and \ref{eqn:a min cond} in the second step. Thus $\X^+\X \a$ is a minimizer of $\|\cdot\|$ among minimizers of $\|\X u- \y\|$. Since by Step 1 above $\a$ is the unique such minimizer, $\X^+\X\a = \a$. Thus,
\begin{align*}
    \a &= \X^+\X\a\\
    &= (\X^\top \X)^+\X^\top \X \a &(\text{since } \X^+ = (\X^\top \X)^+ \X^\top )\\
    &= (\X^\top \X)^+\X^\top  \y & (\text{by } \ref{eqn:a min cond})\\
    &=\X^+ \y. &(\text{since } \X^+ = (\X^\top \X)^+ \X^\top )
\end{align*}
\end{proof}

\subsection{Proof that (\ref{model}) is a special case of (\ref{model linear}) in the Gaussian case}\label{proofs:lin model}
% We use the following lemma in Section \ref{sec:low dim}.
\begin{lemma}\label{thm:gaussian frm}
Suppose that $(X,y)$ follows model (\ref{model}) and is furthermore jointly Gaussian. Then model (\ref{model linear}) holds with $\theta=\ab$ and and error  $\eta \coloneqq y-  X^\top \ab$, independent of $X$, where $\ab = \sx^+\sxy$ is the best linear predictor under model (\ref{model}).
\end{lemma}
\begin{proof}
We first compute
\[\EE[X\eta] = \EE[X( y-  X^\top \ab)^2]=\EE[XX^\top ]\ab - \EE[Xy] = \sx\ab - \sxy,\]
where we use that $X$ and $y$ are mean zero in the final step. Using the fact that $\sx\ab = \sxy$ from (\ref{eqn:sx sxy identity}) above, we find $\EE[X\eta]=0$ so $X$ and $\eta$ are uncorrelated, where we again use that $(X,y)$ are mean zero, so $\eta$ is mean zero. Since $X$ and $y$ are jointly normal, it follows that $X$ and $\eta$ are jointly normal. Thus, $X$ and $\eta$ are independent and so model (\ref{model linear}) holds as claimed.
\end{proof}

\subsection{Risk of $\a$ under the factor regression model for $p\ll n$}\label{proof:LS}
For completeness, we provide a risk bound for the minimum-norm estimator $\a$ under the factor regression model in the low-dimensional regime $p\ll n$.
\begin{thm}\label{LS}
Under model \ref{model}, suppose that Assumptions \ref{ass:x}, \ref{ass:fullrank} \& \ref{ass:subg fm} hold. Then if $n > C\cdot p$ for some $C>0$ large enough and $p\ge K$, with probability at least $1-c/n$, 
%%Suppose model \ref{model} holds together with assumptions ... . Let $\ke = \|\se\|/\lp(\se)$. Then there exists $c_1>1$ such that if $n > c_1 p$ then with probability at least ...,
\[R(\a) - \sep^2 \lesssim \ke  \frac{\|\beta\|^2_{\sz}}{\xi} +  \frac{p}{n}\sep^2\log n,\]
where $\ke = \lambda_1(\se)/\lambda_p(\se)$ is the condition number of $\se$.
\end{thm}
\begin{proof}
As in the proof of Theorem \ref{thm:upper bound} found in section \ref{proof:upper bound} above, 
\[R(\a) \le 2 (B_1+B_2) + 2(V_1+V_2),\]
where
\begin{align*}
    &B_1 = \|\Sigma_E^{1/2}\X^+\Z\beta\|^2 \\
    &B_2=\|\sz^{1/2}(A^\top \X^+\Z - I_K)\beta\|^2  \\
    & V_1 =\|\Sigma_E^{1/2}\X^+\Eps\|^2 \\
    &V_2=\|\sz^{1/2}A^\top\X^+\Eps\|^2.
\end{align*}
We will bound these four terms on the event $\B = \B_1 \cap \B_2$, where
\[\B_1 \coloneqq \{\|\tE\|^2 < c_1n,\ \sk^2(\tZ) > c_2 n,\ \sigma_p^2(\tX)\ge c_3 n\}\]
and
\[\B_2 \coloneqq \left\{ \tEps^\top \X^{+\top}\sx\X^+ \tEps\le c_5\log(n)\cdot \tr(\X^{+\top}\sx\X^+)\right\}.\]
As the last step of the proof, we will show that $\PP(\B)\ge 1-c/n$.

\paragraph{Bounding the bias component:} First observe that since $K<n$, when $\Z$ is full rank, $\Z^+\Z=I_K$ and so
\[A^\top\X^+ = \Z^+\Z A^\top\X^+ = \Z^+(\X - \E)\X^+= \Z^+\X\X^+ - \Z^+\E\X^+.\]
Thus, 
\begin{align}\label{eqn:B_2 initial bound p<n}
    B_2 &= \|(A^\top\X^+\Z - I_K)\beta\|^2 \nonumber\\
    &=\|(\Z^+\X\X^+\Z - I_K)\beta - \Z^+\E\X^+\Z\beta\|^2_{\sz}\nonumber\\
    &\le 2\|(\Z^+\X\X^+\Z - I_K)\beta\|^2_{\sz} + 2\|\Z^+\E\X^+\Z\beta\|^2_{\sz}.
\end{align}
Note that since $p\ge K$, by Assumption \ref{ass:fullrank}, $\rank(A)=K$ so by Lemma \ref{thm:psuedo lemma} of Appendix \ref{sec:pseudo-inverse}, 
\begin{equation}\label{eqn:aa = I}
    A^\top A^{+\top} = I_K.
\end{equation}
We thus have
\begin{align}\label{eqn:bias ub1 p<n}
    \|(\Z^+\X\X^+\Z - I_K)\beta\|^2_{\sz} &= \|(\Z^+\X\X^+\Z - \Z^+\Z)\beta\|^2_{\sz}\nonumber\\
    &=\|\tZ^+(\X\X^+-I_p)\Z\beta\|^2\nonumber\\
    &\le \frac{\|(\X\X^+-I_p)\Z\beta\|^2}{\sk^2(\tZ)}\nonumber\\
    &\lesssim \frac{1}{n}\|(\X\X^+-I_p)\Z\beta\|^2\nonumber & (\text{on } \B)\\
    &= \frac{1}{n}\|(\X\X^+-I_p)\Z A^\top A^{+\top}\beta\|^2 \nonumber & (\text{by } (\ref{eqn:aa = I}))\\
    &= \frac{1}{n}\|(\X\X^+-I_p)(\X-\E) A^{+\top}\beta\|^2 \nonumber & (\text{since } \X = \Z A^\top + \E)\\
    &= \frac{1}{n}\|(\X\X^+-I_p)\E A^{+\top}\beta\|^2 \nonumber & (\text{since } \X\X^+\X = \X)\\
    &\le \frac{1}{n}\|\X\X^+-I_p\|\cdot\|\E A^{+\top}\beta\|^2 \nonumber\\
    &\le \frac{1}{n}\|\E A^{+\top}\beta\|^2 \nonumber\\
    &\lesssim \frac{n\|\se\|}{n}\frac{\|\beta\|^2_{\sz}}{\lk(\sza)}\nonumber &(\text{on } \B \text{ and by } (\ref{eqn:A+b ub}))\\
    &= \frac{\|\beta\|^2_{\sz}}{\xi},
\end{align}
where in the penultimate step we used 
\begin{equation}\label{eqn: ab upper bound}
    \|A^{+\top}\beta\|^2 \le \frac{\|\beta\|^2_{\sz}}{\lk(\sza)}
\end{equation}
from (\ref{eqn:A+b ub}). We can bound the second term in \ref{eqn:B_2 initial bound p<n} as follows:
\begin{align*}
    \|\Z^+\E\X^+\Z\beta\|^2_{\sz} &= \|\tZ^+\E\X^+\Z\beta\|^2\\
    &\le \frac{\|\E\|^2}{\sk^2(\tZ)}\|\X^+\Z\beta\|^2\\
    &\lesssim \|\se\|\cdot \|\X^+\Z\beta\|^2& (\text{on } \B)\\
    &=\|\se\|\cdot \|\X^+\Z  A^\top A^{+\top}\beta\|^2  & (\text{since } A^\top A^{+\top} = I_K)\\
    &= \|\se\|\cdot \|\X^+(\X-\E) A^{+\top}\beta\|^2  & (\text{since } \X = \Z A^\top + \E)\\
    &\le 2\|\se\|\cdot \|\X^+\X A^{+\top}\beta\|^2 + 2\|\se\|\cdot \|\X^+\E A^{+\top}\beta\|^2\\
    &\lesssim \|\se\| \|A^{+\top}\beta\|^2 + \|\se\|\frac{\|\E\|}{\sigma_p^2(\X)}\|A^{+\top}\beta\|^2 & (\text{since } \|\X^+\X\|\le 1)\\
    &\lesssim \|\se\| \cdot \ke \|A^{+\top}\beta\|^2\\
    &\le \ke \frac{\|\beta\|^2_{\sz}}{\xi}. & (\text{by } (\ref{eqn: ab upper bound}))
\end{align*}
Using this and (\ref{eqn:bias ub1 p<n}) in (\ref{eqn:B_2 initial bound p<n}), and using the fact that $\ke >1$, we find that on the event $\B$,
\begin{equation}
    B_2\lesssim \ke \frac{\|\beta\|^2_{\sz}}{\xi}.
\end{equation}

\paragraph{Bounding the variance component:} We have
\begin{align}
    V_1+V_2 &= \Eps^\top \X^{+\top} \sx \X^+ \Eps\nonumber\\
    &= \sep^2 \tEps^\top \X^{+\top} \sx \X^+ \tEps & (\text{by Assumption } \ref{ass:subg fm})\nonumber\\
    &\lesssim \sep^2\log(n)\tr(\X^{+\top}\sx\X^+) & (\text{on } \B_2)\nonumber\\
    &\le \sep^2\log(n)\cdot p \|\X^{+\top}\sx\X^+\| & (\text{since } \rank(\X^+)=p)\nonumber\\
    &= \sep^2\log(n)\cdot p \|\sx^{1/2}\X^+\|^2.\label{eqn:var p<n init bound}
\end{align}
From Assumption \ref{ass:x}, $\X = \tX \sx^{1/2}$, and from Lemma \ref{thm:psuedo lemma} of Appendix \ref{sec:pseudo-inverse} below, 
\[(\tX \sx^{1/2})^+ = (\tX^+ \tX \sx^{1/2})^+ (\tX \sx^{1/2}\sx^{-1/2})^+ = \sx^{-1/2}\tX^+.\]
Using this in (\ref{eqn:var p<n init bound}), we find
\[V_1+V_2 \lesssim \sep^2\log(n)\cdot p\|\tX^+\|^2 = \sep^2\log(n)\frac{p}{\sigma_p^2(\tX)}\]
\paragraph{Proof that $\PP(\B)\ge 1-c/n$.}
% {\color{red} can we just say "by similar reasoning as in previous proofs" or something like that? It feels repetitive otherwise.}
The bounds $\PP(\B_1) \ge 1-c/n$ and $\PP(\B_2) \ge 1 - e^{-cn}$ follow respectively from Theorem 4.6.1 of \cite{verHDP} and Lemma \ref{thm:trace lemma} in Appendix \ref{proofs:noiseless} above, by similar reasoning as in the proof of Theorem \ref{thm:upper bound}, for example.
\end{proof}

\subsection{Signal to noise ratio bound for clustered variables}\label{sec:snr cluster}

We present here a lower bound on the signal-to-noise ratio $\xi = \lk(\sza)/\|\se\|$ in terms of the number $|I_a|$ of features related to cluster $a$ only, for $1\le a\le K$. We recall the definition
\[I_a \coloneqq \left\{i\in [p]:\ |A_{ia}|=1, A_{ib} = 0 \text{ for } b\neq a\right\}.\]
\begin{lemma}\label{thm:snr cluster}
$\xi \ge \min_a |I_a|\cdot \lk(\sz)/\|\se\|$.
\end{lemma}
\begin{proof}
For any $v\in\R^K$ with $\|v\|=1$,
\begin{align*}
    v^\top A^\top Av &= \|Av\|^2 = \sum_{i=1}^p \left(\sum_{a=1}^KA_{ia}v_a\right)^2\\
    &\ge \sum_{i\in I} \left(\sum_{a=1}^KA_{ia}v_a\right)^2\\
    &= \sum_{b=1}^K \sum_{i\in I_b} A_{ib}^2v_b^2 \\
    &= \sum_{b=1}^K |I_b|v_b^2 & (|A_{ib}|=1 \text{ for } i\in I_b)\\
    &\ge \min_a |I_a| \cdot\sum_{b=1}^Kv_b^2 = \min_a |I_a|. &(\text{since } \|v\|=1).
\end{align*}
Thus, using $\lk(\sza)\ge \lk(\sz)\lk(A^\top A)$,
\[\xi = \lk(\sza)/\|\se\|\ge \lk(A^\top A)\lk(\sz)/\|\se\|\ge \min_a |I_a| \lk(\sz)/\|\se\|,\]
which completes the proof.
\end{proof}

\section{Properties of the Moore-Penrose pseudo-inverse}
\label{sec:pseudo-inverse}
We state the definition and some properties of the pseudo-inverse in this section for completeness. The material here can be found in \cite{matrix_cookbook}, along with proofs of some of the statements. For a matrix $B\in \R^{n\times m}$, there exists a unique matrix $B^+$, which we define as the pseudo-inverse of $B$, satisfying the following four conditions:
\begin{align}
    &BB^+B = B \\
    &B^+B B^+ = B^+\\
    &BB^+ \text{ is symmetric}\\
    &B^+B \text{ is symmetric}\label{eqn: b+b sym}
\end{align}
We will use the following properties of the pseudo-inverse in this paper.

\begin{lemma}\label{thm:psuedo lemma}
For any $B\in \R^{n\times m}$ and $C\in \R^{m\times d}$,
\begin{equation}\label{eqn:bc}
    (BC)^+ = (B^+BC)^+(BCC^+)^+.
\end{equation}
Furthermore, for any matrix $B\in \R^{n\times m}$ with $r = \rank(B)$ and smallest non-zero  singular value $\sigma_{r}(B)$,
\begin{align}
    &B^\top BB^+ = B^\top \\
    &B^\top (BB^\top)^+ = B^+\\
    &(B^\top B)^+ B^\top = B^+\\
    &B^+B = I_m \ \text{\rm if } r= m\\
    &BB^+ = I_n \ \text{\rm if } r= n\\
    &\|B^+\| = 1/\sigma_{r}(B)\label{eq:X+norm}\\
    &\rank(B^+) = \rank(B) = r.
\end{align}
\end{lemma}

\end{document}